\pdfoutput=1
\documentclass[oneside,11pt]{article}

\newif\ifmydraft
\mydrafttrue
\mydraftfalse

\newif\ifarXiv
\arXivtrue

\newif\ifanon
\anontrue
\anonfalse

\usepackage[a4paper, 
                     margin = 35mm,
                     bottom = 38mm,
                     top       = 38mm]{geometry}

\usepackage{fancyhdr}
\usepackage{natbib}
\usepackage[hidelinks]{hyperref}
\usepackage{xcolor}

\ifmydraft
  \newcommand{\draftcolor}{purple}
  \newcommand{\draftcolorpage}{black!20}
  \usepackage{showlabels}
  \pagecolor{\draftcolorpage}
  \newcommand{\jl}[1]{\color{blue}#1\color{black}}
  \newcommand{\later}[1]{\color{olive}#1\color{black}}
\else
  \newcommand{\draftcolor}{black}
  \usepackage{microtype}      
  \newcommand{\jl}[1]{}
  \newcommand{\later}[1]{}
\fi

\usepackage{amsfonts, amsmath, amsthm}
\usepackage{graphicx}
\usepackage{tikz}
  \usetikzlibrary{arrows.meta}
   \usetikzlibrary{positioning}
   \tikzset{
     included node/.style={circle, draw=black!100, thick, on grid, minimum width=13mm, fill=black!20}, 
    hidden node/.style={circle, draw=white, thick, on grid, minimum width=0.2cm},
    included connection/.style={->, thick, draw=black!100},
    hidden connection/.style={->, thick, draw=black!30, draw opacity=0},
    fused connection/.style={->, thick, draw=black!100},
    fidden connection/.style={->, thick, draw=black!30, draw opacity=0},
    node title/.style={above=0.8cm of five-one, font=\bfseries},
    general/.style={node distance=1cm and 0.7cm}
}

\theoremstyle{definition}

\newtheorem{lemma}{Lemma}
\newtheorem{example}{Example}
\newtheorem{definition}{Definition}

\DeclareMathOperator*{\argmax}{arg\,max}
\DeclareMathOperator*{\argmin}{arg\,min}

\makeatletter 
\newcommand*{\deq}{\ensuremath{\mathrel{\rlap{%
\raisebox{0.3ex}{$\m@th\cdot$}}%
\raisebox{-0.3ex}{$\m@th\cdot$}}=}}
\makeatother

\newcommand{\sign}{\operatorname{sign}}
\newcommand{\signF}[1]{\sign[#1]}

\newcommand{\inv}{^{-1}}
\newcommand{\tp}{^\top}

\newcommand{\abs}[1]{\lvert#1\rvert}

\newcommand{\absBB}[1]{\Bigl\lvert#1\Bigr\rvert}

\newcommand{\norm}[1]{|\!|#1|\!|}

\newcommand{\normtwo}[1]{\norm{#1}_2}

\newcommand{\nbrlayers}{\textcolor{\draftcolor}{\ensuremath{l}}}
\newcommand{\nbrsamples}{\textcolor{\draftcolor}{\ensuremath{n}}}
\newcommand{\nbrinput}{\textcolor{\draftcolor}{\ensuremath{d}}}

\newcommand{\nbrwidth}{\textcolor{\draftcolor}{\ensuremath{w}}}

\newcommand{\datainE}{\textcolor{\draftcolor}{\ensuremath{x}}}
\newcommand{\datain}{\textcolor{\draftcolor}{\boldsymbol{\datainE}}}
\newcommand{\dataoutE}{\textcolor{\draftcolor}{\ensuremath{y}}}

\newcommand{\argument}{\ensuremath{\textcolor{\draftcolor}{z}}}
\newcommand{\argumentP}{\ensuremath{\textcolor{\draftcolor}{w}}}
\newcommand{\argumentPZ}{\ensuremath{\textcolor{\draftcolor}{\argumentP_0}}}

\newcommand{\function}{\textcolor{\draftcolor}{\ensuremath{\mathfrak f}}}
\newcommand{\functionG}{\textcolor{\draftcolor}{\ensuremath{\mathfrak g}}}

\newcommand{\functionF}[1]{\textcolor{\draftcolor}{\function[#1]}}
\newcommand{\functionGF}[1]{\textcolor{\draftcolor}{\functionG[#1]}}

\newcommand{\functionFkBBBB}[1]{\textcolor{\draftcolor}{\function^k\Biggl[#1\Biggr]}}
\newcommand{\functionFBBBB}[1]{\textcolor{\draftcolor}{\function\Biggl[#1\Biggr]}}
\newcommand{\functionGFBBBB}[1]{\textcolor{\draftcolor}{\functionG\hspace{-1mm}\left[#1\right]}}
\newcommand{\functionP}{\ensuremath{\textcolor{\draftcolor}{\mathfrak g}}}
\newcommand{\functionPF}[1]{\ensuremath{\textcolor{\draftcolor}{\functionP[#1]}}}

\newcommand{\functionPFBBB}[1]{\ensuremath{\textcolor{\draftcolor}{\functionP\biggl[#1\biggr]}}}

\newcommand{\functionD}{\textcolor{\draftcolor}{\dot\function}}
\newcommand{\functionDF}[1]{\textcolor{\draftcolor}{\functionD[#1]}}

\newcommand{\functionGD}{\textcolor{\draftcolor}{\dot\functionG}}
\newcommand{\functionGDF}[1]{\textcolor{\draftcolor}{\functionGD[#1]}}

\newcommand{\functionPD}{\textcolor{\draftcolor}{\dot\functionP}}
\newcommand{\functionPDF}[1]{\textcolor{\draftcolor}{\functionPD[#1]}}

\newcommand{\functionDD}{\ensuremath{\textcolor{\draftcolor}{\ddot\function}}}
\newcommand{\functionDDF}[1]{\ensuremath{\textcolor{\draftcolor}{\functionDD[#1]}}}

\newcommand{\functionDDD}{\ensuremath{\textcolor{\draftcolor}{\dddot\function}}}

\newcommand{\functionDDDD}{\ensuremath{\textcolor{\draftcolor}{\ddddot\function}}}

\newcommand{\R}{\textcolor{\draftcolor}{\mathbb{R}}}

\newcommand{\direction}{\textcolor{\draftcolor}{ \ensuremath{v}}}

\newcommand{\diffdirectional}{\textcolor{\draftcolor}{\ensuremath{d_{\direction}}}}
\newcommand{\diffdirectionalM}{\textcolor{\draftcolor}{\ensuremath{d_{-\direction}}}}

\newcommand{\directionlim}{\textcolor{\draftcolor}{\ensuremath{t}}}
\newcommand{\directionlimZ}{\textcolor{\draftcolor}{\ensuremath{\bar t}}}

\newcommand{\functionpar}{\textcolor{\draftcolor}{\ensuremath{a}}}
\newcommand{\functionparP}{\textcolor{\draftcolor}{\ensuremath{b}}}

\newcommand{\nameswish}{\textcolor{\draftcolor}{\texttt{swish}}}
\newcommand{\nameswishf}{\textcolor{\draftcolor}{\operatorname{swish}}}
\newcommand{\namesoftplus}{\textcolor{\draftcolor}{\texttt{softplus}}}
\newcommand{\namesoftplusf}{\textcolor{\draftcolor}{\operatorname{soft+}}}
\newcommand{\namelog}{\textcolor{\draftcolor}{\texttt{logistic}}}
\newcommand{\nametanh}{\textcolor{\draftcolor}{\texttt{tanh}}}
\newcommand{\namerelu}{\textcolor{\draftcolor}{\texttt{relu}}}
\newcommand{\nameselu}{\textcolor{\draftcolor}{\texttt{selu}}}
\newcommand{\nameelu}{\textcolor{\draftcolor}{\texttt{elu}}}
\newcommand{\namelrelu}{\textcolor{\draftcolor}{\texttt{leakyrelu}}}
\newcommand{\nameprelu}{\textcolor{\draftcolor}{\texttt{prelu}}}
\newcommand{\namepelu}{\textcolor{\draftcolor}{\texttt{pelu}}}
\newcommand{\namemaxout}{\textcolor{\draftcolor}{\texttt{maxout}}}
\newcommand{\nameelliott}{\textcolor{\draftcolor}{\texttt{elliottsig}}}
\newcommand{\namesoft}{\textcolor{\draftcolor}{\texttt{softsign}}}
\newcommand{\namesoftf}{\textcolor{\draftcolor}{\operatorname{soft}}}
\newcommand{\namearctan}{\textcolor{\draftcolor}{\texttt{arctan}}}
\newcommand{\namerrelu}{\textcolor{\draftcolor}{\texttt{rrelu}}}
\newcommand{\namelinear}{\textcolor{\draftcolor}{\texttt{linear}}}
\newcommand{\namesil}{\textcolor{\draftcolor}{\texttt{sil}}}
\newcommand{\namemish}{\textcolor{\draftcolor}{\texttt{mish}}}
\newcommand{\nameelish}{\textcolor{\draftcolor}{\texttt{elish}}}

\newcommand{\functionbinary}{\textcolor{\draftcolor}{\ensuremath{\function_{\operatorname{binary}}}}}

\newcommand{\functionlog}{\textcolor{\draftcolor}{\ensuremath{\function_{\operatorname{log}}}}}
\newcommand{\functionlogF}[1]{\textcolor{\draftcolor}{\functionlog[#1]}}
\newcommand{\functionlogD}{\textcolor{\draftcolor}{\ensuremath{\functionD_{\operatorname{log}}}}}
\newcommand{\functionlogDF}[1]{\textcolor{\draftcolor}{\functionlogD[#1]}}
\newcommand{\functionlogDD}{\textcolor{\draftcolor}{\ensuremath{\functionDD_{\operatorname{log}}}}}
\newcommand{\functionlogDDF}[1]{\textcolor{\draftcolor}{\functionlogDD[#1]}}
\newcommand{\functionlogDDD}{\textcolor{\draftcolor}{\ensuremath{\functionDDD_{\operatorname{log}}}}}
\newcommand{\functionlogDDDF}[1]{\textcolor{\draftcolor}{\functionlogDDD[#1]}}

\newcommand{\functionarctan}{\textcolor{\draftcolor}{\ensuremath{\function_{\operatorname{arctan}}}}}
\newcommand{\functionarctanF}[1]{\textcolor{\draftcolor}{\functionarctan[#1]}}
\newcommand{\functionarctanD}{\textcolor{\draftcolor}{\ensuremath{\functionD_{\operatorname{arctan}}}}}
\newcommand{\functionarctanDF}[1]{\textcolor{\draftcolor}{\functionarctanD[#1]}}
\newcommand{\functionarctanDD}{\textcolor{\draftcolor}{\ensuremath{\functionDD_{\operatorname{arctan}}}}}
\newcommand{\functionarctanDDF}[1]{\textcolor{\draftcolor}{\functionarctanDD[#1]}}
\newcommand{\functionarctanDDD}{\textcolor{\draftcolor}{\ensuremath{\functionDDD_{\operatorname{arctan}}}}}
\newcommand{\functionarctanDDDF}[1]{\textcolor{\draftcolor}{\functionarctanDDD[#1]}}
\newcommand{\functionarctanDDDD}{\textcolor{\draftcolor}{\ensuremath{\functionDDDD_{\operatorname{arctan}}}}}
\newcommand{\functionarctanDDDDF}[1]{\textcolor{\draftcolor}{\functionarctanDDDD[#1]}}

\newcommand{\functionlrelu}{\textcolor{\draftcolor}{\ensuremath{\function_{\operatorname{lrelu,\functionpar}}}}}
\newcommand{\functionlreluF}[1]{\textcolor{\draftcolor}{\functionlrelu[#1]}}
\newcommand{\functionlreluD}{\textcolor{\draftcolor}{\ensuremath{\functionD_{\operatorname{lrelu,\functionpar}}}}}
\newcommand{\functionlreluDF}[1]{\textcolor{\draftcolor}{\functionlreluD[#1]}}
\newcommand{\functionlreluDD}{\textcolor{\draftcolor}{\ensuremath{\functionDD_{\operatorname{lrelu,\functionpar}}}}}
\newcommand{\functionlreluDDF}[1]{\textcolor{\draftcolor}{\functionlreluDD[#1]}}

\newcommand{\functiontanh}{\textcolor{\draftcolor}{\ensuremath{\function_{\operatorname{tanh}}}}}
\newcommand{\functiontanhF}[1]{\textcolor{\draftcolor}{\functiontanh[#1]}}
\newcommand{\functiontanhFB}[1]{\textcolor{\draftcolor}{\functiontanh\bigl[#1\bigr]}}
\newcommand{\functiontanhFBB}[1]{\textcolor{\draftcolor}{\functiontanh\Bigl[#1\Bigr]}}
\newcommand{\functiontanhD}{\textcolor{\draftcolor}{\ensuremath{\functionD_{\operatorname{tanh}}}}}
\newcommand{\functiontanhDF}[1]{\textcolor{\draftcolor}{\functiontanhD[#1]}}
\newcommand{\functiontanhDD}{\textcolor{\draftcolor}{\ensuremath{\functionDD_{\operatorname{tanh}}}}}
\newcommand{\functiontanhDDF}[1]{\textcolor{\draftcolor}{\functiontanhDD[#1]}}
\newcommand{\functiontanhDDD}{\textcolor{\draftcolor}{\ensuremath{\functionDDD_{\operatorname{tanh}}}}}
\newcommand{\functiontanhDDDF}[1]{\textcolor{\draftcolor}{\functiontanhDDD[#1]}}
\newcommand{\functiontanhDDDD}{\textcolor{\draftcolor}{\ensuremath{\functionDDDD_{\operatorname{tanh}}}}}
\newcommand{\functiontanhDDDDF}[1]{\textcolor{\draftcolor}{\functiontanhDDDD[#1]}}

\newcommand{\functionselu}{\textcolor{\draftcolor}{\ensuremath{\function_{\operatorname{selu}}}}}
\newcommand{\functionseluF}[1]{\textcolor{\draftcolor}{\functionselu[#1]}}

\newcommand{\functionsoft}{\textcolor{\draftcolor}{\ensuremath{\function_{\namesoftf}}}}
\newcommand{\functionsoftF}[1]{\textcolor{\draftcolor}{\functionsoft[#1]}}
\newcommand{\functionsoftD}{\textcolor{\draftcolor}{\ensuremath{\functionD_{\namesoftf}}}}
\newcommand{\functionsoftDF}[1]{\textcolor{\draftcolor}{\functionsoftD[#1]}}
\newcommand{\functionsoftDD}{\textcolor{\draftcolor}{\ensuremath{\functionDD_{\namesoftf}}}}
\newcommand{\functionsoftDDF}[1]{\textcolor{\draftcolor}{\functionsoftDD[#1]}}

\newcommand{\functionrelu}{\textcolor{\draftcolor}{\ensuremath{\function_{\operatorname{relu}}}}}
\newcommand{\functionreluF}[1]{\textcolor{\draftcolor}{\functionrelu[#1]}}
\newcommand{\functionreluFB}[1]{\textcolor{\draftcolor}{\functionrelu\bigl[#1\bigr]}}
\newcommand{\functionreluD}{\textcolor{\draftcolor}{\ensuremath{\functionD_{\operatorname{relu}}}}}
\newcommand{\functionreluDF}[1]{\textcolor{\draftcolor}{\functionreluD[#1]}}
\newcommand{\functionreluDD}{\textcolor{\draftcolor}{\ensuremath{\functionDD_{\operatorname{relu}}}}}
\newcommand{\functionreluDDF}[1]{\textcolor{\draftcolor}{\functionreluDD[#1]}}

\newcommand{\functionlinear}{\textcolor{\draftcolor}{\ensuremath{\function_{\operatorname{linear}}}}}

\newcommand{\functionlinearD}{\textcolor{\draftcolor}{\ensuremath{\functionD_{\operatorname{linear}}}}}
\newcommand{\functionlinearDF}[1]{\textcolor{\draftcolor}{\functionlinearD[#1]}}
\newcommand{\functionlinearDD}{\textcolor{\draftcolor}{\ensuremath{\functionDD_{\operatorname{linear}}}}}
\newcommand{\functionlinearDDF}[1]{\textcolor{\draftcolor}{\functionlinearDD[#1]}}

\newcommand{\functionsoftplus}{\textcolor{\draftcolor}{\ensuremath{\function_{\namesoftplusf}}}}
\newcommand{\functionsoftplusF}[1]{\textcolor{\draftcolor}{\functionsoftplus[#1]}}
\newcommand{\functionsoftplusD}{\textcolor{\draftcolor}{\ensuremath{\functionD_{\namesoftplusf}}}}
\newcommand{\functionsoftplusDF}[1]{\textcolor{\draftcolor}{\functionsoftplusD[#1]}}
\newcommand{\functionsoftplusDD}{\textcolor{\draftcolor}{\ensuremath{\functionDD_{\namesoftplusf}}}}
\newcommand{\functionsoftplusDDF}[1]{\textcolor{\draftcolor}{\functionsoftplusDD[#1]}}

\newcommand{\functionswish}{\textcolor{\draftcolor}{\ensuremath{\function_{\nameswishf,\functionpar}}}}
\newcommand{\functionswishF}[1]{\textcolor{\draftcolor}{\functionswish[#1]}}
\newcommand{\functionswishZ}{\textcolor{\draftcolor}{\ensuremath{\function_{\nameswishf,0}}}}
\newcommand{\functionswishZF}[1]{\textcolor{\draftcolor}{\functionswishZ[#1]}}
\newcommand{\functionswishZD}{\textcolor{\draftcolor}{\ensuremath{\functionD_{\nameswishf,0}}}}
\newcommand{\functionswishZDF}[1]{\textcolor{\draftcolor}{\functionswishZD[#1]}}

\newcommand{\functionswishO}{\textcolor{\draftcolor}{\ensuremath{\function_{\nameswishf,1}}}}
\newcommand{\functionswishOF}[1]{\textcolor{\draftcolor}{\functionswishO[#1]}}
\newcommand{\functionswishOD}{\textcolor{\draftcolor}{\ensuremath{\functionD_{\nameswish,1}}}}
\newcommand{\functionswishODF}[1]{\textcolor{\draftcolor}{\functionswishOD[#1]}}
\newcommand{\functionswishD}{\textcolor{\draftcolor}{\ensuremath{\functionD_{\nameswishf,\functionpar}}}}
\newcommand{\functionswishDF}[1]{\textcolor{\draftcolor}{\functionswishD[#1]}}
\newcommand{\functionswishDD}{\textcolor{\draftcolor}{\ensuremath{\functionDD_{\nameswishf,\functionpar}}}}
\newcommand{\functionswishDDF}[1]{\textcolor{\draftcolor}{\functionswishDD[#1]}}

\newcommand{\functionelu}{\textcolor{\draftcolor}{\ensuremath{\function_{\operatorname{elu},\functionpar}}}}
\newcommand{\functioneluz}{\textcolor{\draftcolor}{\ensuremath{\function_{\operatorname{elu},0}}}}
\newcommand{\functioneluF}[1]{\textcolor{\draftcolor}{\functionelu[#1]}}
\newcommand{\functioneluD}{\textcolor{\draftcolor}{\ensuremath{\functionD_{\operatorname{elu},\functionpar}}}}
\newcommand{\functioneluDF}[1]{\textcolor{\draftcolor}{\functioneluD[#1]}}

\newcommand{\functioneluOD}{\textcolor{\draftcolor}{\ensuremath{\functionD_{\operatorname{elu},1}}}}
\newcommand{\functioneluODF}[1]{\textcolor{\draftcolor}{\functioneluOD[#1]}}

\newcommand{\functioneluDD}{\textcolor{\draftcolor}{\ensuremath{\functionDD_{\operatorname{elu},\functionpar}}}}
\newcommand{\functioneluDDF}[1]{\textcolor{\draftcolor}{\functioneluDD[#1]}}

\newcommand{\swishbound}{\textcolor{\draftcolor}{\ensuremath{c}}}
\newcommand{\swishboundA}{\textcolor{\draftcolor}{\ensuremath{\swishbound_{\functionpar}}}}

\newcommand{\swishboundP}{\textcolor{\draftcolor}{\ensuremath{c'}}}

\newcommand{\swishboundAPPLower}{\textcolor{\draftcolor}{\underline{c}''_{\functionpar}}}
\newcommand{\swishboundAPPUpper}{\textcolor{\draftcolor}{\overline{c}''_{\functionpar}}}

\newcommand{\neuron}{\textcolor{\draftcolor}{\ensuremath{\mathfrak{n}}}}
\newcommand{\neuronparB}{\textcolor{\draftcolor}{\ensuremath{\beta}}}
\newcommand{\neuronparBB}{\textcolor{\draftcolor}{\ensuremath{\boldsymbol{\neuronparB}}}}
\newcommand{\neuronparBP}{\textcolor{\draftcolor}{\ensuremath{\zeta}}}
\newcommand{\neuronparE}{\textcolor{\draftcolor}{\ensuremath{\theta}}}
\newcommand{\neuronpar}{\textcolor{\draftcolor}{\ensuremath{\boldsymbol{\neuronparE}}}}
\newcommand{\neuronparPE}{\textcolor{\draftcolor}{\ensuremath{\gamma}}}
\newcommand{\neuronparP}{\textcolor{\draftcolor}{\ensuremath{\boldsymbol{\neuronparPE}}}}

\newcommand{\neuronA}{\textcolor{\draftcolor}{\ensuremath{\neuron_{\neuronparB,\neuronpar,\function}}}}

\newcommand{\neuronAmax}{\textcolor{\draftcolor}{\ensuremath{\tilde\neuron_{\neuronparBB,\neuronpar,k}}}}
\newcommand{\neuronAmaxO}{\textcolor{\draftcolor}{\ensuremath{\tilde\neuron_{\neuronparB,\neuronpar,1}}}}

\newcommand{\neuronAO}{\textcolor{\draftcolor}{\ensuremath{\neuron_{\neuronparB^1,\neuronpar^1,\function^1}}}}
\newcommand{\neuronAOF}[1]{\textcolor{\draftcolor}{\ensuremath{\neuronAO[#1]}}}

\newcommand{\neuronAT}{\textcolor{\draftcolor}{\ensuremath{\neuron_{\neuronparB^2,\neuronpar^2,\function^2}}}}

\newcommand{\neuronAK}{\textcolor{\draftcolor}{\ensuremath{\neuron_{\neuronparB^{\nbrwidth},\neuronpar^{\nbrwidth},\function^{\nbrwidth}}}}}
\newcommand{\neuronAKF}[1]{\textcolor{\draftcolor}{\ensuremath{\neuronAK[#1]}}}

\newcommand{\neuronALK}{\textcolor{\draftcolor}{\ensuremath{\neuron_{\neuronparB^{\nbrwidth},\neuronpar^{\nbrwidth},\functionlinear}}}}
\newcommand{\neuronALKF}[1]{\textcolor{\draftcolor}{\ensuremath{\neuronALK[#1]}}}
\newcommand{\neuronALO}{\textcolor{\draftcolor}{\ensuremath{\neuron_{\neuronparB^1,\neuronpar^{1},\functionlinear}}}}
\newcommand{\neuronALOF}[1]{\textcolor{\draftcolor}{\ensuremath{\neuronALO[#1]}}}

\newcommand{\neuronAP}{\textcolor{\draftcolor}{\ensuremath{\neuron_{\neuronparBP,\neuronparP,\functionG}}}}
\newcommand{\neuronAPL}{\textcolor{\draftcolor}{\ensuremath{\neuron_{\neuronparBP,\neuronparP,\functionlinear}}}}

\newcommand{\stepsize}{\textcolor{\draftcolor}{\ensuremath{s}}}

\usepackage[textwidth=3.5cm, textsize=small]{todonotes}
\newcommand{\marginterm}[1]{\todo[color=white, bordercolor=white, noline]{\color{black}{\vspace{-0.9mm}\fontsize{9}{8}\par\selectfont#1\par}}}

\newcommand{\datainEnew}{\textcolor{\draftcolor}{\ensuremath{x_{\operatorname{new}}}}}
\newcommand{\dataoutEnew}{\textcolor{\draftcolor}{\ensuremath{y_{\operatorname{new}}}}}

\newcommand{\argumentmin}{\textcolor{\draftcolor}{\argument_{\functionpar}}}
\newcommand{\argumentminP}{\textcolor{\draftcolor}{\argument_{\functionpar}'}}

\newcommand{\argumentminO}{\textcolor{\draftcolor}{\argument_{1}}}
\newcommand{\argumentminOP}{\textcolor{\draftcolor}{\argument_{1}'}}

\newcommand{\tj}[1]{\tag*{\footnotesize #1}}

\newcommand{\detail}[1]{{\scriptsize Detail: #1}}
\makeatletter
\def\input@path{{.}{./Plots/}{../Plots/}}
\makeatother

\begin{document}

\title{Activation Functions in Artificial Neural Networks:\\
A Systematic Overview}

\ifanon
  \author{\name Somebody \email some@email.com\\
              \addr Some department\\
              Some school\\
              Some country}
\ShortHeadings{Some title}{Somebody}
\else
  \author{\textbf{Johannes Lederer}\\
              Department of Mathematics\\
              Ruhr-University Bochum, Germany\\
johannes.lederer@rub.de}

\maketitle

\begin{abstract}
Activation functions shape the outputs of artificial neurons and, therefore, are integral parts of neural networks in general and  deep learning in particular.
Some activation functions,
such as logistic and relu,
have been used for many decades.
But with deep learning becoming a mainstream research topic,
new activation functions have mushroomed,
leading to confusion in both theory and practice.
This paper provides an analytic yet up-to-date overview of popular activation functions and their properties,
which makes it a timely resource for anyone who studies or applies neural networks.
\end{abstract}


\section{Introduction}
Artificial neural networks were introduced  as mathematical models for biological neural networks~\citep{McCulloch1943,Rosenblatt1958}.
Modern artificial neural networks still reflect this biological motivation,
even though they are often applied in completely different contexts.
Given parameters~$\neuronparB\in\R$ and $\neuronpar\in\R^{\nbrinput}$ and a real function  $\function\,:\,\R\to\R$,
the \emph{(artificial) neuron} $\neuronA$ (we omit the word ``artificial'' when the context is clear) is  the function\marginterm{neuron}
\begin{equation}\label{neuron}
  \begin{split}
      \neuronA\ : \ \R^{\nbrinput}\,&\to\,\R\,;\\ 
    \datain\,&\mapsto\,\functionF{\neuronparB+\neuronpar\tp\datain}=\functionFBBBB{\neuronparB+\sum_{j=1}^{\nbrinput}\neuronparE_j\datainE_j}\,.
  \end{split}
\end{equation}
See the left panel of Figure~\ref{networks} for an illustration.
The weights $\neuronparE_1,\dots,\neuronparE_{\nbrinput}$ can be interpreted as the neuron's sensitivities to the different inputs and 
the bias $\neuronparB$ as the neuron's overall sensitivity;
the function~$\function$ can be interpreted as the neuron's activation pattern and, therefore, is called the \emph{activation function}\footnote{An alternative name that relates to the terminology in electrical engineering is \emph{transfer function}.} \citep{Bishop1995}.\marginterm{activation function}
Hence, 
artificial neurons resemble biological neurons in the way they translate multiple input signals  into a single output signal~\citep{Nicholls2001}.

Just as biological neurons are the basic units of biological neural networks,
artificial neurons are the basic units of artificial neural networks.
The key observation is that artificial neurons can be added and concatenated into new functions  $\R^{\nbrinput}\to\R$:
given parameters $\neuronparBP\in\R,\neuronparP\in\R^{\nbrwidth}$ and $\neuronparB^1,\dots,\neuronparB^{\nbrwidth}\in\R,\neuronpar^1,\dots,\neuronpar^{\nbrwidth}\in\R^{\nbrinput}$ and activation functions $\functionG,\function^1,\dots,\function^{\nbrwidth}\,:\,\R\to\R$,
the function\label{concatenated}
\begin{align*}
   \R^{\nbrinput}\,&\to\,\R\\ 
    \datain\,&\mapsto\,
               \neuronAP\hspace{-1mm} \begin{bmatrix}
                \neuronAOF{\datain}\\
\vdots\\
\neuronAKF{\datain}
               \end{bmatrix}=\functionGFBBBB{\neuronparBP+\sum_{k=1}^{\nbrwidth}\neuronparPE_k\functionFkBBBB{\neuronparB^k+\sum_{j=1}^{\nbrinput}(\neuronpar^k)_j\datainE_j}}
\end{align*}
is  well-defined;
see the right panel of Figure~\ref{networks} for an illustration.
Such functions can then be added and concatenated further with one another,
eventually leading to  very complex functions $\R^{\nbrinput}\to\R$.
We call these functions, and functions that  are derived from or motivated by them, \emph{(artificial) neural networks}.\marginterm{neural network}
We call the use of  neural networks that have many layers of concatenated neurons \marginterm{deep learning}\emph{deep learning} \citep{Goodfellow2016,LeCun15, Schmidhuber15}.

Since a neural network is made of neurons,
its characteristics are governed by the neurons' parameters and activation functions.
The parameters are usually fitted to training data;
in contrast,
the activation functions are usually chosen before looking at any data and remain fixed.
The literature contains an entire zoo of different activation functions,
and arguments for and against specific activation functions are often based on heuristics and anecdotal evidence.
We instead attempt a systematic and objective overview of common activation functions.
In particular,
we examine the mathematical particularities of each function and discuss their practical effects,
making our paper a useful resource for theorists and practitioners alike.

\paragraph{Overview}
The activation functions and their properties are discussed in Section~\ref{sec:activationfunctions},
and practical implications are discussed in Section~\ref{sec:discussion}.
Detailed proofs of the mathematical statements are established in the Appendix.
The \texttt{R}~code for the plots and further visualizations are given on \texttt{\href{https://github.com/LedererLab/ActivationFunctions}{github.com/LedererLab/ActivationFunctions}}.

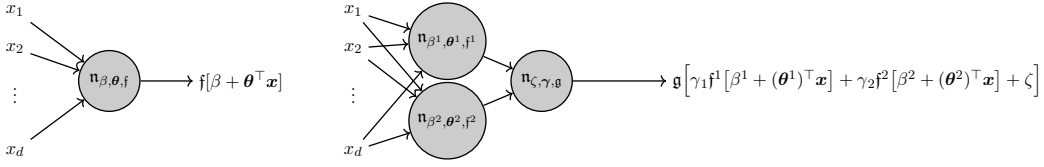
\begin{figure}[t]
\newcommand{\overallscale}{0.62}
  \centering
\scalebox{\overallscale}{\begin{tikzpicture}
\node[included node] (neuron) at (0, 0) {\neuronA};
\node (dataone) at (-2, 1.5) {$\datainE_1$};
\node (datatwo) at (-2, 0.7) {$\datainE_2$};

\node at (-2, -0.2) {$\vdots$};

\node (dataend) at (-2, -1.5) {$\datainE_{\nbrinput}$};

\node (dataout) at (2.8, 0) {$\functionF{\neuronparB+\neuronpar\tp\datain}$};

\draw [included connection] (dataone)--(neuron);
\draw [included connection] (datatwo)--(neuron);
\draw [included connection] (dataend)--(neuron);
\draw [included connection] (neuron)--(dataout);
\end{tikzpicture}}~~~~
\scalebox{\overallscale}{\begin{tikzpicture}
\node[included node] (neuronThree) at (2, 0) {\neuronAP};
\node[included node] (neuronOne) at (0, 0.85) {\neuronAO};
\node[included node] (neuronTwo) at (0, -0.85) {\neuronAT};

\node (dataone) at (-2, 1.5) {$\datainE_1$};
\node (datatwo) at (-2, 0.7) {$\datainE_2$};

\node at (-2, -0.2) {$\vdots$};

\node (dataend) at (-2, -1.5) {$\datainE_{\nbrinput}$};

\node (dataout) at (8.7, 0) {$\functionG\Bigl[\neuronparPE_1
    \function^1\bigl[\neuronparB^1+(\neuronpar^1)\tp\datain\bigr]+\neuronparPE_2
    \function^2\bigl[\neuronparB^2+(\neuronpar^2)\tp\datain\bigr]+\neuronparBP\Bigr]$};

\draw [included connection] (dataone)--(neuronOne);
\draw [included connection] (datatwo)--(neuronOne);
\draw [included connection] (dataend)--(neuronOne);
\draw [included connection] (dataone)--(neuronTwo);
\draw [included connection] (datatwo)--(neuronTwo);
\draw [included connection] (dataend)--(neuronTwo);
\draw [included connection] (neuronOne)--(neuronThree);
\draw [included connection] (neuronTwo)--(neuronThree);
\draw [included connection] (neuronThree)--(dataout);
\end{tikzpicture}}  
  \caption{Left panel: neurons translate multiple input signals into a single output signal.
Right panel: 
neural networks consist of multiple interacting neurons (typically many more than three).}
  \label{networks}
\end{figure}

\section{Common Activation Functions and Their Properties}
\label{sec:activationfunctions}
We discuss a wide range of activation functions,
with \namelog, \nametanh, and \namerelu\ as popular examples.
We study the functions' derivatives as well as the functions themselves.
The activation functions themselves influence the network's expressivity, that is, the network's ability to approximate target functions.
For example, we will show that \namelinear-activation networks are always linear and, therefore, can only approximate linear functions.
The activation functions together with their first derivatives are key factors in the network's computational complexity, that is, the computational costs of parameter optimization.
The reason is that the optimization steps of popular algorithms for parameter optimization, such as  stochastic-gradient descent \citep[Chapter~6]{Bubeck2017}, are based on the gradients or generalized gradients of the network with respect to the parameters,
which means---recall the chain rule---that each optimization step requires many evaluations of the activation functions and their first derivatives.
The second derivatives of the activation functions are finally important for certain mathematical theories about neural networks.

The first and second derivatives of an activation function~\function\ are denoted by~$\functionD$ and~\functionDD, respectively,
and the first directional derivative of~\function\ in direction~\direction\ by \diffdirectional\function.
The natural logarithm is denoted by $\log$.
Proofs, details on directional derivatives, and other mathematical background are provided in the Appendix.

\subsection{Sigmoid Functions}
\label{sec:sigmoid}
A \marginterm{sigmoid function}\emph{sigmoid function} is a bounded and differentiable function that is nondecreasing and has exactly one inflection point.
Roughly speaking, a sigmoid function is a smooth, ``S-shaped'' curve.
Because sigmoid functions ``squash'' the real values into a bounded interval,
they are sometimes called \emph{squashing functions.}
Sigmoid activation has a long-standing tradition in the theory and practice of neural networks.
One motivation has been the sigmoid-shaped activation patterns observed in neuroscience---see, for example, \citet[Figure~3]{Lipetz1969}.
Another motivation has been the interpretation of sigmoid functions as mathematically tractable approximations of the step functions---see below.

In the following,
we discuss four sigmoid functions:
\namelog,
\namearctan,
\nametanh,
and \namesoft.
An overview is provided in Figure~\ref{fig:sigmoid} on the next page.
\begin{figure}[t]
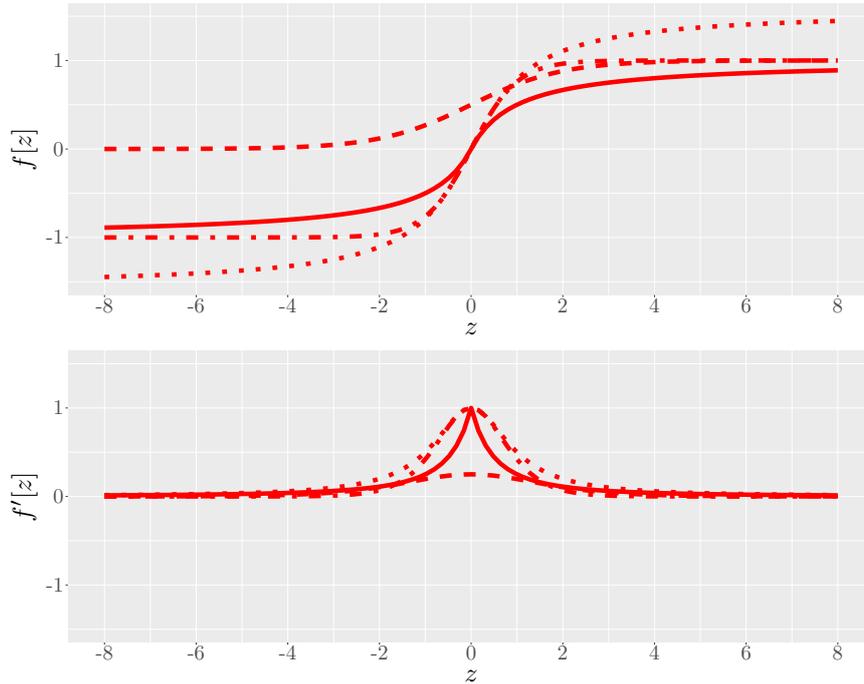

  \centering
 \scalebox{0.3}{\input{PlotSigmoid.tex}}
 \scalebox{0.3}{\input{PlotSigmoidD.tex}}
  \caption{\namelog\ (dashed), \namearctan\ (dotted), \nametanh\ (dotdashed), and \namesoft\ (solid) and their derivatives.
The functions mainly differ in their output range.
(The $y$-axis is scaled up to highlight the differences.)}
  \label{fig:sigmoid}
\end{figure}

\subsubsection{\namelog}
\label{logistic}

The function 
\begin{align*}
  \functionlog\ :\ \R\,&\to\,(0,1)\\
\argument\,&\mapsto\frac{1}{1+e^{-\argument}}
\end{align*}
is called \emph{logistic (sigmoid) function}  (\marginterm{\namelog}\emph{\namelog}).

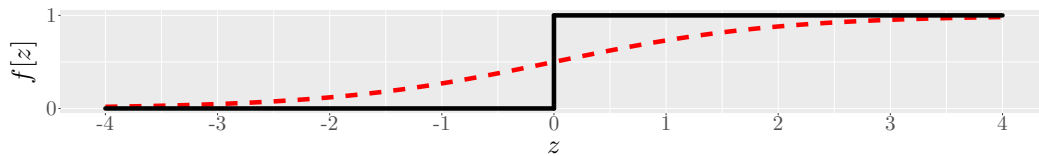
\begin{figure}[b]
  \centering
 \scalebox{0.3}{
\begin{tikzpicture}[x=1pt,y=1pt]
\definecolor{fillColor}{RGB}{255,255,255}
\path[use as bounding box,fill=fillColor,fill opacity=0.00] (0,0) rectangle (1300.86,195.13);
\begin{scope}
\path[clip] (  0.00,  0.00) rectangle (1300.86,195.13);
\definecolor{drawColor}{RGB}{255,255,255}
\definecolor{fillColor}{RGB}{255,255,255}

\path[draw=drawColor,line width= 0.6pt,line join=round,line cap=round,fill=fillColor] (  0.00,  0.00) rectangle (1300.86,195.13);
\end{scope}
\begin{scope}
\path[clip] ( 63.94, 60.89) rectangle (1295.36,189.63);
\definecolor{fillColor}{gray}{0.92}

\path[fill=fillColor] ( 63.94, 60.89) rectangle (1295.36,189.63);
\definecolor{drawColor}{RGB}{255,255,255}

\path[draw=drawColor,line width= 0.3pt,line join=round] ( 63.94,125.26) --
	(1295.36,125.26);

\path[draw=drawColor,line width= 0.3pt,line join=round] (189.88, 60.89) --
	(189.88,189.63);

\path[draw=drawColor,line width= 0.3pt,line join=round] (329.82, 60.89) --
	(329.82,189.63);

\path[draw=drawColor,line width= 0.3pt,line join=round] (469.75, 60.89) --
	(469.75,189.63);

\path[draw=drawColor,line width= 0.3pt,line join=round] (609.68, 60.89) --
	(609.68,189.63);

\path[draw=drawColor,line width= 0.3pt,line join=round] (749.62, 60.89) --
	(749.62,189.63);

\path[draw=drawColor,line width= 0.3pt,line join=round] (889.55, 60.89) --
	(889.55,189.63);

\path[draw=drawColor,line width= 0.3pt,line join=round] (1029.49, 60.89) --
	(1029.49,189.63);

\path[draw=drawColor,line width= 0.3pt,line join=round] (1169.42, 60.89) --
	(1169.42,189.63);

\path[draw=drawColor,line width= 0.6pt,line join=round] ( 63.94, 66.74) --
	(1295.36, 66.74);

\path[draw=drawColor,line width= 0.6pt,line join=round] ( 63.94,183.78) --
	(1295.36,183.78);

\path[draw=drawColor,line width= 0.6pt,line join=round] (119.91, 60.89) --
	(119.91,189.63);

\path[draw=drawColor,line width= 0.6pt,line join=round] (259.85, 60.89) --
	(259.85,189.63);

\path[draw=drawColor,line width= 0.6pt,line join=round] (399.78, 60.89) --
	(399.78,189.63);

\path[draw=drawColor,line width= 0.6pt,line join=round] (539.72, 60.89) --
	(539.72,189.63);

\path[draw=drawColor,line width= 0.6pt,line join=round] (679.65, 60.89) --
	(679.65,189.63);

\path[draw=drawColor,line width= 0.6pt,line join=round] (819.58, 60.89) --
	(819.58,189.63);

\path[draw=drawColor,line width= 0.6pt,line join=round] (959.52, 60.89) --
	(959.52,189.63);

\path[draw=drawColor,line width= 0.6pt,line join=round] (1099.45, 60.89) --
	(1099.45,189.63);

\path[draw=drawColor,line width= 0.6pt,line join=round] (1239.39, 60.89) --
	(1239.39,189.63);
\definecolor{drawColor}{RGB}{255,0,0}

\path[draw=drawColor,line width= 5.7pt,dash pattern=on 15pt off 15pt ,line join=round] (119.91, 68.85) --
	(131.11, 69.02) --
	(142.30, 69.20) --
	(153.50, 69.40) --
	(164.69, 69.62) --
	(175.89, 69.85) --
	(187.08, 70.10) --
	(198.28, 70.38) --
	(209.47, 70.67) --
	(220.67, 70.98) --
	(231.86, 71.32) --
	(243.06, 71.69) --
	(254.25, 72.08) --
	(265.45, 72.51) --
	(276.64, 72.96) --
	(287.84, 73.45) --
	(299.03, 73.97) --
	(310.22, 74.54) --
	(321.42, 75.14) --
	(332.61, 75.78) --
	(343.81, 76.47) --
	(355.00, 77.21) --
	(366.20, 78.00) --
	(377.39, 78.84) --
	(388.59, 79.74) --
	(399.78, 80.69) --
	(410.98, 81.70) --
	(422.17, 82.78) --
	(433.37, 83.92) --
	(444.56, 85.13) --
	(455.76, 86.40) --
	(466.95, 87.74) --
	(478.15, 89.16) --
	(489.34, 90.64) --
	(500.53, 92.20) --
	(511.73, 93.83) --
	(522.92, 95.53) --
	(534.12, 97.30) --
	(545.31, 99.15) --
	(556.51,101.05) --
	(567.70,103.02) --
	(578.90,105.06) --
	(590.09,107.15) --
	(601.29,109.29) --
	(612.48,111.48) --
	(623.68,113.71) --
	(634.87,115.97) --
	(646.07,118.27) --
	(657.26,120.59) --
	(668.46,122.92) --
	(679.65,125.26) --
	(690.85,127.60) --
	(702.04,129.93) --
	(713.23,132.25) --
	(724.43,134.54) --
	(735.62,136.81) --
	(746.82,139.04) --
	(758.01,141.23) --
	(769.21,143.37) --
	(780.40,145.46) --
	(791.60,147.49) --
	(802.79,149.46) --
	(813.99,151.37) --
	(825.18,153.21) --
	(836.38,154.98) --
	(847.57,156.69) --
	(858.77,158.32) --
	(869.96,159.87) --
	(881.16,161.36) --
	(892.35,162.77) --
	(903.54,164.12) --
	(914.74,165.39) --
	(925.93,166.60) --
	(937.13,167.74) --
	(948.32,168.81) --
	(959.52,169.83) --
	(970.71,170.78) --
	(981.91,171.68) --
	(993.10,172.52) --
	(1004.30,173.30) --
	(1015.49,174.04) --
	(1026.69,174.73) --
	(1037.88,175.38) --
	(1049.08,175.98) --
	(1060.27,176.54) --
	(1071.47,177.07) --
	(1082.66,177.56) --
	(1093.86,178.01) --
	(1105.05,178.43) --
	(1116.24,178.83) --
	(1127.44,179.19) --
	(1138.63,179.53) --
	(1149.83,179.85) --
	(1161.02,180.14) --
	(1172.22,180.41) --
	(1183.41,180.66) --
	(1194.61,180.90) --
	(1205.80,181.11) --
	(1217.00,181.31) --
	(1228.19,181.50) --
	(1239.39,181.67);
\definecolor{drawColor}{RGB}{0,0,0}

\path[draw=drawColor,line width= 5.7pt,line join=round,line cap=round] (119.91, 66.74) -- (679.65, 66.74);

\path[draw=drawColor,line width= 5.7pt,line join=round,line cap=round] (679.65, 66.74) -- (679.65,183.78);

\path[draw=drawColor,line width= 5.7pt,line join=round,line cap=round] (679.65,183.78) -- (1239.39,183.78);
\end{scope}
\begin{scope}
\path[clip] (  0.00,  0.00) rectangle (1300.86,195.13);
\definecolor{drawColor}{gray}{0.30}

\node[text=drawColor,anchor=base east,inner sep=0pt, outer sep=0pt, scale=  2.40] at ( 58.99, 58.48) {0};

\node[text=drawColor,anchor=base east,inner sep=0pt, outer sep=0pt, scale=  2.40] at ( 58.99,175.51) {1};
\end{scope}
\begin{scope}
\path[clip] (  0.00,  0.00) rectangle (1300.86,195.13);
\definecolor{drawColor}{gray}{0.20}

\path[draw=drawColor,line width= 0.6pt,line join=round] ( 61.19, 66.74) --
	( 63.94, 66.74);

\path[draw=drawColor,line width= 0.6pt,line join=round] ( 61.19,183.78) --
	( 63.94,183.78);
\end{scope}
\begin{scope}
\path[clip] (  0.00,  0.00) rectangle (1300.86,195.13);
\definecolor{drawColor}{gray}{0.20}

\path[draw=drawColor,line width= 0.6pt,line join=round] (119.91, 58.14) --
	(119.91, 60.89);

\path[draw=drawColor,line width= 0.6pt,line join=round] (259.85, 58.14) --
	(259.85, 60.89);

\path[draw=drawColor,line width= 0.6pt,line join=round] (399.78, 58.14) --
	(399.78, 60.89);

\path[draw=drawColor,line width= 0.6pt,line join=round] (539.72, 58.14) --
	(539.72, 60.89);

\path[draw=drawColor,line width= 0.6pt,line join=round] (679.65, 58.14) --
	(679.65, 60.89);

\path[draw=drawColor,line width= 0.6pt,line join=round] (819.58, 58.14) --
	(819.58, 60.89);

\path[draw=drawColor,line width= 0.6pt,line join=round] (959.52, 58.14) --
	(959.52, 60.89);

\path[draw=drawColor,line width= 0.6pt,line join=round] (1099.45, 58.14) --
	(1099.45, 60.89);

\path[draw=drawColor,line width= 0.6pt,line join=round] (1239.39, 58.14) --
	(1239.39, 60.89);
\end{scope}
\begin{scope}
\path[clip] (  0.00,  0.00) rectangle (1300.86,195.13);
\definecolor{drawColor}{gray}{0.30}

\node[text=drawColor,anchor=base,inner sep=0pt, outer sep=0pt, scale=  2.40] at (119.91, 39.41) {-4};

\node[text=drawColor,anchor=base,inner sep=0pt, outer sep=0pt, scale=  2.40] at (259.85, 39.41) {-3};

\node[text=drawColor,anchor=base,inner sep=0pt, outer sep=0pt, scale=  2.40] at (399.78, 39.41) {-2};

\node[text=drawColor,anchor=base,inner sep=0pt, outer sep=0pt, scale=  2.40] at (539.72, 39.41) {-1};

\node[text=drawColor,anchor=base,inner sep=0pt, outer sep=0pt, scale=  2.40] at (679.65, 39.41) {0};

\node[text=drawColor,anchor=base,inner sep=0pt, outer sep=0pt, scale=  2.40] at (819.58, 39.41) {1};

\node[text=drawColor,anchor=base,inner sep=0pt, outer sep=0pt, scale=  2.40] at (959.52, 39.41) {2};

\node[text=drawColor,anchor=base,inner sep=0pt, outer sep=0pt, scale=  2.40] at (1099.45, 39.41) {3};

\node[text=drawColor,anchor=base,inner sep=0pt, outer sep=0pt, scale=  2.40] at (1239.39, 39.41) {4};
\end{scope}
\begin{scope}
\path[clip] (  0.00,  0.00) rectangle (1300.86,195.13);
\definecolor{drawColor}{RGB}{0,0,0}

\node[text=drawColor,anchor=base,inner sep=0pt, outer sep=0pt, scale=  3.00] at (679.65, 11.33) {$z$};
\end{scope}
\begin{scope}
\path[clip] (  0.00,  0.00) rectangle (1300.86,195.13);
\definecolor{drawColor}{RGB}{0,0,0}

\node[text=drawColor,rotate= 90.00,anchor=base,inner sep=0pt, outer sep=0pt, scale=  3.00] at ( 26.16,125.26) {$f[z]$};
\end{scope}
\end{tikzpicture}}
  \caption{\namelog\ (red, dashed) is a smooth version of the binary function (black, solid)}
  \label{fig:step}
\end{figure}

\namelog\ activation can be thought of as a smooth version of \emph{binary activation},
which is the basis of the \emph{perceptron} \citep{Rosenblatt1958},
an  early example of a neural network.
The \marginterm{binary/step\\function}\emph{binary function} (or \emph{step function}) is
\begin{align*}
  \functionbinary\ :\ \R\,&\to\,\{0,1\}\\
\argument\,&\mapsto\,
  \begin{cases}
    \,1~~~~&\text{if}~\argument\geq 0\,;\\
    \,0~~~~&\text{otherwise}\,.\\
  \end{cases}
\end{align*}
\namelog\  approximates the binary function for large function values.
However,
while the binary function is not differentiable at~$0$,
\namelog\ is infinitely many times differentiable on its entire domain with  first and second derivatives (see Lemma~\ref{difflogistic}) 
\begin{multline*}
  \functionlogDF{\argument}=\functionlogF{\argument}\bigl(1-\functionlogF{\argument}\bigr)\in(0,1/4]\\
\text{and}~~~\functionlogDDF{\argument}=\functionlogF{\argument}\bigl(1-\functionlogF{\argument}\bigr)\bigl(1-2\functionlogF{\argument}\bigr)\in(-\swishbound,\swishbound)
\end{multline*}
for all $\argument\in\R$ and $\swishbound\approx 0.0962$.
Hence, \namelog\ is a smooth approximation of the binary function---see Figure~\ref{fig:step}.

\subsubsection{\namearctan}
\label{arctangent}
The function
\begin{align*}
  \functionarctan\ :\ \R\,&\to\,\Bigl(-\frac{\pi}{2},\frac{\pi}{2}\Bigr)\\
\argument\,&\mapsto\,{\arctan}[\argument]
\end{align*}
is called \marginterm{\namearctan}\emph{arcus tangens (\namearctan)}.
Recall the tangent function 
\begin{align*}
  {\tan}\ :\ \Bigl(-\frac{\pi}{2},\frac{\pi}{2}\Bigr)\,&\to\,\R\\
\argument\,&\mapsto\,\frac{{\sin}[\argument]}{{\cos}[\argument]}
\end{align*}
from basic trigonometry.
The value ${\arctan}[\argument]$ is defined as the inverse function of the tangent function, that is, ${\tan}[{\arctan}[\argument]]=\argument$ for all $\argument\in\R$.
This definition specifies the value of \namearctan\ uniquely,
because the tangent function is strictly increasing. 

An alternative notation for ${\arctan}[\argument]$ is ${\tan}\inv[\argument]$,
but this can cause confusion with the reciprocal (the multiplicative inverse):
$ {\tan}[\argument]\cdot({\tan}[\argument])\inv= 1$ for all $\argument\in(-\pi/2,\pi/2)$ but ${\tan}[\argument]\cdot{\tan}\inv[\argument]={\tan}[\argument]\cdot {\arctan}[\argument]=1$ only for two values $\argument\in(-\pi/2,\pi/2)$.

The output range $(0,1)$ of \namelog\ can be motivated biologically (``degree of neuron activation'').
The output  range $(-\pi/2,\pi/2)$ of \namearctan\ (and similarly the output ranges of \nametanh\ and \namesoft\ below) deviates from this motivation,
but the additional symmetry $\functionarctanF{\argument}=-\functionarctanF{-\argument}$ with respect to the origin is convenient from a mathematical and algorithmic perspective.

\namearctan\ is infinitely many times differentiable on its entire domain.
The first and second derivatives are  (see Lemma~\ref{diffarctan})
\begin{equation*}
  \functionarctanDF{\argument}=\frac{1}{1+\argument^2}\in[0,\infty)~~~~\text{and}~~~~\functionarctanDDF{\argument}=-\frac{2\argument}{(1+\argument^2)^2}\in\R
\end{equation*}
for all $\argument\in\R$.
The derivatives cannot be expressed by the original function,
but they are still comparably easy to compute.

\subsubsection{\nametanh}
\label{hyperbolictangent}

The function
\begin{align*}
  \functiontanh\ :\ \R\,&\to\,(-1,1)\\
\argument\,&\mapsto\,\operatorname{tanh}[\argument]\deq \frac{e^{\argument}-e^{-\argument}}{e^{\argument}+e^{-\argument}}
\end{align*}
is called  \marginterm{\nametanh}\emph{hyperbolic tangent function (\nametanh)}.

\nametanh\ is infinitely many times differentiable with first and second derivatives (see Lemma~\ref{difftanh})
\begin{equation*}
  \functiontanhDF{\argument}=1-\bigl(\functiontanhF{\argument}\bigr)^2\in(0,1)~~~\text{and}~~~\functiontanhDDF{\argument}=-2\functiontanhF{\argument}\Bigl(1-\bigl(\functiontanhF{\argument}\bigr)^2\Bigr)\in(-c,c)
\end{equation*}
  for all $\argument\in\R$ and $c\approx 0.770$.

Moreover (see Lemma~\ref{proptanh}),
\begin{equation*}
  \functiontanhF{\argument}=2\functionlogF{2\argument}-1~~~~\text{for all}~\argument\in\R\,,
\end{equation*}
and
the Taylor series of \nametanh\ and \namearctan\ agree up to the fourth order.
Thus,
we can think of \nametanh\ as a shifted and scaled version of \namelog\ or as an approximation of \namearctan\ (see Figure~\ref{fig:arctantanh}).
In particular,
\nametanh\ combines two popular features of \namelog\ and \namearctan:
the derivative of \nametanh\ is a simple expression of the original function (cf.~\namelog), 
and it is centered around zero (cf.~\namearctan).

\begin{figure}[t]
  \centering
 \scalebox{0.3}{
\begin{tikzpicture}[x=1pt,y=1pt]
\definecolor{fillColor}{RGB}{255,255,255}
\path[use as bounding box,fill=fillColor,fill opacity=0.00] (0,0) rectangle (505.89,433.62);
\begin{scope}
\path[clip] (  0.00,  0.00) rectangle (505.89,433.62);
\definecolor{drawColor}{RGB}{255,255,255}
\definecolor{fillColor}{RGB}{255,255,255}

\path[draw=drawColor,line width= 0.6pt,line join=round,line cap=round,fill=fillColor] (  0.00,  0.00) rectangle (505.89,433.62);
\end{scope}
\begin{scope}
\path[clip] ( 90.60, 60.89) rectangle (500.39,428.12);
\definecolor{fillColor}{gray}{0.92}

\path[fill=fillColor] ( 90.60, 60.89) rectangle (500.39,428.12);
\definecolor{drawColor}{RGB}{255,255,255}

\path[draw=drawColor,line width= 0.3pt,line join=round] ( 90.60,119.31) --
	(500.39,119.31);

\path[draw=drawColor,line width= 0.3pt,line join=round] ( 90.60,202.77) --
	(500.39,202.77);

\path[draw=drawColor,line width= 0.3pt,line join=round] ( 90.60,286.24) --
	(500.39,286.24);

\path[draw=drawColor,line width= 0.3pt,line join=round] ( 90.60,369.70) --
	(500.39,369.70);

\path[draw=drawColor,line width= 0.3pt,line join=round] (155.79, 60.89) --
	(155.79,428.12);

\path[draw=drawColor,line width= 0.3pt,line join=round] (248.93, 60.89) --
	(248.93,428.12);

\path[draw=drawColor,line width= 0.3pt,line join=round] (342.06, 60.89) --
	(342.06,428.12);

\path[draw=drawColor,line width= 0.3pt,line join=round] (435.20, 60.89) --
	(435.20,428.12);

\path[draw=drawColor,line width= 0.6pt,line join=round] ( 90.60, 77.58) --
	(500.39, 77.58);

\path[draw=drawColor,line width= 0.6pt,line join=round] ( 90.60,161.04) --
	(500.39,161.04);

\path[draw=drawColor,line width= 0.6pt,line join=round] ( 90.60,244.50) --
	(500.39,244.50);

\path[draw=drawColor,line width= 0.6pt,line join=round] ( 90.60,327.97) --
	(500.39,327.97);

\path[draw=drawColor,line width= 0.6pt,line join=round] ( 90.60,411.43) --
	(500.39,411.43);

\path[draw=drawColor,line width= 0.6pt,line join=round] (109.23, 60.89) --
	(109.23,428.12);

\path[draw=drawColor,line width= 0.6pt,line join=round] (202.36, 60.89) --
	(202.36,428.12);

\path[draw=drawColor,line width= 0.6pt,line join=round] (295.50, 60.89) --
	(295.50,428.12);

\path[draw=drawColor,line width= 0.6pt,line join=round] (388.63, 60.89) --
	(388.63,428.12);

\path[draw=drawColor,line width= 0.6pt,line join=round] (481.76, 60.89) --
	(481.76,428.12);
\definecolor{drawColor}{RGB}{0,0,0}

\path[draw=drawColor,line width= 5.7pt,line join=round] (109.23,113.40) --
	(112.95,115.09) --
	(116.68,116.81) --
	(120.40,118.56) --
	(124.13,120.35) --
	(127.85,122.18) --
	(131.58,124.04) --
	(135.30,125.94) --
	(139.03,127.88) --
	(142.76,129.86) --
	(146.48,131.87) --
	(150.21,133.93) --
	(153.93,136.03) --
	(157.66,138.16) --
	(161.38,140.34) --
	(165.11,142.56) --
	(168.83,144.82) --
	(172.56,147.13) --
	(176.28,149.47) --
	(180.01,151.86) --
	(183.73,154.30) --
	(187.46,156.77) --
	(191.19,159.29) --
	(194.91,161.85) --
	(198.64,164.46) --
	(202.36,167.11) --
	(206.09,169.80) --
	(209.81,172.54) --
	(213.54,175.31) --
	(217.26,178.13) --
	(220.99,180.99) --
	(224.71,183.89) --
	(228.44,186.82) --
	(232.16,189.80) --
	(235.89,192.81) --
	(239.61,195.85) --
	(243.34,198.93) --
	(247.07,202.04) --
	(250.79,205.19) --
	(254.52,208.36) --
	(258.24,211.55) --
	(261.97,214.78) --
	(265.69,218.02) --
	(269.42,221.29) --
	(273.14,224.57) --
	(276.87,227.87) --
	(280.59,231.18) --
	(284.32,234.50) --
	(288.04,237.83) --
	(291.77,241.17) --
	(295.50,244.50) --
	(299.22,247.84) --
	(302.95,251.18) --
	(306.67,254.51) --
	(310.40,257.83) --
	(314.12,261.14) --
	(317.85,264.44) --
	(321.57,267.72) --
	(325.30,270.99) --
	(329.02,274.23) --
	(332.75,277.45) --
	(336.47,280.65) --
	(340.20,283.82) --
	(343.92,286.96) --
	(347.65,290.08) --
	(351.38,293.16) --
	(355.10,296.20) --
	(358.83,299.21) --
	(362.55,302.19) --
	(366.28,305.12) --
	(370.00,308.02) --
	(373.73,310.88) --
	(377.45,313.70) --
	(381.18,316.47) --
	(384.90,319.21) --
	(388.63,321.90) --
	(392.35,324.55) --
	(396.08,327.15) --
	(399.81,329.72) --
	(403.53,332.24) --
	(407.26,334.71) --
	(410.98,337.15) --
	(414.71,339.54) --
	(418.43,341.88) --
	(422.16,344.19) --
	(425.88,346.45) --
	(429.61,348.67) --
	(433.33,350.85) --
	(437.06,352.98) --
	(440.78,355.08) --
	(444.51,357.13) --
	(448.23,359.15) --
	(451.96,361.13) --
	(455.69,363.07) --
	(459.41,364.97) --
	(463.14,366.83) --
	(466.86,368.65) --
	(470.59,370.44) --
	(474.31,372.20) --
	(478.04,373.92) --
	(481.76,375.61);
\definecolor{drawColor}{RGB}{255,0,0}

\path[draw=drawColor,line width= 5.7pt,dash pattern=on 15pt off 15pt on 6pt off 15pt ,line join=round] (109.23,117.38) --
	(112.95,118.80) --
	(116.68,120.27) --
	(120.40,121.78) --
	(124.13,123.33) --
	(127.85,124.94) --
	(131.58,126.59) --
	(135.30,128.28) --
	(139.03,130.03) --
	(142.76,131.82) --
	(146.48,133.66) --
	(150.21,135.55) --
	(153.93,137.49) --
	(157.66,139.49) --
	(161.38,141.53) --
	(165.11,143.62) --
	(168.83,145.77) --
	(172.56,147.96) --
	(176.28,150.21) --
	(180.01,152.51) --
	(183.73,154.86) --
	(187.46,157.26) --
	(191.19,159.71) --
	(194.91,162.21) --
	(198.64,164.76) --
	(202.36,167.37) --
	(206.09,170.02) --
	(209.81,172.71) --
	(213.54,175.46) --
	(217.26,178.25) --
	(220.99,181.08) --
	(224.71,183.96) --
	(228.44,186.88) --
	(232.16,189.84) --
	(235.89,192.84) --
	(239.61,195.88) --
	(243.34,198.95) --
	(247.07,202.06) --
	(250.79,205.19) --
	(254.52,208.36) --
	(258.24,211.56) --
	(261.97,214.78) --
	(265.69,218.02) --
	(269.42,221.29) --
	(273.14,224.57) --
	(276.87,227.87) --
	(280.59,231.18) --
	(284.32,234.50) --
	(288.04,237.83) --
	(291.77,241.17) --
	(295.50,244.50) --
	(299.22,247.84) --
	(302.95,251.18) --
	(306.67,254.51) --
	(310.40,257.83) --
	(314.12,261.14) --
	(317.85,264.44) --
	(321.57,267.72) --
	(325.30,270.99) --
	(329.02,274.23) --
	(332.75,277.45) --
	(336.47,280.65) --
	(340.20,283.81) --
	(343.92,286.95) --
	(347.65,290.06) --
	(351.38,293.13) --
	(355.10,296.17) --
	(358.83,299.17) --
	(362.55,302.13) --
	(366.28,305.05) --
	(370.00,307.93) --
	(373.73,310.76) --
	(377.45,313.55) --
	(381.18,316.30) --
	(384.90,318.99) --
	(388.63,321.64) --
	(392.35,324.24) --
	(396.08,326.80) --
	(399.81,329.30) --
	(403.53,331.75) --
	(407.26,334.15) --
	(410.98,336.50) --
	(414.71,338.80) --
	(418.43,341.05) --
	(422.16,343.24) --
	(425.88,345.39) --
	(429.61,347.48) --
	(433.33,349.52) --
	(437.06,351.51) --
	(440.78,353.46) --
	(444.51,355.35) --
	(448.23,357.19) --
	(451.96,358.98) --
	(455.69,360.73) --
	(459.41,362.42) --
	(463.14,364.07) --
	(466.86,365.67) --
	(470.59,367.23) --
	(474.31,368.74) --
	(478.04,370.21) --
	(481.76,371.63);
\end{scope}
\begin{scope}
\path[clip] (  0.00,  0.00) rectangle (505.89,433.62);
\definecolor{drawColor}{gray}{0.30}

\node[text=drawColor,anchor=base east,inner sep=0pt, outer sep=0pt, scale=  2.40] at ( 85.65, 69.32) {-1.0};

\node[text=drawColor,anchor=base east,inner sep=0pt, outer sep=0pt, scale=  2.40] at ( 85.65,152.78) {-0.5};

\node[text=drawColor,anchor=base east,inner sep=0pt, outer sep=0pt, scale=  2.40] at ( 85.65,236.24) {0.0};

\node[text=drawColor,anchor=base east,inner sep=0pt, outer sep=0pt, scale=  2.40] at ( 85.65,319.70) {0.5};

\node[text=drawColor,anchor=base east,inner sep=0pt, outer sep=0pt, scale=  2.40] at ( 85.65,403.16) {1.0};
\end{scope}
\begin{scope}
\path[clip] (  0.00,  0.00) rectangle (505.89,433.62);
\definecolor{drawColor}{gray}{0.20}

\path[draw=drawColor,line width= 0.6pt,line join=round] ( 87.85, 77.58) --
	( 90.60, 77.58);

\path[draw=drawColor,line width= 0.6pt,line join=round] ( 87.85,161.04) --
	( 90.60,161.04);

\path[draw=drawColor,line width= 0.6pt,line join=round] ( 87.85,244.50) --
	( 90.60,244.50);

\path[draw=drawColor,line width= 0.6pt,line join=round] ( 87.85,327.97) --
	( 90.60,327.97);

\path[draw=drawColor,line width= 0.6pt,line join=round] ( 87.85,411.43) --
	( 90.60,411.43);
\end{scope}
\begin{scope}
\path[clip] (  0.00,  0.00) rectangle (505.89,433.62);
\definecolor{drawColor}{gray}{0.20}

\path[draw=drawColor,line width= 0.6pt,line join=round] (109.23, 58.14) --
	(109.23, 60.89);

\path[draw=drawColor,line width= 0.6pt,line join=round] (202.36, 58.14) --
	(202.36, 60.89);

\path[draw=drawColor,line width= 0.6pt,line join=round] (295.50, 58.14) --
	(295.50, 60.89);

\path[draw=drawColor,line width= 0.6pt,line join=round] (388.63, 58.14) --
	(388.63, 60.89);

\path[draw=drawColor,line width= 0.6pt,line join=round] (481.76, 58.14) --
	(481.76, 60.89);
\end{scope}
\begin{scope}
\path[clip] (  0.00,  0.00) rectangle (505.89,433.62);
\definecolor{drawColor}{gray}{0.30}

\node[text=drawColor,anchor=base,inner sep=0pt, outer sep=0pt, scale=  2.40] at (109.23, 39.41) {-1.0};

\node[text=drawColor,anchor=base,inner sep=0pt, outer sep=0pt, scale=  2.40] at (202.36, 39.41) {-0.5};

\node[text=drawColor,anchor=base,inner sep=0pt, outer sep=0pt, scale=  2.40] at (295.50, 39.41) {0.0};

\node[text=drawColor,anchor=base,inner sep=0pt, outer sep=0pt, scale=  2.40] at (388.63, 39.41) {0.5};

\node[text=drawColor,anchor=base,inner sep=0pt, outer sep=0pt, scale=  2.40] at (481.76, 39.41) {1.0};
\end{scope}
\begin{scope}
\path[clip] (  0.00,  0.00) rectangle (505.89,433.62);
\definecolor{drawColor}{RGB}{0,0,0}

\node[text=drawColor,anchor=base,inner sep=0pt, outer sep=0pt, scale=  3.00] at (295.50, 11.33) {$z$};
\end{scope}
\begin{scope}
\path[clip] (  0.00,  0.00) rectangle (505.89,433.62);
\definecolor{drawColor}{RGB}{0,0,0}

\node[text=drawColor,rotate= 90.00,anchor=base,inner sep=0pt, outer sep=0pt, scale=  3.00] at ( 26.16,244.50) {$f[z]$};
\end{scope}
\end{tikzpicture}}
  \caption{\nametanh\ (red, dotdashed) and \namearctan\ (black, solid) almost coincide around~0}
  \label{fig:arctantanh}
\end{figure}
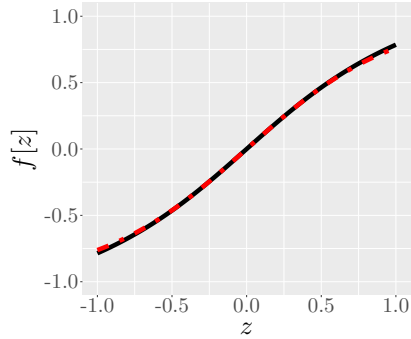

\subsubsection{\namesoft}
\label{sec:softsign}
The function 
\begin{align*}
  \functionsoft\ :\ \R\,&\to\,(-1,1)\\
\argument\,&\mapsto\frac{\argument}{1+\abs{\argument}}
\end{align*}
is called \emph{\nameelliott} or \marginterm{\namesoft}\emph{\namesoft}.
\namesoft\ activation was introduced in  \citet{Elliott1993}.

\namesoft\ is one time differentiable on its entire domain  with derivative  (see Lemma~\ref{diffsoft})
\begin{equation*}
  \functionsoftDF{\argument}=\bigl(1-\abs{\functionsoftF{\argument}}\bigr)^2\in (0,1)~~~~\text{for all}~\argument\in\R\,.
\end{equation*}
\namesoft\ is infinitely many times differentiable at all points except for $\argument=0$ with  second derivative 
\begin{equation*}
\functionsoftDDF{\argument}=-2 \signF{\argument}  \bigl(1-\abs{\functionsoftF{\argument}}\bigr)^3\in(-2,2)~~~~\text{for all}~\argument\in\R\setminus\{0\}\,.
\end{equation*}

Like \nametanh,
\namesoft\ is centered around zero and has  derivatives that are simple expressions of the original function.
A computational advantage of \namesoft\ over \nametanh\ is its simplicity, especially   the absence of exponential functions,
which makes the evaluation of~\functionsoft\ and~\functionsoftD\ particularly cheap.
Hence,
\namesoft\ is a particularly interesting candidate in the class of sigmoid functions.

\subsection{Piecewise-Linear Functions}
\label{piecewise}
A \marginterm{piecewise-linear\\function}\emph{piecewise-linear function} is composed of line segments.
Similarly as sigmoid activation, 
piecewise-linear activation has been used for many decades.
It has initially been motivated by neurobiological observations;
for example, the inhibiting effect of the activity of a visual-receptor unit on the activity of the neighboring units can be modeled by two line segments~\citep[Figure~2]{Hartline1957}.
The current popularity of piecewise-linear activation functions, however, originates in their computational properties:
First,
the functions and their derivatives (or directional derivatives---see~Section~\ref{sec:directionalderivatives}) are computationally inexpensive;
in particular, in contrast to most of the sigmoid functions discussed in the previous section, no exponential or trigonometric functions are involved.
Second, the derivatives of the functions discussed below do not converge to zero in the limit of infinitely large arguments (but they can be zero in the limit of infinitely small arguments),
which might alleviate the vanishing-gradient problem of sigmoid activation.
\later{make a reference here later to our paper}

Another interesting feature of the functions discussed below is that they are nonnegative homogenous: $\functionF{a\argument}=a\functionF{\argument}$ for all $a\in[0,\infty)$ and $\argument\in\R$.
This feature has turned out useful in deep-learning theory \citep{Neyshabur2015,Taheri20} and methodology \citep{Hebiri20}.

In the following, we discuss three piecewise-linear activation functions:
\namelinear, \namerelu, and \namelrelu.

\subsubsection{\namelinear}
In the context of neural-network activation,
\marginterm{\namelinear}\emph{\namelinear} is the name for the identity function
\begin{align*}
  \functionlinear\ :\ \R\,&\to\,(-\infty,\infty)\,;\\
  \argument\,&\mapsto\,\argument\,.
\end{align*}
Hence, \namelinear\ activation leaves the inputs unchanged.

Of course, 
$\functionlinear$ is differentiable on its entire domain with derivatives
\begin{equation*}
  \functionlinearDF{\argument}=1~~~\text{and}~~~\functionlinearDDF{\argument}=0~~~~~~~~\text{for all}~\argument\in\R\,.
\end{equation*}
It is often claimed that the constant derivatives  lead to uninformative updates in the optimization,
but this  neglects that objective functions in deep learning comprise not only the networks but also a loss function (such as cross-entropy, least-squares, or robust versions of it~\citep{Lederer2020b}).
An actual drawback of \namelinear\ activation is its low expressivity;
in particular,
networks where all activations are \namelinear\ are always linear functions---see Example~\ref{ex:linear}.
Hence, \namelinear\ activation is typically only applied   to certain layers, such as output layers in regression settings.

However, \namelinear-activation networks can also be used as toy models for analyzing and testing optimization algorithms;
indeed, such networks are simple from a modeling perspective but highly intricate from an optimization perspective~\citep{Arora2018}.
Consequently, one might hope that networks with \namelinear\ activation convey general principles of deep-learning optimization.

\subsubsection{\namerelu}
\label{sec:relu}
The function
\begin{align*}
  \functionrelu\ :\ \R\,&\to\,[0,\infty)\\
  \argument\,&\mapsto\,\max\{0,\argument\}
\end{align*}
is called the \emph{positive-part function} or \emph{ramp function}.
The positive-part function is the identity function (linear with slope~1) for positive arguments  and the constant function with value zero otherwise.

In electrical engineering,
a \emph{rectifier} is a device that converts alternating current to direct current.
Similarly, the positive-part function  lets positive inputs pass unaltered but cuts negative inputs, that is, it transforms negative and nonnegative inputs into nonnegative outputs.
Therefore, a neuron equipped with a positive-part function as the activation function is often called a \emph{rectifier linear unit (\namerelu)},
and the positive-part function itself is often called \emph{\namerelu}\marginterm{\namerelu} in the context of neural networks.

A clear benefit of \namerelu\ is that both the function itself and its derivatives are easy to implement and computationally inexpensive.
\namerelu\ is infinitely many times differentiable at $\argument\in\R\setminus\{0\}$
with first and second derivatives (see Lemma~\ref{piecewiseder})
\begin{equation*}
  \functionreluDF{\argument}=
  \begin{cases}
    1~~~\text{for all}~\argument\in(0,\infty)\\
0~~~\text{for all}~\argument\in(-\infty,0)
  \end{cases}~~~\text{and}~~~~~~\functionreluDDF{\argument}=0~~\text{for all}~\argument\in\R\setminus\{0\}\,.
\end{equation*}
\namerelu\ is not differentiable at~$\argument=0$,
but it is ``almost'' differentiable;
for example, 
the directional derivatives exist for all points $\argument\in\R$ and directions~$\direction\in\R$  and equal (see Lemma~\ref{piecewiseder})
\begin{equation*}
  \diffdirectional\functionreluF{\argument}=
  \begin{cases}
    \direction~~~\text{for all}~\direction\in\R~\text{and}~\argument\in(0,\infty)~\text{and for all}~\direction\in[0,\infty)~\text{and}~\argument=0\,;\\
0~~~\text{otherwise}\,.
  \end{cases}
\end{equation*}
This result motivates using gradient-descent-type approaches in practice with the derivatives at zero (which do not exist) replaced by a fixed value between~$0$ and~$1$.

\namerelu\ activation can be subject to the dying-relu phenomenon, which is a version of the vanishing-gradient problem.
The dying-\namerelu\ phenomenon indicates a situation where  many \namerelu\ nodes are inactive during much of the training process---see Example~\ref{dyingrelu}---which can  prevent the algorithms from learning complex models.
Potential remedies are to choose lower learning rates or to replace \namerelu\ by \namelrelu---see the next section.
However, the relevance of the dying-relu phenomenon in practice, as well as the effectiveness of the mentioned remedies, remain unclear.

\jl{maybe say that sparstiy kinda unclear, much better to do regularization}

\subsubsection{\namelrelu}
Given a parameter $\functionpar\in[0,\infty)$, 
we call \marginterm{\namelrelu}\emph{\namelrelu} the function
\begin{align*}
  \functionlrelu\ :\ \R\,&\to\,[0,\infty)\,;\\
  \argument\,&\mapsto\,\max\{0,\argument\}+\min\{0,\functionpar\argument\}\,.
\end{align*}
\namelrelu\  activation was introduced in \citet{Maas2013}.

The idea behind \namelrelu\ is to mimic \namerelu\ but to avoid the dying-relu phenomenon.
\namelrelu\ equals \namerelu\ in the case $\functionpar=0$;
for positive parameters~\functionpar, however, the functions differ for negative inputs,
most notably in their derivatives.
\namelrelu\ is infinitely many times differentiable at $\argument\in\R\setminus\{0\}$
with first and second derivatives (see Lemma~\ref{piecewiseder})
\begin{equation*}
  \functionlreluDF{\argument}=
  \begin{cases}
    1~~~\text{for all}~\argument\in(0,\infty)\\
\functionpar~~~\text{for all}~\argument\in(-\infty,0)
  \end{cases}~~~\text{and}~~~~~~\functionlreluDDF{\argument}=0~~\text{for all}~\argument\in\R\setminus\{0\}\,.
\end{equation*}
\namelrelu\ is not differentiable at~$\argument=0$,
but
the directional derivatives exist for all points $\argument\in\R$ and directions~$\direction\in\R$  and equal (see Lemma~\ref{piecewiseder})
\begin{equation*}
  \diffdirectional\functionlreluF{\argument}=
  \begin{cases}
    \direction~~~&\text{for all}~\direction\in\R~\text{and}~\argument\in(0,\infty)~\text{and for all}~\direction\in[0,\infty)~\text{and}~\argument=0\,;\\
\functionpar\direction~~~&\text{otherwise}\,.
  \end{cases}
\end{equation*}
The key difference to \namerelu\ is that $\functionlreluDF{\argument}>0$ for all $\argument\in\R\setminus\{0\}$ and $\functionpar>0$.
This property can be seen as an approach to avoid the problem of vanishing gradients.

A practical challenge inflicted by \namelrelu\ is choosing the parameter~\functionpar.
As  we have just seen, $\functionpar$ is the slope of  \namelrelu\ for negative inputs.  
It is usually chosen between 0 (where \namelrelu\ equals \namerelu) and 1 (where \namelrelu\ equals \namelinear),
but there is no further consensus:
the original paper sets the parameter to $\functionpar=0.01$;
\citep{Xu2015} introduces \emph{randomized leaky rectifier linear unit} (\namerrelu),
which replaces the fixed parameter with a stochastic one,
and \citep{He2015} suggests to train the parameter (see Section~\ref{learning}).
The benefits of these choices as compared to each other and as compared to vanilla \namerelu\ still need to be explored.


\subsection{Other Functions}
\label{sec:other}
Since each of the discussed activations has certain shortcomings,
proposals of new activations have mushroomed.
Most of these functions lack empirical or mathematical support,
but there are notable exceptions.
In the following,
we discuss recently proposed activations that are popular or include novel ideas.
This includes \namesoftplus,
\nameelu\ and \nameselu,
\nameswish,
and activations with parameters that are learned during training.

\later{Other functions (all downloaded): SReLU \citep{Jin2015} BinaryConnect: \citep{Courbariaux2015}.}

\subsubsection{\namesoftplus}
\label{softplus}
\namesoftplus\ is defined as\marginterm{\namesoftplus} 
\begin{align*}
  \functionsoftplus\ :\ \R\,&\to\,[0,\infty)\,;\\
  \argument\,&\mapsto\,\log[1+e^{\argument}]\,.
\end{align*}
\namesoftplus\ activation was introduced  in~\citet{Dugas2001}.

\namesoftplus\ is infinitely many times differentiable on its entire domain,
and its first and second derivatives are (see Lemma~\ref{dersoftplus})
\begin{equation*}
  \functionsoftplusDF{\argument}=\functionlogF{\argument}\in(0,1)~~~~\text{and}~~~~\functionsoftplusDDF{\argument}= \functionlogDF{\argument}\in(0,1/4]
\end{equation*}
for all $\argument\in\R$. 
Hence, \namesoftplus\ is a primitive of \namelog.

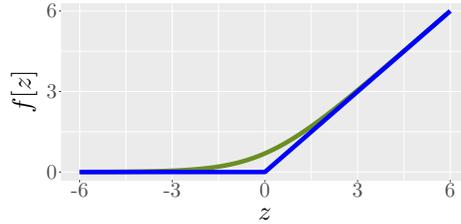
\begin{figure}[t]
  \centering
 \scalebox{0.3}{
\begin{tikzpicture}[x=1pt,y=1pt]
\definecolor{fillColor}{RGB}{255,255,255}
\path[use as bounding box,fill=fillColor,fill opacity=0.00] (0,0) rectangle (578.16,289.08);
\begin{scope}
\path[clip] (  0.00,  0.00) rectangle (578.16,289.08);
\definecolor{drawColor}{RGB}{255,255,255}
\definecolor{fillColor}{RGB}{255,255,255}

\path[draw=drawColor,line width= 0.6pt,line join=round,line cap=round,fill=fillColor] (  0.00,  0.00) rectangle (578.16,289.08);
\end{scope}
\begin{scope}
\path[clip] ( 63.94, 60.89) rectangle (572.66,283.58);
\definecolor{fillColor}{gray}{0.92}

\path[fill=fillColor] ( 63.94, 60.89) rectangle (572.66,283.58);
\definecolor{drawColor}{RGB}{255,255,255}

\path[draw=drawColor,line width= 0.3pt,line join=round] ( 63.94,121.62) --
	(572.66,121.62);

\path[draw=drawColor,line width= 0.3pt,line join=round] ( 63.94,222.85) --
	(572.66,222.85);

\path[draw=drawColor,line width= 0.3pt,line join=round] (144.87, 60.89) --
	(144.87,283.58);

\path[draw=drawColor,line width= 0.3pt,line join=round] (260.49, 60.89) --
	(260.49,283.58);

\path[draw=drawColor,line width= 0.3pt,line join=round] (376.11, 60.89) --
	(376.11,283.58);

\path[draw=drawColor,line width= 0.3pt,line join=round] (491.73, 60.89) --
	(491.73,283.58);

\path[draw=drawColor,line width= 0.6pt,line join=round] ( 63.94, 71.01) --
	(572.66, 71.01);

\path[draw=drawColor,line width= 0.6pt,line join=round] ( 63.94,172.23) --
	(572.66,172.23);

\path[draw=drawColor,line width= 0.6pt,line join=round] ( 63.94,273.46) --
	(572.66,273.46);

\path[draw=drawColor,line width= 0.6pt,line join=round] ( 87.06, 60.89) --
	( 87.06,283.58);

\path[draw=drawColor,line width= 0.6pt,line join=round] (202.68, 60.89) --
	(202.68,283.58);

\path[draw=drawColor,line width= 0.6pt,line join=round] (318.30, 60.89) --
	(318.30,283.58);

\path[draw=drawColor,line width= 0.6pt,line join=round] (433.92, 60.89) --
	(433.92,283.58);

\path[draw=drawColor,line width= 0.6pt,line join=round] (549.54, 60.89) --
	(549.54,283.58);
\definecolor{drawColor}{RGB}{107,142,35}

\path[draw=drawColor,line width= 5.7pt,line join=round] ( 87.06, 71.09) --
	( 91.69, 71.10) --
	( 96.31, 71.12) --
	(100.94, 71.13) --
	(105.56, 71.15) --
	(110.19, 71.16) --
	(114.81, 71.18) --
	(119.44, 71.20) --
	(124.06, 71.23) --
	(128.69, 71.26) --
	(133.31, 71.29) --
	(137.94, 71.32) --
	(142.56, 71.36) --
	(147.19, 71.41) --
	(151.81, 71.46) --
	(156.44, 71.51) --
	(161.06, 71.58) --
	(165.68, 71.65) --
	(170.31, 71.73) --
	(174.93, 71.82) --
	(179.56, 71.92) --
	(184.18, 72.03) --
	(188.81, 72.16) --
	(193.43, 72.31) --
	(198.06, 72.47) --
	(202.68, 72.65) --
	(207.31, 72.85) --
	(211.93, 73.08) --
	(216.56, 73.34) --
	(221.18, 73.62) --
	(225.81, 73.94) --
	(230.43, 74.30) --
	(235.06, 74.69) --
	(239.68, 75.14) --
	(244.30, 75.63) --
	(248.93, 76.17) --
	(253.55, 76.78) --
	(258.18, 77.45) --
	(262.80, 78.19) --
	(267.43, 79.00) --
	(272.05, 79.89) --
	(276.68, 80.88) --
	(281.30, 81.95) --
	(285.93, 83.12) --
	(290.55, 84.39) --
	(295.18, 85.77) --
	(299.80, 87.26) --
	(304.43, 88.87) --
	(309.05, 90.59) --
	(313.68, 92.43) --
	(318.30, 94.40) --
	(322.93, 96.48) --
	(327.55, 98.69) --
	(332.17,101.02) --
	(336.80,103.46) --
	(341.42,106.02) --
	(346.05,108.69) --
	(350.67,111.46) --
	(355.30,114.34) --
	(359.92,117.32) --
	(364.55,120.38) --
	(369.17,123.54) --
	(373.80,126.77) --
	(378.42,130.08) --
	(383.05,133.46) --
	(387.67,136.91) --
	(392.30,140.41) --
	(396.92,143.97) --
	(401.55,147.57) --
	(406.17,151.23) --
	(410.79,154.92) --
	(415.42,158.65) --
	(420.04,162.41) --
	(424.67,166.21) --
	(429.29,170.03) --
	(433.92,173.87) --
	(438.54,177.74) --
	(443.17,181.63) --
	(447.79,185.53) --
	(452.42,189.45) --
	(457.04,193.39) --
	(461.67,197.34) --
	(466.29,201.29) --
	(470.92,205.26) --
	(475.54,209.24) --
	(480.17,213.23) --
	(484.79,217.22) --
	(489.42,221.22) --
	(494.04,225.22) --
	(498.66,229.23) --
	(503.29,233.24) --
	(507.91,237.26) --
	(512.54,241.28) --
	(517.16,245.31) --
	(521.79,249.34) --
	(526.41,253.37) --
	(531.04,257.40) --
	(535.66,261.43) --
	(540.29,265.47) --
	(544.91,269.50);
\definecolor{drawColor}{RGB}{0,0,255}

\path[draw=drawColor,line width= 5.7pt,line join=round] ( 87.06, 71.01) --
	( 91.69, 71.01) --
	( 96.31, 71.01) --
	(100.94, 71.01) --
	(105.56, 71.01) --
	(110.19, 71.01) --
	(114.81, 71.01) --
	(119.44, 71.01) --
	(124.06, 71.01) --
	(128.69, 71.01) --
	(133.31, 71.01) --
	(137.94, 71.01) --
	(142.56, 71.01) --
	(147.19, 71.01) --
	(151.81, 71.01) --
	(156.44, 71.01) --
	(161.06, 71.01) --
	(165.68, 71.01) --
	(170.31, 71.01) --
	(174.93, 71.01) --
	(179.56, 71.01) --
	(184.18, 71.01) --
	(188.81, 71.01) --
	(193.43, 71.01) --
	(198.06, 71.01) --
	(202.68, 71.01) --
	(207.31, 71.01) --
	(211.93, 71.01) --
	(216.56, 71.01) --
	(221.18, 71.01) --
	(225.81, 71.01) --
	(230.43, 71.01) --
	(235.06, 71.01) --
	(239.68, 71.01) --
	(244.30, 71.01) --
	(248.93, 71.01) --
	(253.55, 71.01) --
	(258.18, 71.01) --
	(262.80, 71.01) --
	(267.43, 71.01) --
	(272.05, 71.01) --
	(276.68, 71.01) --
	(281.30, 71.01) --
	(285.93, 71.01) --
	(290.55, 71.01) --
	(295.18, 71.01) --
	(299.80, 71.01) --
	(304.43, 71.01) --
	(309.05, 71.01) --
	(313.68, 71.01) --
	(318.30, 71.01) --
	(322.93, 75.06) --
	(327.55, 79.11) --
	(332.17, 83.16) --
	(336.80, 87.21) --
	(341.42, 91.26) --
	(346.05, 95.30) --
	(350.67, 99.35) --
	(355.30,103.40) --
	(359.92,107.45) --
	(364.55,111.50) --
	(369.17,115.55) --
	(373.80,119.60) --
	(378.42,123.65) --
	(383.05,127.70) --
	(387.67,131.74) --
	(392.30,135.79) --
	(396.92,139.84) --
	(401.55,143.89) --
	(406.17,147.94) --
	(410.79,151.99) --
	(415.42,156.04) --
	(420.04,160.09) --
	(424.67,164.14) --
	(429.29,168.19) --
	(433.92,172.23) --
	(438.54,176.28) --
	(443.17,180.33) --
	(447.79,184.38) --
	(452.42,188.43) --
	(457.04,192.48) --
	(461.67,196.53) --
	(466.29,200.58) --
	(470.92,204.63) --
	(475.54,208.67) --
	(480.17,212.72) --
	(484.79,216.77) --
	(489.42,220.82) --
	(494.04,224.87) --
	(498.66,228.92) --
	(503.29,232.97) --
	(507.91,237.02) --
	(512.54,241.07) --
	(517.16,245.12) --
	(521.79,249.16) --
	(526.41,253.21) --
	(531.04,257.26) --
	(535.66,261.31) --
	(540.29,265.36) --
	(544.91,269.41) --
	(549.54,273.46);
\end{scope}
\begin{scope}
\path[clip] (  0.00,  0.00) rectangle (578.16,289.08);
\definecolor{drawColor}{gray}{0.30}

\node[text=drawColor,anchor=base east,inner sep=0pt, outer sep=0pt, scale=  2.40] at ( 58.99, 62.75) {0};

\node[text=drawColor,anchor=base east,inner sep=0pt, outer sep=0pt, scale=  2.40] at ( 58.99,163.97) {3};

\node[text=drawColor,anchor=base east,inner sep=0pt, outer sep=0pt, scale=  2.40] at ( 58.99,265.19) {6};
\end{scope}
\begin{scope}
\path[clip] (  0.00,  0.00) rectangle (578.16,289.08);
\definecolor{drawColor}{gray}{0.20}

\path[draw=drawColor,line width= 0.6pt,line join=round] ( 61.19, 71.01) --
	( 63.94, 71.01);

\path[draw=drawColor,line width= 0.6pt,line join=round] ( 61.19,172.23) --
	( 63.94,172.23);

\path[draw=drawColor,line width= 0.6pt,line join=round] ( 61.19,273.46) --
	( 63.94,273.46);
\end{scope}
\begin{scope}
\path[clip] (  0.00,  0.00) rectangle (578.16,289.08);
\definecolor{drawColor}{gray}{0.20}

\path[draw=drawColor,line width= 0.6pt,line join=round] ( 87.06, 58.14) --
	( 87.06, 60.89);

\path[draw=drawColor,line width= 0.6pt,line join=round] (202.68, 58.14) --
	(202.68, 60.89);

\path[draw=drawColor,line width= 0.6pt,line join=round] (318.30, 58.14) --
	(318.30, 60.89);

\path[draw=drawColor,line width= 0.6pt,line join=round] (433.92, 58.14) --
	(433.92, 60.89);

\path[draw=drawColor,line width= 0.6pt,line join=round] (549.54, 58.14) --
	(549.54, 60.89);
\end{scope}
\begin{scope}
\path[clip] (  0.00,  0.00) rectangle (578.16,289.08);
\definecolor{drawColor}{gray}{0.30}

\node[text=drawColor,anchor=base,inner sep=0pt, outer sep=0pt, scale=  2.40] at ( 87.06, 39.41) {-6};

\node[text=drawColor,anchor=base,inner sep=0pt, outer sep=0pt, scale=  2.40] at (202.68, 39.41) {-3};

\node[text=drawColor,anchor=base,inner sep=0pt, outer sep=0pt, scale=  2.40] at (318.30, 39.41) {0};

\node[text=drawColor,anchor=base,inner sep=0pt, outer sep=0pt, scale=  2.40] at (433.92, 39.41) {3};

\node[text=drawColor,anchor=base,inner sep=0pt, outer sep=0pt, scale=  2.40] at (549.54, 39.41) {6};
\end{scope}
\begin{scope}
\path[clip] (  0.00,  0.00) rectangle (578.16,289.08);
\definecolor{drawColor}{RGB}{0,0,0}

\node[text=drawColor,anchor=base,inner sep=0pt, outer sep=0pt, scale=  3.00] at (318.30, 11.33) {$z$};
\end{scope}
\begin{scope}
\path[clip] (  0.00,  0.00) rectangle (578.16,289.08);
\definecolor{drawColor}{RGB}{0,0,0}

\node[text=drawColor,rotate= 90.00,anchor=base,inner sep=0pt, outer sep=0pt, scale=  3.00] at ( 26.16,172.23) {$f[z]$};
\end{scope}
\end{tikzpicture}}
  \caption{\namesoftplus\ (green) is a smooth version of \namerelu\ (blue). 
(Best seen in color.)}
  \label{fig:softplusrelu}
\end{figure}

However, $\functionsoftplusF{\argument}\approx\log[1]=0=\functionreluF{\argument}$ for all $\argument\ll 0$ and  $\functionsoftplusF{\argument}\approx \log[e^{\argument}]=\argument=\functionreluF{\argument}$ for all $\argument\gg 0$.
So rather than comparing \namesoftplus\ to \namelog,
we should consider it a smooth version of \namerelu---see Figure~\ref{fig:softplusrelu}.
The differentiability at $\argument=0$ is a mathematical convenience of \namesoftplus\ as compared to \namerelu,
but the practical relevance of this feature is unclear:
practical implementations seem to work with the derivative of \namerelu\ in $\argument=0$ just set to zero (cf.~the discussion of the directional derivatives of \namerelu\ in Section~\ref{sec:relu}).
The strict positivity of the first derivative of \namesoftplus\ 
can be seen as a measure against the dying-\namerelu\ phenomenon,
but the practical effect of this measure is again unclear:
the derivatives are still small for small arguments, that is,
$\functionsoftplusDF{\argument}\approx 0$ for all $\argument\ll 0$.
A disadvantage of \namesoftplus\ is that the function and its derivative are both computationally more costly than their \namerelu\ counterparts.
Thus,
in lack of clear evidence in favor of~\namesoftplus,
and given \namerelu's computational simplicity,
\namerelu\ is currently favored over \namesoftplus.


\subsubsection{\nameelu\ and \nameselu}\label{elu}

Given a parameter $\functionpar\in[0,\infty)$, 
the function
\begin{align*}
  \functionelu\ :\ \R\,&\to\,(-\functionpar,\infty)\\
  \argument\,&\mapsto\,
               \begin{cases}
                 \argument~~~~&\text{for all}~\argument\in[0,\infty)\\
                 \functionpar(e^{\argument}-1)~~~~&\text{for all}~\argument\in(-\infty,0)
               \end{cases}
\end{align*}
is called---in analogy with \namerelu---\emph{exponential linear unit (\nameelu)}\marginterm{\nameelu}.
(See Lemma~\ref{eluproplem} for a calculation of the output ranges.)
In fact, \namerelu\ is a special case of \nameelu: $\functionrelu=\functioneluz$.
\nameelu\ activation was introduced in \citet{Clevert2015}.

\nameelu\ is infinitely many times differentiable on~$\R\setminus\{0\}$ with first and second derivatives
\begin{equation*}
  \functioneluDF{\argument}=
  \begin{cases}
    1&\text{for all}~\argument\in(0,\infty)\\
\functioneluF{\argument}+\functionpar&\text{for all}~\argument\in(-\infty,0)
  \end{cases}
\end{equation*}
and
\begin{equation*}
\functioneluDDF{\argument}=
  \begin{cases}
    0&\text{for all}~\argument\in(0,\infty)\,;\\
\functioneluF{\argument}+\functionpar&\text{for all}~\argument\in(-\infty,0)\,.
  \end{cases}
\end{equation*}
The first directional derivatives exist on the entire real line and are equal to
\begin{equation*}
  \diffdirectional\functioneluF{\argument}=
  \begin{cases}
    \direction&\text{for all}~\direction\in\R~\text{and}~\argument\in(0,\infty)~\text{or}~\direction\in[0,\infty)~\text{and}~\argument=0\,;\\
\direction\bigl(\functioneluF{\argument}+\functionpar\bigr)&\text{otherwise}\,.
  \end{cases}
\end{equation*}
Similarly as discussed for \namesoftplus\ in the previous section,
a key difference of \nameelu\ to \namerelu\ is the fact that the derivatives of \nameelu\ with $\functionpar\neq 0$ are strictly positive (where they exist).
But again,
the practical effect of this property is unclear.
A difference to both \namesoftplus\ and \namerelu\  is that \nameelu\ is somewhat ``centered'' around zero---see Figure~\ref{fig:elu}.

The mathematically most convenient parameter is $\functionpar=1$,
because this is the only choice that makes \nameelu\ one time differentiable on the entire real line (but not twice differentiable).
But except for this observation,
there is little insight into how to choose~\functionpar\ in practice.

Observe also that the first derivatives of~\nameelu\ can be computed easily from the original functions,
but the original functions  involve an exponential function and, therefore, are more costly to compute than  \namerelu.
Hence, in view of the unclear practical benefits and the computational disadvantages of \nameelu,
\namerelu\ is currently preferred.

\begin{figure}[t]
  \centering
 \scalebox{0.3}{\input{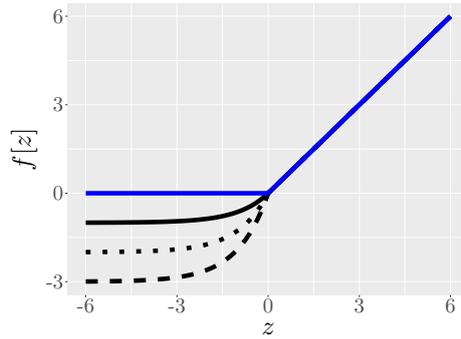}}
  \caption{\namerelu, that is, \nameelu\ with $\functionpar=0$ (blue, dashed) and \nameelu\ with $\functionpar=1$ (black, solid), $\functionpar=2$ (black, dotted), and $\functionpar=3$ (black, dashed) coincide for positive values but differ for negative values.
(Best seen in color.)}
  \label{fig:elu}
\end{figure}

A variant of \nameelu\ is \emph{scaled exponential linear unit (\nameselu)}\marginterm{\nameselu} introduced in~\citet{Klambauer2017}.
\nameselu\ is (see Section~\ref{seludetails} for further details)
\begin{align*}
  \functionselu\ :\ \R\,&\to\,[-\functionparP_0,\infty)\\
  \argument\,&\mapsto\,
               \begin{cases}
                 \functionpar_0\argument~~~~&\text{for all}~\argument\in[0,\infty)\\
                 \functionparP_0(e^{\argument}-1)~~~~&\text{for all}~\argument\in(-\infty,0)
               \end{cases}
\end{align*}
for fixed parameters $\functionpar_0\approx 1.05$ and $\functionparP_0\approx1.76$.
\nameselu\ is very similar to \nameelu\ in general,
but the specific choice of the parameters leads to an additional self-scaling property.
Consider the map
\begin{align*}
  \mathfrak{c}\ : \ \R^2\,&\to\,\R^2\,;\\
(\mu,\nu)\,&\mapsto\,\Bigl( E\bigl[\functionseluF{r_{\mu,\nu}}\bigr],E\bigl[(\functionseluF{r_{\mu,\nu}})^2\bigr]\Bigr)\,,
\end{align*}
where $r_{\mu,\nu}\sim\mathcal N[\mu,\sqrt{\nu}]$ is a univariate Gaussian random variable with mean~$\mu$ and variance~$\nu$,
and $E$ is the corresponding expectation.
In other words, the function $\mathfrak{c}$ captures the mean and variance of a normal random variable after \nameselu\ activation.
It is easy to show that $\mathfrak{c}$ is a contraction with fixed point $(\overline{\mu},\overline{\nu})=(0,1)$,
which can be interpreted as a normalization property:
broadly speaking, a Gaussian input is transformed into an output with mean closer to~$0$ and variance closer to~$1$.
One can argue that this normalization property might alleviate vanishing and exploding gradients especially in networks that are wide (where the central-limit theorem might indeed ensure Gaussian-like inputs) and deep (where many gradients are multiplied together).
However, the practical benefits, especially in view of batch normalization and other normalization schemes that are often used in deep-learning pipelines, are currently unclear.

\subsubsection{\nameswish}
\label{sec:swish}

\begin{figure}[b!]
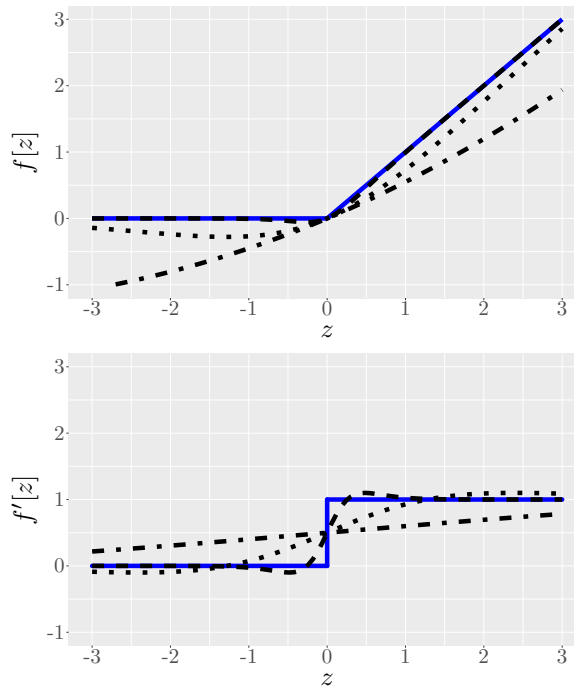

  \centering
 \scalebox{0.3}{\input{PlotSwishRelu.tex}}

\scalebox{0.3}{\input{PlotSwishReluD.tex}}
  \caption{\nameswish\  with $\functionpar=0.2$ (black, dotdashed), $\functionpar=1$ (black, dotted), and $\functionpar=5$ (blue, dashed) as well as  \namerelu\ (blue, solid) and their derivatives.
The larger the function parameter~\functionpar,
the more \nameswish\ resembles \namerelu;
the smaller the function parameter,
the more \nameswish\ resembles (a scaled version of) \namelinear.
}
  \label{fig:swishrelu}
\end{figure}

Given a parameter $\functionpar\in[0,\infty)$, 
the function\marginterm{\nameswish}
\begin{align*}
  \functionswish\ :\ \R\,&\to\,[\swishboundA,\infty)\\
  \argument\,&\mapsto\,\argument\cdot\functionlogF{\functionpar\argument}=\frac{\argument}{1+e^{-\functionpar\argument}}
\end{align*}
is called \emph{\nameswish}.
The minimum of the function is $\swishboundA\approx -0.278/\functionpar$ for $\functionpar\neq 0$ and $\swishboundA=-\infty$ for $\functionpar= 0$ (see Lemma~\ref{swishbound}).
\nameswish\ activation with $\functionpar=1$ was as introduced in \citet{Elfwing2018} under the name \marginterm{\namesil}\emph{\namesil};
the parametrized version with \functionpar\ potentially a learnable parameter (see the following section) and the name ``\nameswish'' were introduced later in \citet{Ramachandran2017}.

\nameswish\ can be seen as an interpolation between \namerelu\ and (a scaled version of) \namelinear:
$\functionswish\approx\functionrelu$ for $\functionpar\gg 1$ and $\functionswish\approx\functionlinear/2$ for $\functionpar\ll 1$---see  Figure~\ref{fig:swishrelu}.
\later{derivatives of Sil like sigmoid \cite{Elfwing2018}.}

\nameswish\ is  infinitely many times differentiable on its entire domain with first and second derivatives (see Lemmas~\ref{swishdiff}, \ref{swishboundD}, and~\ref{res:swishsym})
\begin{equation*}
  \functionswishDF{\argument}
  =\functionpar\functionswishF{\argument}+\functionlogF{\functionpar\argument}\bigl(1-\functionpar\functionswishF{\argument}\bigr)\in(1-\swishboundP,\swishboundP)
\end{equation*}
and (see Lemma~\ref{swishdiff})
\begin{equation*}
  \functionswishDDF{\argument}=\functionpar\Bigl(\functionpar\functionswishF{\argument}+2\functionlogF{\functionpar\argument}\bigl(1-\functionpar\functionswishF{\argument}\bigr)\Bigr)\bigl(1-\functionlogF{\functionpar\argument}\bigr)\in(\swishboundAPPLower,\swishboundAPPUpper)
\end{equation*}
for all $\functionpar\in[0,\infty)$ and $\argument\in\R$,
and for $\swishboundP\approx 1.098$, 
$\swishboundAPPLower\approx -0.0369 \functionpar$, and $\swishboundAPPUpper = 0.5 \functionpar$. 
We observe that 
the first and second derivatives of \nameswish, as well as \nameswish\ itself, are simple combinations of $\functionlogF{\functionpar\argument}$ and $\argument \functionlogF{\functionpar\argument}$. 
Hence,  \nameswish\ activation is comparable to \namelog\ activation in terms of basic computational complexity.

Several variants of \nameswish\ have been suggested recently,\marginterm{\nameelish/\namemish}
such as \nameelish\ \citep{Basirat2018} and \namemish\ \citep{Misra2019}.
A difference of \nameswish\ and its variants to all other activation functions discussed so far is that \nameswish\ is not monotone.
But
the practical benefits of this property,
as well as the features of these activations more generally,
are currently unclear.
An interesting aspect of \nameswish\ is, however, that it is the first activation function that originate from an automated search.

\subsubsection{Learning Activation Functions}
\label{learning}

Activations usually remain fixed during the entire process of training and application of a neural network.
But activation functions can also be  fitted during training by selecting within predefined sets of activation functions~\citep{Liu1996,White1993} or by fitting the parameters of preselected activation functions~\citep{Augusteijn2004,Duch2001}.
Recent implementations of this idea include fitting polynomial activations with the Taylor coefficients of standard activation functions as initial values for the parameters \citep{Chung2016},
\emph{parametric rectified linear unit (\nameprelu)} \citep{He2015},\marginterm{\nameprelu/\namerelu} 
which fits the parameter of \namelrelu,
\emph{parametric exponential linear unit (\namepelu)} \citep{Trottier2017},
which fits the parameters of a version of~\nameelu,
and the aforementioned \nameswish\ with its parameter fitted to the training data \citep{Ramachandran2017}. 

A related concept is \marginterm{\namemaxout}\namemaxout\ introduced in~\citet{Goodfellow2013}.
 \namemaxout\  replaces the neuron in~\eqref{neuron} by 
\begin{align*}
      \neuronAmax\ : \ \R^{\nbrinput}\,&\to\,\R\,;\\ 
    \datain\,&\mapsto\,
\max\Biggl\{\neuronparB^1+\sum_{j=1}^{\nbrinput}\neuronparE_j\datainE_j,\neuronparB^2+\sum_{j=1}^{\nbrinput}\neuronparE_{j+\nbrinput}\datainE_j,\dots,\neuronparB^{k}\sum_{j=1}^{\nbrinput}\neuronparE_{j+(k-1)\nbrinput}\datainE_{j}\Biggr\}\,,
\end{align*}
where $\neuronparBB\in\R^{k}$ and $\neuronpar\in\R^{k\nbrinput}$ are vector-valued parameters and $k\in\{1,2,\dots\}$ a fixed number.
\namemaxout\ was designed to improve the performance of dropout (a technique that tries to alleviate overfitting by randomly excluding nodes when updating weights with stochastic-gradient descent \citep{Srivastava2014}),
and it can be seen as a max pooling over neurons with linear activations
(max pooling is a data-aggregation scheme that is popular especially for  convolutional neural networks \citep{Scherer2010}).
\namemaxout\ is not an activation function,
but it can be interpreted in the spirit of activation functions in two ways:
First,
as a generalization of~\namerelu:
in the special case $k=1$,
\namemaxout\ simply yields a \namerelu\ neuron, that is, $\neuronAmaxO=\neuron_{\neuronparB,\neuronpar,\functionrelu}$;
this interpretation connects \namemaxout\ with our Section~\ref{sec:relu}.
Second,
 as a method for fitting piecewise-linear activation functions;
this interpretation connects \namemaxout\ with the above-stated learning approaches. 

\namemaxout\ hinges on~$k$,
which can be difficult to choose in practice,
and it augments the parameters space from $\nbrinput+1$ to $k(\nbrinput+1)$ dimensions,
 which bears computational challenges and the risk of overfitting;
in particular, when comparing \namemaxout\ networks to other networks, 
one should account for the increased parameter space~\citep{Castaneda2019}.
More generally, theoretical and empirical support for \namemaxout\ as well as of the other learning schemes stated above is currently very limited.

\later{\section{Empirical Comparisons}
See the following papers (all downloaded):
\citep{Pedamonti2018} \citep{Eger2019}.
}

\section{Practical Implications}
\label{sec:discussion}
Figure~\ref{fig:overview} on the next page displays  the main activation functions discussed in Section~\ref{sec:activationfunctions} in a single graph. 
Our findings in Section~\ref{sec:activationfunctions} have three practical implications:

First,
our findings support \namesoft\ activation. 
Activation functions that have similar graphs presumably give similar empirical results when the network parameters are ``sufficiently optimized'' (see, for example, \citet[Section~5]{Basirat2018}).
This suggests that choices among similar activation functions should be based 
 on how much computational effort it takes to  optimize the network parameters, that is,
 on how simple the activation functions are from a computational perspective.
A practical implication of our comparison of sigmoid activations in Section~\ref{sec:sigmoid} is, therefore, that 
\namesoft,
a particularly simple sigmoid function,
 should receive much more attention (see~Section~\ref{sec:softsign}).

Second, our findings support \namerelu\ activation in a similar way:
there is limited empirical and theoretical evidence for its competitors,
and \namerelu\ stands out in view of its computational simplicity (see~Section~\ref{sec:relu}).
A practical implication of Section~\ref{sec:activationfunctions} is, therefore, that \namesoft\ and \namerelu\ should---at least for now---be the standard activations in practice.

Third, our findings highlight three promising approaches for improving on these choices:
theoretical considerations (such as in the context of \nameselu, see Section~\ref{elu}),
automated searches (such as in the context of \nameswish, see Section~\ref{sec:swish}),
and data-adaptive selection schemes (see Section~\ref{learning}).

\ifarXiv\paragraph{Acknowledgments}
I thank Shih-Ting Huang for the careful reading of a draft version of this manuscript. 
\fi

\begin{figure}[t]
  \centering
 \scalebox{0.3}{\input{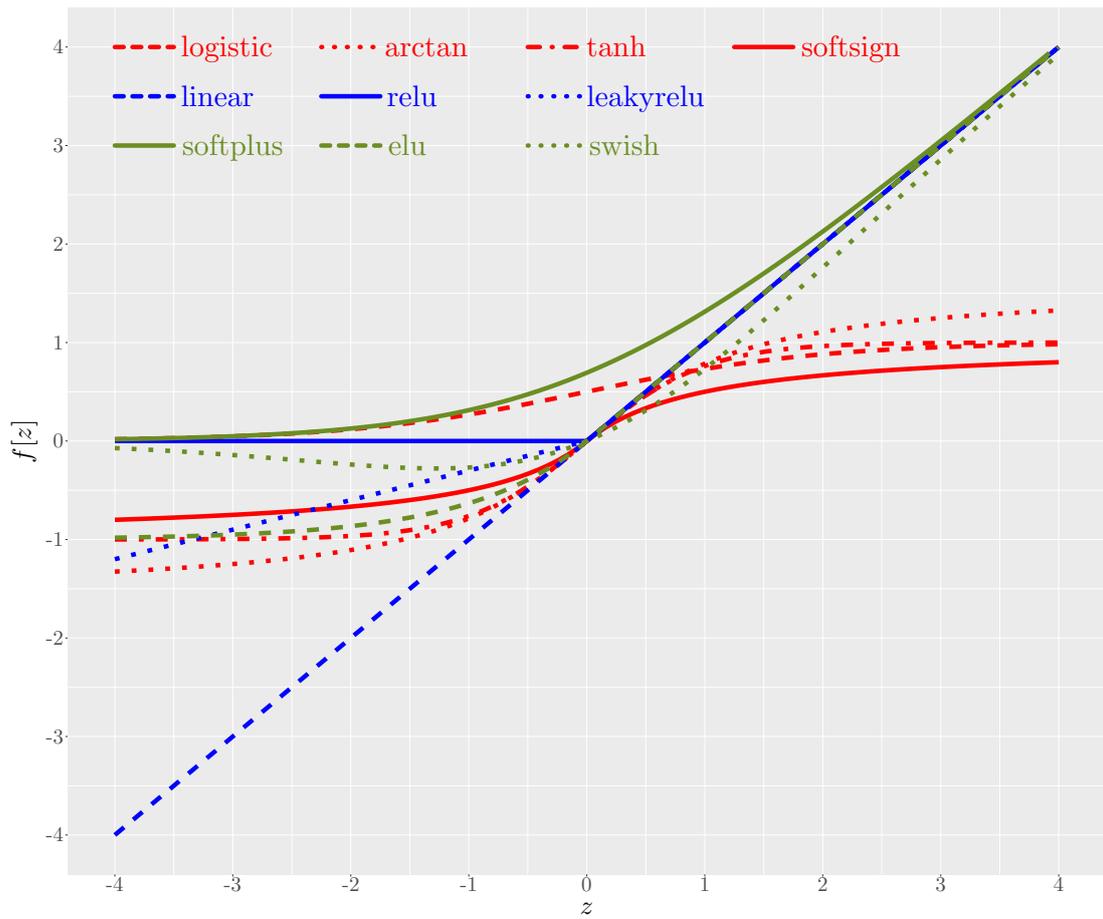}}
  \caption{Overview of sigmoid functions (red, see Section~\ref{sec:sigmoid}), piecewise-linear functions (blue, see Section~\ref{piecewise}), and other functions (green, see Section~\ref{sec:other}). 
(Best seen in color.)}
  \label{fig:overview}
\end{figure}

\clearpage



\appendix


\bibliography{Bibliography}

\appendix
\newpage
\setcounter{lemma}{0}
\renewcommand{\thelemma}{\Alph{section}.\arabic{lemma}}
\setcounter{example}{0}
\renewcommand{\theexample}{\Alph{section}.\arabic{example}}
\setcounter{definition}{0}
\renewcommand{\thedefinition}{\Alph{section}.\arabic{definition}}

\section{Mathematical Details}
In the following,
we prove our mathematical claims.
Our derivations require minimal prerequisites,
and we recall mathematical background where needed.

\subsection{Derivatives and Directional Derivatives}
\label{sec:directionalderivatives}
We first recall the notions of derivatives and directional derivatives.
The basis of derivatives are (two-sided) limits.
For all  $\directionlimZ\in\R\cup\{\pm\infty\}$ and $a\in(0,\infty)$,
we define ``balls'' around $\directionlimZ$ as follows:
\begin{itemize}
\item if $\directionlimZ\in\R$, define $\mathcal B_a[\directionlimZ]\deq\{c\in\R\,:\,\abs{\directionlimZ-c}<a \}$;
\item if $\directionlimZ=\infty$, define $\mathcal B_a[\directionlimZ]\deq\{c\in\R\,:\,c> 1/a \}$;
\item if $\directionlimZ=-\infty$, define $\mathcal B_a[\directionlimZ]\deq\{c\in\R\,:\,c< -1/a \}$.
\end{itemize}

\noindent Consider now a function  $\function\,:\,\R\to\R$.
We say that \marginterm{(two-sided) limit}\emph{the (two-sided) limit of~$\function$ at~$\directionlimZ$ exists (in $\R\cup\{\pm\infty\}$)} if and only if there is an $l\in\R\cup\{\pm\infty\}$ such that the following holds:
for every $a\in(0,\infty)$, there is a $b\in(0,\infty)$ such that $\functionF{\directionlim}\in\mathcal B_a[l]$ for all $\directionlim\in\mathcal B_b[\directionlimZ]$.
We write $\lim_{\directionlim\to\directionlimZ}\functionF{\directionlim}\deq l$.

The \marginterm{derivative}\emph{(usual) derivatives} of a function~$\function\,:\,\R\to\R$  are then
\begin{equation*}
  \functionDF{\argument}\deq\frac{\partial}{\partial \argument}\functionF{\argument}\deq \lim_{\directionlim\to0}\frac{\functionF{\argument+\directionlim}-\functionF{\argument}}{\directionlim},~~\functionDDF{\argument}\deq \frac{\partial}{\partial \argument}\functionDF{\argument}\deq \lim_{\directionlim\to0}\frac{\functionDF{\argument+\directionlim}-\functionDF{\argument}}{\directionlim},~~\dots 
\end{equation*}
for all $\argument\in\R$ where the limits exist.
If needed for clarity,
the distinction between the function variable (say $\argumentP$) and the point at which the derivative is evaluated (say $\argumentPZ$) is made explicit:
for example,
\begin{equation*}
  \frac{\partial}{\partial \argumentP}\functionF{\argumentP}\Big|_{\argumentP=\argumentPZ}\deq \lim_{\directionlim\to0}\frac{\functionF{\argumentPZ+\directionlim}-\functionF{\argumentPZ}}{\directionlim}
\end{equation*}
(if the limit exists) is the derivative of~$\argumentP\mapsto \functionF{\argumentP}$ with respect to~\argumentP\ at point~$\argumentP_0$.

Directional derivatives generalize standard derivatives:
directional derivatives exist at every differentiable point of a function,
and they also exist for some nondifferentiable points, 
such as for the point $\argument=0$ of $\functionrelu\,:\,\argument\mapsto\max\{\argument,0\}$ from Section~\ref{sec:relu}.

The basis of directional derivatives are right-sided limits.
For all  $\directionlimZ\in\R$ and $a\in(0,\infty)$,
we define ``directed balls'' around $\directionlimZ$ as follows:
\begin{itemize}
\item $\mathcal B_a^+[\directionlimZ]\deq\{c\in(\directionlimZ,\infty)\,:\,\abs{\directionlimZ-c}\leq a\}$.
\end{itemize}

\noindent Consider again a function $\function\,:\,\R\to\R$.
We say that \marginterm{right-sided limit}  
\emph{the right-sided limit of~$\function$ at~$\directionlimZ$ exists (in $\R\cup\{\pm\infty\}$)} 
if and only if there is an $l\in\R\cup\{\pm\infty\}$ such that the following holds:
for every $a\in(0,\infty)$, there is a $b\in(0,\infty)$ such that $\functionF{\directionlim}\in\mathcal B_a[l]$ for all $\directionlim\in\mathcal B_b^+[\directionlimZ]$.
We then write $\lim_{\directionlim\to\directionlimZ^+}\functionF{\directionlim}\deq l$.

The directional derivatives are then:
\begin{definition}[Directional derivatives]\label{directional}
Consider a function $\function\,:\,\R\to\R$ and 
two values $\direction,\argument\in\R$. 
The quantity
\begin{equation*}
  \diffdirectional\functionF{\argument}\deq \lim_{\directionlim\to0^+}\frac{\functionF{\argument+\directionlim\direction}-\functionF{\argument}}{\directionlim}\,,
\end{equation*}
if the limit exists,
is called the \marginterm{directional\\derivative}\emph{first directional derivative of \function\ at~\argument\ in the direction~\direction}.
\end{definition}
\noindent If~\function\ is differentiable at~\argument, 
it holds that $\diffdirectional\functionF{\argument}=\direction\functionDF{\argument}$.
Conversely, if for a given $\argument\in\R$, it holds that $\diffdirectional\functionF{\argument}=\diffdirectionalM\functionF{\argument}$ for all $\direction\in\R$,
then, \function\ is differentiable at $\argument$ with derivative that satisfies $\diffdirectional\functionF{\argument}=\direction\functionDF{\argument}$.

\subsection{l'H\^opital's Rule}
l'H\^opital's rule is a standard technique for evaluating limits.
We state a version that fits our needs.
\begin{lemma}[l'H\^opital's rule]\label{res:krankenhaus}
Consider two functions $\function,\functionG\,:\,\R\to\R$ that are differentiable on an open interval
 $\mathcal I\subset\R$,
and consider a point $\directionlimZ\in\R$.

1.~If $\argumentP\in\mathcal I$, $\functionGF{\directionlim}\neq 0$ for all $\directionlim\in\mathcal I\setminus\{\directionlimZ\}$, and
\begin{itemize}
\item   $\lim_{\directionlim\to\directionlimZ}\functionF{\directionlim}=\lim_{\directionlim\to\directionlimZ}\functionGF{\directionlim}=0$ or
\item   $\lim_{\directionlim\to\directionlimZ}\functionF{\directionlim}=\lim_{\directionlim\to\directionlimZ}\functionGF{\directionlim}=\infty$,
\end{itemize}
it holds that 
\begin{equation*}
  \lim_{\directionlim\to\directionlimZ}\frac{\functionF{\directionlim}}{\functionGF{\directionlim}}=  \lim_{\directionlim\to\directionlimZ}\frac{\functionDF{\directionlim}}{\functionGDF{\directionlim}}
\end{equation*}
---as long as the second limit exists.
  
2.~If $\mathcal I=(\directionlimZ,\infty)$, $\functionGF{\directionlim}\neq 0$ for all $\directionlim\in\mathcal I$, and
\begin{itemize}
\item   $\lim_{\directionlim\to\directionlimZ^+}\functionF{\directionlim}=\lim_{\directionlim\to\directionlimZ^+}\functionGF{\directionlim}=0$ or
\item   $\lim_{\directionlim\to\directionlimZ^+}\functionF{\directionlim}=\lim_{\directionlim\to\directionlimZ^+}\functionGF{\directionlim}=\infty$,
\end{itemize}
it holds that 
\begin{equation*}
  \lim_{\directionlim\to\directionlimZ^+}\frac{\functionF{\directionlim}}{\functionGF{\directionlim}}=  \lim_{\directionlim\to\directionlimZ^+}\frac{\functionDF{\directionlim}}{\functionGDF{\directionlim}}
\end{equation*}
---as long as the second limit exists.

3.~Analog statements hold for $\directionlimZ\in\{\pm\infty\}$.
\end{lemma}

\later{Could extend and be more formal with limits more generally: adding limits etc. See comments ST.}

\subsection{Derivatives of the Sigmoid Functions}

\subsubsection{Derivatives of \namelog}
\begin{lemma}[Derivatives of \namelog]
\label{difflogistic}
The first and second derivatives of \namelog\ are 
\begin{equation*}
  \functionlogDF{\argument}=\frac{e^{-\argument}}{(1+e^{-\argument})^2}=\functionlogF{\argument}\bigl(1-\functionlogF{\argument}\bigr)\in(0,1/4]
\end{equation*}
and
\begin{equation*}
\functionlogDDF{\argument}=\frac{e^{-\argument}(e^{-\argument}-1)}{(1+e^{-\argument})^3}=\functionlogDF{\argument}\bigl(1-2\functionlogF{\argument}\bigr)=\functionlogF{\argument}\bigl(1-\functionlogF{\argument}\bigr)\bigl(1-2\functionlogF{\argument}\bigr)\in(-\swishbound,\swishbound)
\end{equation*}
for all $\argument\in\R$ and $\swishbound\approx 0.0962$
\end{lemma}

\begin{proof}[Proof of Lemma~\ref{difflogistic}]
  The first derivatives follow more or less directly from the definition of \namelog\ and the sum and chain rules:
  \begin{align*}
\functionlogDF{\argument}
&=\frac{\partial}{\partial\argument}\frac{1}{1+e^{-\argument}}\tj{definition of \functionlog}\\
    &=\frac{-1}{(1+e^{-\argument})^2}\frac{\partial}{\partial\argument}(1+e^{-\argument})\tj{chain rule}\\
&=\frac{-1}{(1+e^{-\argument})^2}\cdot (0-e^{-\argument})\tj{$1'=0$; $(e^{-\argument})'=-e^{-\argument}$; sum rule}\\
&=\frac{e^{-\argument}}{(1+e^{-\argument})^2}\tj{consolidating}
\end{align*}
and further
\begin{align*}
\functionlogDF{\argument}&=\frac{e^{-\argument}}{(1+e^{-\argument})^2}\tj{previous display}\\
&=\frac{1}{1+e^{-\argument}}\cdot\frac{e^{-\argument}}{1+e^{-\argument}}\tj{splitting the fraction into two parts}\\    
&=\frac{1}{1+e^{-\argument}}\cdot\Bigl(\frac{1+e^{-\argument}}{1+e^{-\argument}}-\frac{1}{1+e^{-\argument}}\Bigr)\tj{splitting the second factor into two parts}\\
    &=\functionlogF{\argument}\bigl(1-\functionlogF{\argument}\bigr)\tj{definition of \functionlog\ and simplifcation}\,,
  \end{align*}
as desired.

This result can then be used together with the sum and product rules to find the second derivatives:
  \begin{align*}
    \functionlogDDF{\argument}
    &=\frac{\partial}{\partial\argument}\functionlogDF{\argument}\tj{definition of the second derivative}\\
&=\frac{\partial}{\partial\argument}\Bigl(\functionlogF{\argument}\bigl(1-\functionlogF{\argument}\bigr)\Bigr)\tj{above result}\\
&=\Bigl(\frac{\partial}{\partial\argument}\functionlogF{\argument}\Bigr)\bigl(1-\functionlogF{\argument}\bigr)+\functionlogF{\argument}\frac{\partial}{\partial\argument}\bigl(1-\functionlogF{\argument}\bigr)\tj{product rule}\\
&=\functionlogDF{\argument}\bigl(1-\functionlogF{\argument}\bigr)+\functionlogF{\argument}\bigl(0-\functionlogDF{\argument}\bigr)\tj{$1'=0$; definition of \functionlogD; sum rule}\\
    &=\functionlogDF{\argument}\bigl(1-2\functionlogF{\argument}\bigr)\tj{factoring out \functionlogD}
  \end{align*}
and further, using the result $\functionlogD=\functionlogF{\argument}(1-\functionlogF{\argument})$, 
\begin{equation*}
      \functionlogDDF{\argument}=\functionlogF{\argument}\bigl(1-\functionlogF{\argument}\bigr)\bigl(1-2\functionlogF{\argument}\bigr)\,.
\end{equation*}
We can then finally plug in the definition of~\functionlog\ to find
\begin{align*}
  \functionlogDDF{\argument}&=\functionlogDF{\argument}\bigl(1-2\functionlogF{\argument}\bigr)\tj{penultimate display}\\
&=\frac{e^{-\argument}}{(1+e^{-\argument})^2}\biggl(1-\frac{2}{1+e^{-\argument}}\biggr)\tj{above result for~\functionlogD; definition of~\functionlog}\\
&=\frac{e^{-\argument}}{(1+e^{-\argument})^2}\frac{1+e^{-\argument}-2}{1+e^{-\argument}}\tj{summarizing the second factor}\\
&=\frac{e^{-\argument}(e^{-\argument}-1)}{(1+e^{-\argument})^3}\tj{consolidating}\,,
\end{align*}
as desired.
\detail{Alternatively, one could derive
\begingroup
\allowdisplaybreaks
  \begin{align*}
    \functionlogDDF{\argument}&=\frac{\partial}{\partial\argument}\functionlogDF{\argument}\tj{definition of the second derivative}\\
&=\frac{\partial}{\partial\argument}\frac{e^{-\argument}}{(1+e^{-\argument})^2}\tj{result for \functionlogD}\\
&=\frac{\bigl(\frac{\partial}{\partial\argument}e^{-\argument}\bigr)(1+e^{-\argument})^2 -e^{-\argument}\Bigr(\frac{\partial}{\partial\argument}\bigr((1+e^{-\argument})^2\bigr)\Bigr)}{(1+e^{-\argument})^4}\tj{quotient rule}\\
&=\frac{\bigl(\frac{\partial}{\partial\argument}e^{-\argument}\bigr)(1+e^{-\argument})^2 -e^{-\argument}\cdot 2(1+e^{-\argument})\frac{\partial}{\partial\argument}(1+e^{-\argument})}{(1+e^{-\argument})^4}\tj{chain rule}\\
&=\frac{-e^{-\argument}(1+e^{-\argument})^2 -e^{-\argument}\cdot 2(1+e^{-\argument})(0-e^{-\argument})}{(1+e^{-\argument})^4}\tj{$(e^{-\argument})'=-e^{-\argument}$; $1'=0$; sum rule}\\
&=\frac{-e^{-\argument}(1+e^{-\argument}) +2e^{-\argument}e^{-\argument}}{(1+e^{-\argument})^3}\tj{consolidating}\\
&=\frac{e^{-\argument}(e^{-\argument}-1)}{(1+e^{-\argument})^3}\tj{factoring out $e^{-\argument}$ in the numerator}\,.
  \end{align*}
\endgroup
}

To identify the output range of the first derivatives,
we use the above-derived equality
\begin{equation*}
  \functionlogDF{\argument}=\frac{e^{-\argument}}{(1+e^{-\argument})^2}\,,
\end{equation*}
which implies  the fact that $\functionlogDF{\argument}>0$ for all $\argument\in\R$   and $\functionlogDF{\argument}\to0$ for $\argument\to\pm\infty$.
Using this and the continuity of the derivative,
we can conclude that the output range is $(0,\swishbound]$ for some $\swishbound\in(0,\infty)$.
To determine~$\swishbound$,
we use that \functionlogD\ can be continuously differentiated (see above),
so that we can find its maximum by setting its derivatives to zero.
We find

The above-derived equality for \functionlogDD\ and the positivity of~\functionlogD\ and then the above-derived equality for \functionlogD\ yield
\begin{align*}
  \functionlogDDF{\argument}&=0\\
\Rightarrow~~~~\functionlogDF{\argument}\bigl(1-2\functionlogF{\argument}\bigr)&=0\tj{above result for \functionlogDD}\\
\Rightarrow~~~~1-2\functionlogF{\argument}&=0\tj{positivity of \functionlogD}\\
\Rightarrow~~~~\functionlogF{\argument}&=\frac{1}{2}\tj{rearranging the terms}\\
\Rightarrow~~~~\functionlogF{\argument}\bigl(1-\functionlogF{\argument}\bigr)&=\frac{1}{4}\tj{$1/2(1-1/2)=1/4$}\\
\Rightarrow~~~~\functionlogDF{\argument}&=\frac{1}{4}\tj{above result for \functionlogD}\,. 
\end{align*}
Thus,  the output range of the derivative is indeed $(0,1/4]$, as desired.

We can identify the output range of the second derivative with similar arguments above.
We first compute the third derivatives of \namelog:
\begingroup
\allowdisplaybreaks
  \begin{align*}
    \functionlogDDDF{\argument}
&=\frac{\partial}{\partial\argument}\functionlogDDF{\argument}\tj{definition of the third derivative}\\
&=\frac{\partial}{\partial\argument}\Bigl(\functionlogDF{\argument}\bigl(1-2\functionlogF{\argument}\bigr)\Bigr)\tj{above results}\\
    &=\Bigl(\frac{\partial}{\partial\argument}\functionlogDF{\argument}\Bigr)\bigl(1-2\functionlogF{\argument}\bigr)+\functionlogDF{\argument}\Bigl(\frac{\partial}{\partial\argument}\bigl(1-2\functionlogF{\argument}\bigr)\Bigr)\tj{product rule}\\
    &=\functionlogDDF{\argument}\bigl(1-2\functionlogF{\argument}\bigr)+\functionlogDF{\argument}\bigl(0-2\functionlogDF{\argument}\bigr)\tj{definition of \functionlogD; $1'=0$; $(-2\functionlog)'=-2\functionlogD$; sum rule}\\
&=\functionlogDF{\argument}\bigl(1-2\functionlogF{\argument}\bigr)\bigl(1-2\functionlogF{\argument}\bigr)-2\bigl(\functionlogDF{\argument}\bigr)^2\tj{above equality for $\functionlogDD$; consolidating the second term}\\
&=\functionlogDF{\argument}\Bigl(1-4\functionlogF{\argument}+4\bigl(\functionlogF{\argument}\bigr)^2-2\functionlogDF{\argument}\Bigr)\tj{summarizing the two terms}\\
&=\functionlogDF{\argument}\frac{(1+e^{-\argument})^2-4\bigl(1+e^{-\argument})+4-2e^{-\argument}}{\bigl(1+e^{-\argument})^2}\tj{definition of \functionlog; above-derived equality for \functionlogD}\\
&=\functionlogDF{\argument}\frac{1+2e^{-\argument}+e^{-2\argument}-4-4e^{-\argument}+4-2e^{-\argument}}{\bigl(1+e^{-\argument})^2}\tj{expanding the terms}\\
&=\functionlogDF{\argument}\frac{1-4e^{-\argument}+e^{-2\argument}}{\bigl(1+e^{-\argument})^2}\tj{consolidating}\\
&=\functionlogDF{\argument}\frac{(1-2e^{-\argument})^2-3(e^{-\argument})^2}{\bigl(1+e^{-\argument})^2}\tj{rearranging the terms}\,,
  \end{align*}
\endgroup
which is equal to zero (both $\functionlogD$ and $\bigl(1+e^{-\argument})^2$ are positive) if and only if $(1-2e^{-\argument})^2=3(e^{-\argument})^2$.
The claim then follows similarly as in the case of the first derivative (we omit some details) from
\begin{align*}
  (1-2e^{-\argument})^2&=3(e^{-\argument})^2\\
\Rightarrow~~~~   1-2e^{-\argument}&=\pm\sqrt{3} e^{-\argument}\tj{taking square roots on both sides}\\
\Rightarrow~~~~  (-2\pm\sqrt{3})e^{-\argument}&=-1\tj{factoring out $e^{-\argument}$}\\
\Rightarrow~~~~  e^{-\argument}&=1/(2\pm\sqrt{3})\tj{dividing both sides by $(-2\pm\sqrt{3})=-(2\pm\sqrt{3})\neq 0$}\\
\Rightarrow~~~~  -\argument&=\log[1/(2\pm\sqrt{3})]\tj{taking logarithms on both sides}\\
\Rightarrow~~~~  -\argument&=-\log[2\pm\sqrt{3}]\tj{$\log[1/b]=-\log[b]$}\\
\Rightarrow~~~~  \argument&=\log[2\pm\sqrt{3}]\tj{multiplying both sides by $-1$}\\
\Rightarrow~~~~  \argument&=\pm\log[2+\sqrt{3}]\tj{see below}\,,
\end{align*}
where the last step follows from 
\begin{align*}
\log[1]&=0\tj{basic property of the logarithm}\\
\Rightarrow~~~~\log\bigl[(2+\sqrt{3})(2-\sqrt{3})\bigr]&=0\tj{$(2+\sqrt{3})(2-\sqrt{3})=1$}\\
\Rightarrow~~~~ \log[2+\sqrt{3}]+\log[2-\sqrt{3}]&=0\tj{$\log[ab]=\log[a]+\log[b]$}\\
\Rightarrow~~~~ \log[2+\sqrt{3}]&=-\log[2-\sqrt{3}]\tj{subtracting $\log[2-\sqrt{3}]$ on both sides}\,.
\end{align*}
Numerical evaluation then yields the desired output range:
\begin{equation*}
  \pm\swishbound=\functionlogDDF{\pm\log[2+\sqrt{3}]}\approx \pm 0.0962\,.
\end{equation*}
\end{proof}

\subsubsection{Derivatives of \namearctan}
\begin{lemma}[Derivatives of \namearctan]
\label{diffarctan}
The first and second derivatives of \namearctan\ are
\begin{equation*}
  \functionarctanDF{\argument}=\frac{1}{1+\argument^2}\in[0,\infty)~~~~\text{and}~~~~\functionarctanDDF{\argument}=-\frac{2\argument}{(1+\argument^2)^2}=-\frac{2\argument}{\bigl(\functionarctanDF{\argument}\bigr)^2}\in\R
\end{equation*}
 for all $\argument\in\R$.
\end{lemma}

\begin{proof}[Proof of Lemma~\ref{diffarctan}]
By the definition of \namearctan\ via  ${\tan}[{\arctan}[\argument]]=\argument$,
it holds that
  \begin{equation*}
    \frac{\partial}{\partial\argument}{\tan}\bigl[{\arctan}[\argument]\bigr]=\frac{\partial}{\partial\argument}\argument=1\,.
  \end{equation*}
On the other hand, the chain rule ensures that
  \begin{equation*}
    \frac{\partial}{\partial\argument}{\tan}\bigl[{\arctan}[\argument]\bigr]=    \biggl(\frac{\partial}{\partial\argumentP}{\tan}[\argumentP]\Big|_{\argumentP={\arctan}[\argument]} \biggr)\frac{\partial}{\partial\argument}{\arctan}[\argument]\,.
  \end{equation*}
Combining these two inequalities yields  
\begin{equation*}
\biggl(\frac{\partial}{\partial\argumentP} {\tan}[\argumentP]\Big|_{\argumentP={\arctan}[\argument]}\biggr)\frac{\partial}{\partial\argument}{\arctan}[\argument]=1\,.
\end{equation*}
This equation allows us to calculate the derivative of \namearctan\ through the derivative of ${\tan}$.

The derivative of ${\tan}$ is
  \begin{align*}
   & \frac{\partial}{\partial\argumentP}\tan[\argumentP]\Big|_{\argumentP=\arctan[\argument]} \\
   & =\frac{\partial}{\partial\argumentP}\frac{{\sin}[\argumentP]}{{\cos}[\argumentP]}\Big|_{\argumentP={\arctan}[\argument]}\tj{$\tan=\sin/\cos$ by definition}\\
   & =\frac{\bigl(\frac{\partial}{\partial\argumentP}{\sin}[\argumentP]\bigr){\cos}[\argumentP]-{\sin}[\argumentP]\bigl(\frac{\partial}{\partial\argumentP}{\cos}[\argumentP]\bigr)}{\bigl({\cos}[\argumentP]\bigr)^2}\Big|_{\argumentP={\arctan}[\argument]}\tj{quotient rule}\\
   & =\frac{\bigl({\cos}[\argumentP]\bigr)^2+\bigl({\sin}[\argumentP]\bigr)^2}{\bigl({\cos}[\argumentP]\bigr)^2}\Big|_{\argumentP={\arctan}[\argument]}\tj{$\sin'=\cos$; $\cos'=-\sin$}\\
   & =1+\frac{\bigl({\sin}[\argumentP]\bigr)^2}{\bigl({\cos}[\argumentP]\bigr)^2}\Big|_{\argumentP={\arctan}[\argument]}\tj{splitting the fraction}\\
&=1+(\tan[\argumentP])^2|_{\argumentP={\arctan}[\argument]}\tj{$\tan=\sin/\cos$ by definition}\\
&=1+\bigl(\tan\bigl[{\arctan}[\argument]\bigr]\bigr)^2\tj{evaluating the function}\\
&=1+\argument^2\tj{${\tan}[{\arctan}[\argument]]=\argument$ by definition}\,.
  \end{align*}
Plugging this back into the above display yields
\begin{equation*}
(1+\argument^2)\frac{\partial}{\partial\argument}{\arctan}[\argument]=1\,,
\end{equation*}
which yields after dividing both sides by $1+\argument^2$ (observe that $1+\argument^2>0$)
\begin{equation*}
\functionarctanDF{\argument}=\frac{\partial}{\partial\argument}{\arctan}[\argument]=\frac{1}{1+\argument^2}\,,
\end{equation*}
which is the desired first derivative.

The second derivative then follows essentially from the chain rule:
\begingroup
\allowdisplaybreaks
\begin{align*}
\functionarctanDDF{\argument}
  &=\frac{\partial}{\partial\argument}\functionarctanDF{\argument}\tj{definition of the second derivative}\\
  &=\frac{\partial}{\partial\argument}\frac{1}{1+\argument^2}\tj{above result}\\
 &= -\frac{\frac{\partial}{\partial\argument}(1+\argument^2)}{(1+\argument^2)^2}\tj{chain rule}\\
 &= -\frac{2\argument}{(1+\argument^2)^2}\tj{$1'=0$; $(\argument^2)'=2\argument$}\\
&=-\frac{2\argument}{\bigl(\functionarctanDF{\argument}\bigr)^2}\tj{above derivations}\,,
\end{align*}
\endgroup
which is the desired second derivative.

The output ranges then follow readily.
\end{proof}

\subsubsection{Derivatives of \nametanh}

\begin{lemma}[Derivatives of \nametanh]\label{difftanh}
  The first and second derivatives of \nametanh\ are 
\begin{equation*}
  \functiontanhDF{\argument}=1-\bigl(\functiontanhF{\argument}\bigr)^2\in(0,1)~~~\text{and}~~~\functiontanhDDF{\argument}=-2\functiontanhF{\argument}\Bigl(1-\bigl(\functiontanhF{\argument}\bigr)^2\Bigr)\in(-c,c)
\end{equation*}
  for all $\argument\in\R$ and $c\approx 0.770$.
\end{lemma}
\begin{proof}[Proof of Lemma~\ref{difftanh}]
The claims can be derived by using elementary differential calculus:
  \begin{align*}
  \functiontanhDF{\argument}  
 &= \frac{\partial}{\partial\argument}\frac{e^{\argument}-e^{-\argument}}{e^{\argument}+e^{-\argument}}\tj{definition of \functiontanh}\\
    &= \frac{\bigl(\frac{\partial}{\partial\argument}(e^{\argument}-e^{-\argument})\bigr)(e^{\argument}+e^{-\argument})-(e^{\argument}-e^{-\argument})\bigl(\frac{\partial}{\partial\argument}(e^{\argument}+e^{-\argument})\bigr) }{(e^{\argument}+e^{-\argument})^2}\tj{quotient rule}\\
    &=\frac{(e^{\argument}+e^{-\argument})(e^{\argument}+e^{-\argument})-(e^{\argument}-e^{-\argument})(e^{\argument}-e^{-\argument})}{(e^{\argument}+e^{-\argument})^2}\tj{$(e^{\argument})'=e^{\argument}$; $e^{-\argument}=-e^{-\argument}$; sum rule}\\
    &=1-\biggl(\frac{e^{\argument}-e^{-\argument}}{e^{\argument}+e^{-\argument}}\biggr)^2\tj{splitting the fraction up and simplifying}\\
&=1-\bigl(\functiontanhF{\argument}\bigr)^2\tj{definition of \functiontanh}
 \end{align*}
and
\begingroup
\allowdisplaybreaks
  \begin{align*}
  \functiontanhDDF{\argument}
 &= \frac{\partial}{\partial\argument} \functiontanhDF{\argument}\tj{definition of the second derivative}\\
 &= \frac{\partial}{\partial\argument} \Bigl(1-\bigl(\functiontanhF{\argument}\bigr)^2\Bigr)\tj{above result}\\
 &= -2\functiontanhF{\argument}\functiontanhDF{\argument}\tj{$1'=0$; sum and chain rules}\\
    &= -2\functiontanhF{\argument}\Bigl(1-\bigl(\functiontanhF{\argument}\bigr)^2\Bigr)\tj{above results}\,,
 \end{align*}
\endgroup
as desired.

The output range of the first derivative follows directly from the output range of the original function and of the form of the derivatives.

The output range of the second derivative is the output range of the following function (see the form of the second derivative):
\begin{align*}
  \functionP\ : (-1,1)\,&\to\,\R\\
  \argumentP\,&\mapsto\,-2\argumentP(1-\argumentP^2)\,.
\end{align*}
The function $\functionP$ is smooth and $\lim_{\argumentP\to\pm1}\functionPF{\argumentP}=0$.
Hence, the output range is determined by the minima and maxima of the function,
which must satisfy $\functionPDF{\argumentP}=0$.
Now,
\begingroup
\allowdisplaybreaks
\begin{align*}
  \frac{\partial}{\partial \argumentP}\functionPF{\argumentP}&=0\\
\Rightarrow~~~~\frac{\partial}{\partial \argumentP}\bigl(-2\argumentP(1-\argumentP^2)\bigr)&=0\tj{definition of \functionP}\\
\Rightarrow~~~~-2\frac{\partial}{\partial \argumentP}\bigl(\argumentP(1-\argumentP^2)\bigr)&=0\tj{linearity of differentiation}\\
\Rightarrow~~~~\frac{\partial}{\partial \argumentP}\bigl(\argumentP(1-\argumentP^2)\bigr)&=0\tj{dividing both sides by $-2$}\\
\Rightarrow~~~~1\cdot (1-\argumentP^2)+\argumentP\cdot (0-2\argumentP)&=0\tj{$\argumentP'=1$; $1'=0$; $-(\argumentP^2)'=-2\argumentP$; product and sum rules}\\
\Rightarrow~~~~1-3\argumentP^2&=0\tj{consolidation}\\
\Rightarrow~~~~\argumentP^2&=\frac{1}{3}\tj{rearranging the equation}\\
\Rightarrow~~~~\argument&=\pm\frac{1}{\sqrt{3}}\tj{taking square roots}\,.
\end{align*}
\endgroup
One can check readily that these points are indeed the minimum and maximum of~\functionP.

Plugging the points back into the function yields via basic algebra
\begin{align*}
  \functionPFBBB{\pm\frac{1}{\sqrt{3}}}=-2\biggl(\pm\frac{1}{\sqrt{3}}\biggr)\biggl(1-\biggl(\pm\frac{1}{\sqrt{3}}\biggr)^2\biggr)=\frac{\mp2}{\sqrt{3}}\biggl(1-\frac{1}{3}\biggr)=\frac{\mp 4}{\sqrt{27}}\approx \mp 0.770\,,
\end{align*}
as desired.
\end{proof}

\subsubsection{Derivatives of \namesoft}

\begin{lemma}[Derivatives of \namesoft]\label{diffsoft}
The first derivative of \namesoft\ is
\begin{equation*}
  \functionsoftDF{\argument}=\frac{1}{\bigl(1+\abs{\argument}\bigr)^2}=\bigl(1-\abs{\functionsoftF{\argument}}\bigr)^2\in (0,1)~~~~\text{for all}~\argument\in\R\,.
\end{equation*}
The second derivative of \namesoft\ is
\begin{multline*}
  \functionsoftDDF{\argument}=-\frac{2 \signF{\argument}}{\bigl(1+\abs{\argument}\bigr)^{3}} 
=-2 \signF{\argument}  \bigl(\functionsoftDF{\argument}\bigr)^{3/2}\\
=-2 \signF{\argument}  \bigl(1-\abs{\functionsoftF{\argument}}\bigr)^3\in(-2,2)
\end{multline*}
for all $\argument\in\R\setminus\{0\}$.
\end{lemma}
\begin{proof}[Proof of Lemma~\ref{diffsoft}]
We establish the first and second derivatives in order.
The only small difficulty is that the case $\argument=0$ needs special attention.
The output ranges of the derivatives follow almost directly from the explicit forms of the derivatives---we omit the details.

We start with the first derivative.
For $\argument\neq 0$, we find that
\begingroup
\allowdisplaybreaks
  \begin{align*}
  \frac{\partial}{\partial\argument}\functionsoftF{\argument}
 &= \frac{\partial}{\partial\argument}\frac{\argument}{1+\abs{\argument}}\tj{definition of \functionsoft}\\
 &= \frac{\bigl(\frac{\partial}{\partial\argument}\argument\bigr)(1+\abs{\argument}\bigr)-\argument \frac{\partial}{\partial\argument}(1+\abs{\argument}\bigr)}{\bigl(1+\abs{\argument}\bigr)^2}\tj{quotient rule}\\
 &= \frac{1\cdot\bigl(1+\abs{\argument})-\argument\cdot\bigl(0+ \signF{\argument}\bigr)}{\bigl(1+\abs{\argument}\bigr)^2}\tj{$\argument'=1$; $1'=0$; $\abs{\argument}'=\signF{\argument}$ for $\argument\neq 0$; sum rule}\\
 &= \frac{1+\abs{\argument}-\argument\signF{\argument}}{\bigl(1+\abs{\argument}\bigr)^2}\tj{simplification}\\
 &= \frac{1}{\bigl(1+\abs{\argument}\bigr)^2}\tj{$\argument\signF{\argument}=\abs{\argument}$}\,.
\end{align*}
\endgroup
This expression can be related to the  original function:
\begin{align*}
\frac{\partial}{\partial\argument}\functionsoftF{\argument}
 &= \frac{1}{\bigl(1+\abs{\argument}\bigr)^2}\tj{previous display}\\
 &= \biggl(\frac{1+\abs{\argument}-\abs{\argument}}{1+\abs{\argument}}\biggr)^2\tj{adding a zero-valued term}\\
 &= \biggl(1-\frac{\abs{\argument}}{1+\abs{\argument}}\biggr)^2\tj{splitting up the fraction}\\
 &= \biggl(1-\absBB{\frac{\argument}{1+\abs{\argument}}}\biggr)^2\tj{$1+\abs{\argument}>0$}\\
&=\bigl(1-\abs{\functionsoftF{\argument}}\bigr)^2\tj{definition of \functionsoft}\\
\end{align*}
as desired.

The case $\argument= 0$ can be treated via the basic definition of derivatives:
\begingroup
\allowdisplaybreaks
\begin{align*}
 \frac{\partial}{\partial\argument}\functionsoftF{\argument}\Big|_{\argument=0}
&=\lim_{\directionlim\to 0} \frac{\functionsoftF{0+\directionlim}-\functionsoftF{0}}{\directionlim}\tj{definition of derivatives}\\
&=\lim_{\directionlim\to 0} \frac{\frac{\directionlim}{1+\abs{\directionlim}}-0}{\directionlim}\tj{definition of \functionsoft}\\
&=\lim_{\directionlim\to 0} \frac{1}{1+\abs{\directionlim}}=1\tj{simplification and evaluation of the limit}\,.
\end{align*}
\endgroup
Again, we can formulate this in terms of the original function:
\begin{align*}
 \frac{\partial}{\partial\argument}\functionsoftF{\argument}\Big|_{\argument=0}
&=1\tj{previous display}\\
&=(1-0)^2\tj{basic calculus}\\
&=\bigl(1-\abs{\functionsoftF{\argument}}\bigr)^2\tj{$\functionsoftF{0}=0$}\,,
\end{align*}
as desired.

The second derivative at $\argument\neq0$ now follows readily from the first derivative:
\begingroup
\allowdisplaybreaks
\begin{align*}
\functionsoftDDF{\argument}&=   \frac{\partial}{\partial\argument}\functionsoftDF{\argument}\tj{definition of second derivative}\\
&=\frac{\partial}{\partial\argument}\bigl(1-\abs{\functionsoftF{\argument}}\bigr)^2\tj{above results}\\
&=2 \bigl(1-\abs{\functionsoftF{\argument}}\bigr) \frac{\partial}{\partial\argument}\bigl(1-\abs{\functionsoftF{\argument}}\bigr)\tj{chain rule}\\
&=2 \bigl(1-\abs{\functionsoftF{\argument}}\bigr) \frac{\partial}{\partial\argument}\bigl(1-\signF{\argument}\functionsoftF{\argument}\bigr)\tj{$\abs{b}=\signF{b}b$}\\
&=2 \bigl(1-\abs{\functionsoftF{\argument}}\bigr) \bigl(0-\signF{\argument}\functionsoftDF{\argument}-0\functionsoftF{\argument}\bigr)\tj{$1'=0$; $\signF{\argument}'=0$ at $\argument\neq0$; sum and product rules}\\
&=-2 \signF{\argument} \bigl(1-\abs{\functionsoftF{\argument}}\bigr) \functionsoftDF{\argument}\tj{consolidation}\\
&=-2 \signF{\argument}  \bigl(\functionsoftDF{\argument}\bigr)^{3/2}\tj{above results; $\functionsoftD>0$}\,,
\end{align*}
\endgroup
as desired.
In explicit terms,
we find (see again the above results)
\begin{equation*}
  \functionsoftDDF{\argument}=-2 \signF{\argument}  \bigl(\functionsoftDF{\argument}\bigr)^{3/2}=-\frac{2 \signF{\argument}}{\bigl(1+\abs{\argument}\bigr)^{3}}\,,
\end{equation*}
 as desired.

\detail{The second derivative can also be established from the explicit form of the first derivative:
\begin{align*}
\functionsoftDDF{\argument}&=   \frac{\partial}{\partial\argument}\functionsoftDF{\argument}\tj{definition of second derivative}\\
&=\frac{\partial}{\partial\argument}\frac{1}{\bigl(1+\abs{\argument}\bigr)^2}\tj{above results}\\
&=-2\frac{\frac{\partial}{\partial\argument}\bigl(1+\abs{\argument}\bigr)}{\bigl(1+\abs{\argument}\bigr)^3}\tj{chain rule}\\
&=-2\frac{0+\signF{\argument}}{\bigl(1+\abs{\argument}\bigr)^3}\tj{$1'=0$; $\abs{\argument}'=\signF{\argument}$ for $\argument\neq 0$; sum rule}\\
&=-\frac{2\signF{\argument}}{\bigl(1+\abs{\argument}\bigr)^3}\tj{consolidation}\,.
\end{align*}}

\end{proof}

\subsection{Properties of \nametanh}

\begin{lemma}[Properties of \nametanh]
\label{proptanh}
  It holds that
\begin{equation*}
  \functiontanhF{\argument}=2\functionlogF{2\argument}-1~~~~\text{for all}~\argument\in\R
\end{equation*}
and
\begin{multline*}
  \functiontanhF{0}= \functionarctanF{0},\,  \functiontanhDF{0}= \functionarctanDF{0},\,  \functiontanhDDF{0}= \functionarctanDDF{0},\\\functiontanhDDDF{0}= \functionarctanDDDF{0},~\text{and}~  \functiontanhDDDDF{0}= \functionarctanDDDDF{0}\,.
\end{multline*}
\end{lemma}

\begin{proof}[Proof of Lemma~\ref{proptanh}]
The first claim follows from elementary differential calculus:
\begingroup
\allowdisplaybreaks
  \begin{align*}
    \functiontanhF{\argument}&=\frac{e^{\argument}-e^{-\argument}}{e^{\argument}+e^{-\argument}}\tj{definition of \functiontanh}\\
&=\frac{e^{\argument}+e^{\argument}-e^{\argument}-e^{-\argument}}{e^{\argument}+e^{-\argument}}\tj{adding a zero-valued term}\\
&=\frac{2e^{\argument}}{e^{\argument}+e^{-\argument}}-1\tj{splitting the fraction up and simplifying}\\
&=\frac{2e^{\argument}}{e^{\argument}(1+e^{-2\argument})}-1\tj{factoring out an $e^{\argument}$}\\
&=2\frac{1}{1+e^{-2\argument}}-1\tj{consolidating}\\
&=2\functionlogF{2\argument}-1\tj{definition of \functionlog}\,,
  \end{align*}
\endgroup
as desired.

We leave the second part to the reader.
\jl{do}
\end{proof}

\subsection{Derivatives of the Piecewise-Linear Functions}
\label{sec:piecewiseder}

\subsubsection{Derivatives of \namelinear, \namerelu, and \namelrelu}

\begin{lemma}[Derivatives of \namelrelu]
\label{piecewiseder}
It holds that 
\begin{equation*}
  \functionlreluDF{\argument}=
  \begin{cases}
    1~~~\text{for all}~\argument\in(0,\infty)\\
\functionpar~~~\text{for all}~\argument\in(-\infty,0)
  \end{cases}~~~~\text{and}~~~~~~~\functionlreluDDF{\argument}=0~~\text{for all}~\argument\in\R\setminus\{0\}
\end{equation*}
and
\begin{equation*}
  \diffdirectional\functionlreluF{\argument}=
  \begin{cases}
    \direction~&\text{for all}~\direction\in\R~\text{and}~\argument\in(0,\infty)~\text{and for all}~\direction\in[0,\infty)~\text{and}~\argument=0\,;\\
\functionpar\direction~&\text{otherwise}\,.
  \end{cases}
\end{equation*}
\end{lemma}
\noindent 
As special cases, the lemma entails first and second derivatives and first directional derivatives also for \namelinear\ ($\functionpar=1$) and \namerelu\ ($\functionpar=0$).

\begin{proof}[Proof of Lemma~\ref{piecewiseder}]
Observe first that $\diffdirectional\functionlreluF{\argument}=\diffdirectionalM\functionlreluF{\argument}$ for all $\direction\in\R$ and $\argument\in\R\setminus\{0\}$;
hence, 
in view of our comments after Definition~\ref{directional},
the first derivatives on $\R\setminus\{0\}$ follow from the directional derivatives.
The second derivatives $\R\setminus\{0\}$ follow readily from the first derivatives.

What is left to prove is that the directional derivatives are as claimed.
We separate this proof into four cases.

\emph{Case 1:} $\direction\in\R$, $\argument\in(0,\infty)$

It holds that $\argument>0$ by assumption and, therefore, that $\argument+\directionlim\direction>0$ for small enough $\directionlim\in(0,\infty)$.  
These two observations combined with the definition of \namelrelu\ ensure that
$\functionF{\argument}=\argument$ and 
$\functionF{\argument+\directionlim\direction}=\argument+\directionlim\direction$ for small enough $\directionlim\in(0,\infty)$.
Using these two equalities and basic algebra, as well as Definition~\ref{directional} about directional derivatives, yields
\begin{equation*}
    \diffdirectional\functionF{\argument}= \lim_{\directionlim\to0^+}\frac{\functionF{\argument+\directionlim\direction}-\functionF{\argument}}{\directionlim}
= \lim_{\directionlim\to0^+}\frac{\argument+\directionlim\direction-\argument}{\directionlim}
= \lim_{\directionlim\to0^+}\direction=\direction\,,
\end{equation*}
as desired.

\emph{Case 2:} $\direction\in[0,\infty)$, $\argument=0$

It holds that $\argument=0$ and $\direction>0$ by assumption and, therefore, that $\argument+\directionlim\direction>0$ for small enough $\directionlim\in(0,\infty)$.
These observations combined with the definition of \namelrelu\ ensure that
$\functionF{\argument}=0$ and 
$\functionF{\argument+\directionlim\direction}=\directionlim\direction$ for small enough $\directionlim\in(0,\infty)$.
Therefore, similarly as above,
\begin{equation*}
    \diffdirectional\functionF{\argument}= \lim_{\directionlim\to0^+}\frac{\functionF{\argument+\directionlim\direction}-\functionF{\argument}}{\directionlim}
= \lim_{\directionlim\to0^+}\frac{\directionlim\direction-0}{\directionlim}
= \lim_{\directionlim\to0^+}\direction=\direction\,,
\end{equation*}
as desired.

\emph{Case 3:} $\direction\in\R$, $\argument\in(-\infty,0)$

It holds that $\argument<0$ by assumption and, therefore, that $\argument+\directionlim\direction<0$ for small enough $\directionlim\in(0,\infty)$.  
These two observations combined with the definition of \namelrelu\ ensure that
$\functionF{\argument}=\functionpar\argument$ and 
$\functionF{\argument+\directionlim\direction}=\functionpar(\argument+\directionlim\direction)$ for small enough $\directionlim\in(0,\infty)$.
Using these equalities and basic algebra yields similarly as before 
\begin{equation*}
    \diffdirectional\functionF{\argument}= \lim_{\directionlim\to0^+}\frac{\functionF{\argument+\directionlim\direction}-\functionF{\argument}}{\directionlim}
= \lim_{\directionlim\to0^+}\frac{\functionpar(\argument+\directionlim\direction)-\functionpar\argument}{\directionlim}
= \lim_{\directionlim\to0^+}\functionpar\direction=\functionpar\direction\,,
\end{equation*}
as desired.

\emph{Case 4:} $\direction\in(-\infty,0)$, $\argument=0$

It holds that $\argument=0$ and $\direction<0$ by assumption and, therefore, that $\argument+\directionlim\direction<0$ for small enough $\directionlim\in(0,\infty)$. 
These two observations combined with the definition of \namelrelu\ ensure that
$\functionF{\argument}=0$ and 
$\functionF{\argument+\directionlim\direction}=\functionpar \directionlim\direction$ for small enough $\directionlim\in(0,\infty)$.
Using these equalities and basic algebra yields
\begin{equation*}
    \diffdirectional\functionF{\argument}= \lim_{\directionlim\to0^+}\frac{\functionF{\argument+\directionlim\direction}-\functionF{\argument}}{\directionlim}
= \lim_{\directionlim\to0^+}\frac{\functionpar \directionlim\direction-0}{\directionlim}
= \lim_{\directionlim\to0^+}\functionpar \direction=\functionpar \direction\,,
\end{equation*}
as desired.
\end{proof}

\subsection{Expressivities of the Piecewise-Linear Functions}

\begin{example}[Expressivities of \namelinear, \namerelu, and \namelrelu]
\label{ex:linear}
The \marginterm{expressivity}\emph{expressivity} of a class of networks is its capability to approximate different functions.
In this example,
we illustrate the limited expressivity of \namelinear\ networks as compared to \namerelu\ networks.
The concatenated neurons on Page~\pageref{concatenated} 
simplify in the case of purely \namelinear\ activation to
\begin{align*}
  \neuronAPL\hspace{-1mm} \begin{bmatrix}
                \neuronALOF{\datain}\\
\vdots\\
\neuronALKF{\datain}
               \end{bmatrix}&=\neuronparBP+\sum_{k=1}^{\nbrwidth}\neuronparPE_k\Biggl(\neuronparB^k+\sum_{j=1}^{\nbrinput}(\neuronpar^k)_j\datainE_j\Biggr)\\
&=\underbrace{\neuronparBP+\sum_{k=1}^{\nbrwidth}\neuronparPE_k\neuronparB^k}_{=:\kappa}+\sum_{j=1}^{\nbrinput}\underbrace{\Biggl(\sum_{k=1}^{\nbrwidth}\neuronparPE_k(\neuronpar^k)_j\Biggr)}_{=:\eta_j}\datainE_j\\
&=\neuron_{\kappa,\boldsymbol{\eta},\functionlinear}[\datain]\,,
\end{align*}
which is a single \namelinear\ neuron.
Thus, subsequent layers with \namelinear\ activation collapse into a \namelinear\ layer,
which means that fitting data-generating processes that are nonlinear requires the inclusion of other activations somewhere in the network.

In contrast, \namerelu\ layers (and one can expect that \namelrelu\ behaves very similarly as \namerelu\ in terms of expressivity) only collapse if the weights are nonnegative;
see \citet[Theorem~1]{Hebiri20} for a precise formulation of this feature and  \citet[Section~2.3]{Hebiri20} for a description of how it can be leveraged for regulating layer depths.
More generally,
one can show that \namerelu\ networks can approximate a range of nonlinear functions;
see \citet{Corlay2019} and references therein.

\end{example}

\subsection{Dying-\namerelu\ Phenomenon}
\label{sec:dyingrelu}

\later{
just cite the other paper here later
\begin{example}[Vanishing- and exploding-gradient problems]\label{vanishing}
In this section, we summarize the vanishing-gradient problem.
As it turns out, a very simple toy model is sufficient to explain both the problem and potential remedies.
Consider a network of~$\nbrlayers$ concatenated \nametanh-neurons with one parameter each, 
that is, 
consider the functions
  \begin{align*}
\R\,&\to\,\R\\
  \datainE\,&\mapsto\, \mathfrak{g}_{(\neuronparE^1,\dots,\neuronparE^l)}[\datainE]\deq\functiontanhFBB{\neuronparE^{\nbrlayers}\functiontanhFB{\neuronparE^{\nbrlayers-1}\cdots\functiontanhF{\neuronparE^1 \datainE}}}
  \end{align*}
parametrized by $\neuronparE^1,\dots,\neuronparE^{\nbrlayers}\in\R$.
Despite its simplicity,
this network has all elements needed to stage the problem. 

The goal of training algorithms is to fit the network parameters to data.
Stochastic-gradient descent, the most popular example, updates the parameters sequentially.
In its simples form,
stochastic-gradient descent updates given parameters $\neuronparE^1,\dots,\neuronparE^{\nbrlayers}$ by a gradient step with respect to the least-squares loss 
\begin{equation*}
  \bigl(\dataoutEnew-\mathfrak{g}_{(\neuronparE^1,\dots,\neuronparE^{\nbrlayers})}[\datainEnew]\bigr)^2
\end{equation*}
of a new sample $(\dataoutEnew,\datainEnew)\in \R\times \R$.
If the gradient is very small in absolute value (or even equal to zero),
the gradient update has almost no impact on the parameters,
that is, 
there is no progress in learning the parameters.
If this happens at unsatisfactory parameter values and repeatedly for new samples,
we speak of a \emph{vanishing-gradient problem.}

We now identify two sources of the vanishing-gradient problem.
We first introduce the shorthand
\begin{equation*}
  \function^k\deq\functiontanhFBB{\neuronparE^{k}\functiontanhFB{\neuronparE^{k-1}\cdots\functiontanhF{\neuronparE^1 \datainEnew}}}~~~~\text{for all}~k\in\{1,\dots,\nbrlayers\}\,.
\end{equation*}
Observe that $\function^k\in(-1,1)$ by definition of \nametanh.
The partial derivative of the least-squares loss with respect to $\neuronparE^k$ (that is, the $k$th element of the gradient) is then
\begin{equation*}
 \frac{\partial}{\partial\neuronparE^k}\bigl(\dataoutEnew-\mathfrak{g}_{(\neuronparE^1,\dots,\neuronparE^l)}[\datainEnew]\bigr)^2=
-2\bigl(\dataoutEnew-\mathfrak{g}_{(\neuronparE^1,\dots,\neuronparE^{\nbrlayers})}[\datainEnew]\bigr)\bigl(1-(\function^{\nbrlayers})^2\bigr)\cdots\bigl(1-(\function^k)^2\bigr)\,,
\end{equation*}
where we have used the chain rule and the well-known fact $\nametanh'=1-\nametanh^2$.
The partial derivative can be small in absolute value despite an unsatisfactory model fit, that is, despite $\abs{\dataoutEnew-\mathfrak{g}_{(\neuronparE^1,\dots,\neuronparE^{\nbrlayers})}[\datainEnew]}$ large,
for two reasons.
First, one the factors $1-(\function^{\nbrlayers})^2,\dots, 1-(\function^k)^2$ can be very small in absolute value,
which is the case when one of the parameters $\neuronparE^{k},\dots,\neuronparE^{\nbrlayers}$ is large (recall that $\abs{\functiontanhF{\argument}}\to\infty$ for $\argument\to\pm\infty$).
The other inner derivatives cannot balance the small factor in view of $\abs{1-(\function^{\nbrlayers})^2},\dots, \abs{1-(\function^k)^2}<1$.
For further reference, we call this source of small gradients the \emph{large-parameter regime}.

But large parameters are not a necessary condition for small gradients.
Consider, for example, $\neuronparE^1,\dots,\neuronparE^{\nbrlayers}=0$\jl{not true, need to be away from zero use $1-(tanh(1))^2\leq 1/2$}.
Then,
\begin{equation*}
-2\bigl(\dataoutEnew-\mathfrak{g}_{(\neuronparE^1,\dots,\neuronparE^{\nbrlayers})}[\datainEnew]\bigr)\bigl(1-(\function^{\nbrlayers})^2\bigr)\cdots\bigl(1-(\function^k)^2\bigr)=-2\bigl(\dataoutEnew-\mathfrak{g}_{(\neuronparE^1,\dots,\neuronparE^{\nbrlayers})}[\datainEnew]\bigr) \frac{1}{2^{\nbrlayers-k}}\,,
\end{equation*}
which can be very small for large $\nbrlayers-k$, taht is, at the inner layers of deep networks.
This observation illustrates that predominately a problem in deep learning.

Observe that the problem is not about how hte gradient is computed (for example, with backpropagation) but the gradient itself.
Also, Problem is not a problem of deep learning per se, but of traning them with gradients.
Aslo, not least-square, that is just for illustration, and not tanh.

One could use many approaches to solve this, chaning learning rates,  batch normalization, layer wise-updates, and so one.
But the appraoch that got popular  is the change from sigmoid activation functions to piecewise-lineaer functions.

also talk about exploding gradient problem, which is the same argument, use that relu is unbounded, while sigmoid is bounded, so cannot happend there. Say that somewhat easer to avoid, as you can clip the gradients or use regularization!!!.
 
\jl{cite the diploma thesis there which is mentioned on wiki}
\end{example}}

\begin{example}[Dying-\namerelu\ phenomenon]
\label{dyingrelu}
We illustrate the dying-\namerelu\ and revitalization phenomena in a toy model.
We consider networks that have real-valued inputs and outputs and that consist of a \namerelu\ and a \nametanh\ layer that each has one neuron and no bias term;
in other words,
we consider the functions
  \begin{align*}
\R\,&\to\,\R\\
  \datainE\,&\mapsto\, \functionreluFB{\neuronparPE\functiontanhF{\neuronparE \datainE}}
  \end{align*}
parametrized by $\neuronparPE,\neuronparE\in\R$.
Given data $(\datainE_1,\dataoutE_1),\dots,(\datainE_{\nbrsamples},\dataoutE_{\nbrsamples})\in\R\times \R$,
we want to fit the parameters by optimizing the usual least-squares loss
\begin{equation*}
  (\neuronparPE,\neuronparE)\,\mapsto\, \sum_{i=1}^{\nbrsamples}\Bigl(\dataoutE_i-\functionreluFB{\neuronparPE\functiontanhF{\neuronparE \datainE_i}}\Bigr)^2\,.
\end{equation*}

A standard method for such optimizations is stochastic-gradient descent.
The $i$th updates for $\neuronparPE$ and $\neuronparE$ of stochastic-gradient descent with step size~$\stepsize_i\in(0,\infty)$ at $(\neuronparPE_i,\neuronparE_i)\in\R\times\R$ are (assuming for simplicity that $\neuronparE \datainE_i\neq 0$ to ensure differentiability)
\begin{multline*}
  -\stepsize_i\frac{\partial}{\partial\neuronparPE}\bigg|_{(\neuronparPE,\neuronparE)=(\neuronparPE_i,\neuronparE_i)} \Bigl(\dataoutE_i-\functionreluFB{\neuronparPE\functiontanhF{\neuronparE \datainE_i}}\Bigr)^2\\
\text{and}~~~-\stepsize_i\frac{\partial}{\partial\neuronparE}\bigg|_{(\neuronparPE,\neuronparE)=(\neuronparPE_i,\neuronparE_i)} \Bigl(\dataoutE_i-\functionreluFB{\neuronparPE\functiontanhF{\neuronparE \datainE_i}}\Bigr)^2\,.
\end{multline*}
One can verify readily (use Lemmas~\ref{difftanh} and~\ref {piecewiseder} and the chain rule)
that 
\begin{multline*}
  -\stepsize_i\frac{\partial}{\partial\neuronparPE} \bigg|_{(\neuronparPE,\neuronparE)=(\neuronparPE_i,\neuronparE_i)}\Bigl(\dataoutE_i-\functionreluFB{\neuronparPE\functiontanhF{\neuronparE \datainE_i}}\Bigr)^2=0\\
\text{and}~~~-\stepsize_i\frac{\partial}{\partial\neuronparE} \bigg|_{(\neuronparPE,\neuronparE)=(\neuronparPE_i,\neuronparE_i)}\Bigl(\dataoutE_i-\functionreluFB{\neuronparPE\functiontanhF{\neuronparE \datainE_i}}\Bigr)^2=0
\end{multline*}
if the optimization is in a state with $\neuronparPE_i,\neuronparE_i>0$ and $\datainE_i<0$.
Hence, the parameters remain unchanged.
The underlying reason is that under the stated conditions,
the function $\eta\mapsto \functionreluF{\eta}$ is constant and equal to zero in an environment around $\eta_i\deq \neuronparPE_i\functiontanhF{\neuronparE_i \datainE_i}$,
that is,
the \namerelu\ node does not transmit information;
we could that the \namerelu\ neuron is \marginterm{inactive \namerelu}\emph{inactive}.

But the neurons can become active again.
Assume that $\datainE_{i+1}>0$.
Then, one can verify readily  (use again Lemmas~\ref{difftanh} and~\ref {piecewiseder} and the chain rule) that 
\begin{multline*}
  -\stepsize_{i+1}\frac{\partial}{\partial\neuronparPE} \bigg|_{(\neuronparPE,\neuronparE)=(\neuronparPE_{i+1},\neuronparE_{i+1})}\Bigl(\dataoutE_{i+1}-\functionreluFB{\neuronparPE\functiontanhF{\neuronparE \datainE_{i+1}}}\Bigr)^2\neq0\\
\text{and}~~~-\stepsize_{i+1}\frac{\partial}{\partial\neuronparE} \bigg|_{(\neuronparPE,\neuronparE)=(\neuronparPE_{i+1},\neuronparE_{i+1})}\Bigl(\dataoutE_{i+1}-\functionreluFB{\neuronparPE\functiontanhF{\neuronparE \datainE_{i+1}}}\Bigr)^2\neq 0
\end{multline*}
as long as $\functionreluFB{\neuronparPE_{i+1}\functiontanhF{\neuronparE_{i+1} \datainE_{i+1}}}\neq \dataoutE_{i+1}$.
Hence, the parameters are updated in a nontrivial way,
and we can say that the \namerelu\ node is \marginterm{active \namerelu}\emph{active} again.

These observations indicate 1.~that dead \namerelu\ nodes are common in optimization steps but also 2.~that \namerelu\ nodes rarely stay dead during the entire optimization.
In special cases, 
or simply if there are many \namerelu\ nodes,
some nodes can be inactive for all or almost all of the training process;
we then say that those nodes are \marginterm{dead \namerelu} \emph{dead}.
Dead \namerelu\ nodes can be undesired if they are abundant,
because then, they can avoid learning complex models.
We then speak of the \marginterm{dying-\namerelu\\phenomenon}\emph{dying-\namerelu\ phenomenon}.
\end{example}

\subsection{Properties of the Other Functions}
\subsubsection{Derivatives of \namesoftplus}

\begin{lemma}[Derivatives of \namesoftplus]
\label{dersoftplus}
The first and second derivatives of \namesoftplus\ are
\begin{equation*}
   \functionsoftplusDF{\argument}=\frac{1}{1+e^{-\argument}}=\functionlogF{\argument}\in(0,1)
\end{equation*}
and
\begin{equation*}
\functionsoftplusDDF{\argument}=\frac{e^{-\argument}}{(1+e^{-\argument})^2}= \functionlogDF{\argument}=\functionlogF{\argument}\bigl(1-\functionlogF{\argument}\bigr)\in(0,1/4] \,,
\end{equation*}
respectively, for all $\argument\in\R$.
  
\end{lemma}

\begin{proof}[Proof of Lemma~\ref{dersoftplus}]
  Observe that 
\begingroup
\allowdisplaybreaks
  \begin{align*}
    \frac{\partial}{\partial\argument}\functionsoftplusF{\argument}&=\frac{\partial}{\partial\argument}\log[1+e^{\argument}]\tj{definition of \functionsoftplus}\\
&=\frac{\frac{\partial}{\partial\argument}(1+e^{\argument})}{1+e^{\argument}}\tj{$(\log[\argument])'=1/\argument$ for $\argument>0$; chain rule}\\
&=\frac{e^{\argument}}{1+e^{\argument}}\tj{$1'=0$; $(e^{\argument})'=e^{\argument}$; sum rule}\\
&=\frac{1}{1+e^{-\argument}}\tj{multiplying numerator and denominator by $e^{-\argument}$}\\
&=\functionlogF{\argument}\tj{definition of \functionlog}\,.
  \end{align*}
\endgroup
The rest follows from Section~\ref{logistic} and Lemma~\ref{difflogistic} on \namelog.
\end{proof}

\subsubsection{Output Ranges and Derivatives of \nameelu\ and \nameselu}
\label{eluprop}

\begin{lemma}[Output ranges and derivatives of \nameelu]
\label{eluproplem}
  The graphs of~\nameelu\ satisfy
  \begin{equation*}
   \lim_{\argument\to-\infty} \functioneluF{\argument}=   \inf_{\argument\in\R}\functioneluF{\argument}=-\functionpar~~~~\text{and}~~~~\functioneluF{\argument}>-\functionpar\text{~for all}~\argument\in\R
  \end{equation*}
for all $\functionpar\in[0,\infty)$.
The first and second derivatives of~\nameelu\ are
\begin{equation*}
  \functioneluDF{\argument}=
  \begin{cases}
    1&\text{for all}~\argument\in(0,\infty)\\
\functionpar e^{\argument}=\functioneluF{\argument}+\functionpar&\text{for all}~\argument\in(-\infty,0)
  \end{cases}
\end{equation*}
and
\begin{equation*}
 \functioneluDDF{\argument}=
  \begin{cases}
    0&\text{for all}~\argument\in(0,\infty)\\
\functionpar e^{\argument}=\functioneluF{\argument}+\functionpar&\text{for all}~\argument\in(-\infty,0)
  \end{cases}
\end{equation*}
for all $\functionpar\in[0,\infty)$.
The first directional derivatives of~\nameelu\ are 
\begin{equation*}
  \diffdirectional\functioneluF{\argument}=
  \begin{cases}
    \direction&\text{for all}~\direction\in\R~\text{and}~\argument\in(0,\infty)~\text{or}~\direction\in[0,\infty)~\text{and}~\argument=0\\
\direction\bigl(\functioneluF{\argument}+\functionpar\bigr)&\text{otherwise}
  \end{cases}
\end{equation*}
for all $\functionpar\in[0,\infty)$.
\end{lemma}
\noindent The corresponding properties of~\nameselu\ can be derived along the same lines.
Observe that $\lim_{\argument\to 0^+}\functioneluODF{\argument}=\lim_{\argument\to 0^-}\functioneluODF{-\argument}$,
that is,
\nameelu\ with $\argument=1$ is one time differentiable on the entire real line.

\begin{proof}[Proof of Lemma~\ref{eluproplem}]
The first claim follows from the observations that \nameelu\ is strictly increasing on the real line (both $\argument\mapsto\argument$ and $\argument\mapsto\functionpar(e^{-\argument}-1)$ are strictly increasing and $\functioneluF{\argument_1}<0< \functioneluF{\argument_2}$ for all $\argument_1\in(-\infty,0),\argument_2\in(0,\infty)$) 
and that
\begin{align*}
    \lim_{\argument\to-\infty}\functioneluF{\argument}&= \lim_{\argument\to-\infty}\functionpar(e^{-\argument}-1)\tj{$\functioneluF{\argument}=\functionpar(e^{-\argument}-1)$ for $\argument$ small enough}\\
&=-\functionpar\tj{$e^{-\argument}\to 0$ for $\argument\to-\infty$}\,.
\end{align*}

The first and second derivatives follow from standard differential calculus.

We illustrate the derivations of the directional derivatives in the case $\argument=0$ and $\direction\in(-\infty,0)$---the other cases can be treated in the same way.
Given any such~$\argument$ and~$\direction$ as well as an arbitrary~$\functionpar$, 
we can make three simple observations:

\emph{Observation~(i):} $\functioneluF{0}=0$.

\noindent This observation follows directly from the definition of \functionelu.

\emph{Observation~(ii):} $\functioneluF{0+\directionlim\direction}=\functionpar(e^{\directionlim\direction}-1)$.

\noindent This observation follows from the definition of  \functionelu\ and the fact that $\directionlim\direction<0$ for all $\directionlim\in(0,\infty)$ and $\direction\in(-\infty,0)$

\emph{Observation~(iii):} $(\functionpar e^{\directionlim\direction}-\functionpar)/\directionlim\to \functionpar\direction$ for $\directionlim\to0^+$.

\noindent Using Lemma~\ref{res:krankenhaus} (l'H\^opital's rule) and basic algebra yields
\begin{align*}
  \lim_{\directionlim\to0^+}\frac{\functionpar e^{\directionlim\direction}-\functionpar}{\directionlim}&=\lim_{\directionlim\to0^+}\frac{(\functionpar e^{\directionlim\direction}-\functionpar)'}{(\directionlim)'}\tj{2.~in Lemma~\ref{res:krankenhaus}}\\
&=
\lim_{\directionlim\to0^+}\frac{\functionpar\direction e^{\directionlim\direction}-0}{1}\tj{sum and chain rules for differentiation}\\
&=\lim_{\directionlim\to0^+}\functionpar\direction e^{\directionlim\direction}\tj{consolidation}\\
&=\functionpar\direction\tj{$e^{0\cdot\direction}=1$}
\end{align*}
(where we have used l'H\^opital's rule with 
$\function\,:\,\directionlim\mapsto\functionpar e^{\directionlim\direction}-\functionpar$, 
$\functionG\,:\,\directionlim\mapsto\directionlim$, 
and
$\directionlimZ=0$),
as desired

Using these three observations together with the definition of directional derivatives on Page~\pageref{directional} yields for all $\functionpar\in[0,\infty)$ that
\begingroup
\allowdisplaybreaks
\begin{align*}
    \diffdirectional\functioneluF{0}&= \lim_{\directionlim\to0^+}\frac{\functioneluF{0+\directionlim\direction}-\functioneluF{0}}{\directionlim}\tj{Definition~\ref{directional} (directional derivatives)}\\
&= \lim_{\directionlim\to0^+}\frac{\functionpar(e^{\directionlim\direction}-1)-0}{\directionlim}\tj{Observations (i) and (ii)}\\
&= \lim_{\directionlim\to0^+}\frac{\functionpar e^{\directionlim\direction}-\functionpar}{\directionlim}\tj{consolidation}\\
&=\functionpar\direction\tj{Observation (iii)}\\
&=\direction(0+\functionpar)\tj{adding a zero-valued term}\\
&=\direction\bigl(\functioneluF{0}+\functionpar\bigr)\tj{$\functioneluF{0}=0$ by definition}\,,
\end{align*}
\endgroup
as desired.
\end{proof}

\subsubsection{Further Details on \nameselu}\label{seludetails}
Our formulation of \nameselu\ in Section~\ref{elu} differs slightly from the original formulation \citet[Equation~(1)]{Klambauer2017} in that we set $\functionpar_0\deq\lambda$ and $\functionparP_0\deq\lambda\alpha$ for conciseness.

More precise values for the constants are $\functionpar_0\approx 1.05070098$ and $\functionparP_0\approx1.7580993261$.
Analytical expressions are stated in~\citet[Equation~(8) in the Supplementary Material]{Klambauer2017}.

A precise version of the statements about~$\mathfrak{c}$ is that  $\normtwo{\mathfrak{c}[(\mu,\nu)]-(\overline{\mu},\overline{\nu})}<\normtwo{(\mu,\nu)-(\overline{\mu},\overline{\nu})}$ and $\mathfrak{c}[(\overline{\mu},\overline{\nu})]=(\overline{\mu},\overline{\nu})$ for all $(\mu,\nu)\in\R^2$.

\subsubsection{Derivatives of \nameswish}

\begin{lemma}[Derivatives of \nameswish]
\label{swishdiff}
The first and second derivatives of \nameswish\ are 
\begin{equation*}
  \functionswishDF{\argument}
  =\frac{1+(1+\functionpar \argument) e^{-\functionpar\argument}}{(1+e^{-\functionpar\argument})^2}=\functionpar\functionswishF{\argument}+\functionlogF{\functionpar\argument}\bigl(1-\functionpar\functionswishF{\argument}\bigr)
\end{equation*}
and
\begin{multline*}
  \functionswishDDF{\argument}=\functionpar e^{-\functionpar\argument} \cdot\frac{2-\functionpar\argument+(2+\functionpar\argument) e^{-\functionpar\argument}}{(1+e^{-\functionpar\argument})^3}\\
=\functionpar\Bigl(\functionpar\functionswishF{\argument}+(2+\functionpar\argument)\functionlogF{\functionpar\argument}\bigl(1-\functionpar\functionswishF{\argument}\bigr)\Bigr)\bigl(1-\functionlogF{\functionpar\argument}\bigr)
\end{multline*}
for  all $\argument\in\R$ and $\functionpar\in[0,\infty)$.
\end{lemma}
\noindent

\begin{proof}[Proof of Lemma~\ref{swishdiff}]
Using the basic rules for differentiation and the definitions of \functionswish\ and \functionlog,
we find for the first derivative
\begingroup
\allowdisplaybreaks
  \begin{align*}
\functionswishDF{\argument}
&=\frac{\partial}{\partial\argument}\frac{\argument}{1+e^{-\functionpar\argument}}\tj{definition of \nameswish}\\
&= \frac{\bigl(\frac{\partial}{\partial\argument} \argument\bigr)(1+e^{-\functionpar\argument})- \argument \frac{\partial}{\partial\argument}(1+e^{-\functionpar\argument})}{(1+e^{-\functionpar\argument})^2}\tj{quotient rule}\\
&= \frac{1\cdot(1+e^{-\functionpar\argument})- \argument\cdot (0-\functionpar e^{-\functionpar\argument})}{(1+e^{-\functionpar\argument})^2}\tj{$\argument'=1$; $1'=0$; $(e^{-\functionpar\argument})'=-\functionpar e^{-\functionpar\argument}$; sum rule}\\
&= \frac{1+(1+\functionpar \argument) e^{-\functionpar\argument}}{(1+e^{-\functionpar\argument})^2}\tj{simplifying}\,, 
\end{align*}
\endgroup
which proves the first part of the equality of the derivative.
We then find further
\begin{align*}
\functionswishDF{\argument}
&= \frac{1+(1+\functionpar \argument) e^{-\functionpar\argument}}{(1+e^{-\functionpar\argument})^2}\tj{previous display}\\
&=\frac{\functionpar\argument(1+e^{-\functionpar\argument}) -\functionpar\argument+1+ e^{-\functionpar\argument}}{(1+e^{-\functionpar\argument})^2}\tj{rearranging terms and adding a zero-valued term}\\
  &=\frac{\functionpar\argument(1+e^{-\functionpar\argument})}{(1+e^{-\functionpar\argument})^2}+ \frac{-\functionpar\argument+1+ e^{-\functionpar\argument}}{(1+e^{-\functionpar\argument})^2}\tj{splitting up the fraction}\\
  &=\functionpar\cdot \frac{\argument}{1+e^{-\functionpar\argument}}+\frac{1}{1+e^{-\functionpar\argument}}\cdot \frac{1+ e^{-\functionpar\argument}-\functionpar\argument}{1+e^{-\functionpar\argument}}\tj{simplifying and rearranging}\\
  &=\functionpar\cdot \frac{\argument}{1+e^{-\functionpar\argument}}+\frac{1}{1+e^{-\functionpar\argument}}\cdot \Bigl(1-\functionpar\cdot\frac{\argument}{1+e^{-\functionpar\argument}}\Bigr)\tj{rearranging the last factor further} \\
  &=\functionpar\functionswishF{\argument}+\functionlogF{\functionpar\argument}\bigl(1-\functionpar\functionswishF{\argument}\bigr)\tj{definitions of \nameswish\ and \namelog}\,,
\end{align*}
as desired.

\detail{The second  equality can also be established directly via the composite form of \nameswish: 
\begin{align*}
  \functionswishDF{\argument}
&=\frac{\partial}{\partial\argument}\bigl(\argument\cdot\functionlogF{\functionpar\argument}\bigr)\tj{definition of \nameswish}\\
&=\Bigl(\frac{\partial}{\partial\argument}\argument\Bigl)\functionlogF{\functionpar\argument}+\argument\Bigl(\frac{\partial}{\partial\argument}\functionlogF{\functionpar\argument}\Bigr)\tj{product rule}\\
  &=\functionlogF{\functionpar\argument}+\argument\cdot \functionpar \frac{\partial}{\partial\argumentP}\functionlogF{\argumentP}\Big|_{\argumentP=\functionpar\argument}\tj{$\argument'=1$;  chain rule}\\
  &=\functionlogF{\functionpar\argument}+\argument\cdot\functionpar\functionlogF{\functionpar\argument}\bigl(1-\functionlogF{\functionpar\argument}\bigr)\tj{Lemma~\ref{difflogistic} (derivative of \namelog)}\\
&=\functionlogF{\functionpar\argument}+\functionpar\functionswishF{\argument}\bigl(1-\functionlogF{\functionpar\argument}\bigr)\tj{definition of \nameswish}\\
  &=\functionpar\functionswishF{\argument}+\functionlogF{\functionpar\argument}\bigl(1-\functionpar\functionswishF{\argument}\bigr)\tj{rearranging the terms}\,.
\end{align*}}

Similarly, the second derivative of \nameswish\ can be calculated based on the explicit form or the composite form of the first derivative;
we opt for the latter:
\begingroup
\allowdisplaybreaks
  \begin{align*}
\functionswishDDF{\argument}
&=\frac{\partial}{\partial\argument}\functionswishDF{\argument}\tj{definition of the second derivative}\\
&=\frac{\partial}{\partial\argument}\Bigl(\functionpar\functionswishF{\argument}+\functionlogF{\functionpar\argument}\bigl(1-\functionpar\functionswishF{\argument}\bigr)\Bigr)\tj{above derivations}\\
&=\Bigl(\frac{\partial}{\partial\argument}\bigl(\functionpar\functionswishF{\argument}\bigr)\Bigr)\\
&~~~+\Bigl(\frac{\partial}{\partial\argument}\functionlogF{\functionpar\argument}\Bigr)\bigl(1-\functionpar\functionswishF{\argument}\bigr)\\
&~~~+\functionlogF{\functionpar\argument}\Bigl(\frac{\partial}{\partial\argument}\bigl(1-\functionpar\functionswishF{\argument}\bigr)\Bigr)\tj{sum and product rules}\\
&=\functionpar\Bigl(\frac{\partial}{\partial\argument}\functionswishF{\argument}\Bigr)\\
&~~~+\functionpar\Bigl(\frac{\partial}{\partial\argumentP}\functionlogF{\argumentP}\Bigr)\Big|_{\argumentP=\functionpar\argument}\bigl(1-\functionpar\functionswishF{\argument}\bigr)\\
&~~~+\functionlogF{\functionpar\argument}\Bigl(\frac{\partial}{\partial\argument}1\Bigr)-\functionpar\functionlogF{\functionpar\argument}\Bigl(\frac{\partial}{\partial\argument}\functionswishF{\argument}\Bigr)\tj{chain and sum rules}\\
&=\functionpar\Bigl(\functionpar\functionswishF{\argument}+\functionlogF{\functionpar\argument}\bigl(1-\functionpar\functionswishF{\argument}\bigr)\Bigr)\tj{above derivations for \functionswishD}\\
&~~~+\functionpar\functionlogF{\functionpar\argument}\bigl(1-\functionlogF{\functionpar\argument}\bigr)\bigl(1-\functionpar\functionswishF{\argument}\bigr)\tj{Lemma~\ref{difflogistic} (derivative of \namelog)}\\
&~~~-\functionpar\functionlogF{\functionpar\argument}\Bigl(\functionpar\functionswishF{\argument}+\functionlogF{\functionpar\argument}\bigl(1-\functionpar\functionswishF{\argument}\bigr)\Bigr)\tj{$1'=0$; above derivations for \functionswishD}\\
&=\functionpar\Bigl(\functionpar\functionswishF{\argument}+2\functionlogF{\functionpar\argument}-2\bigl(\functionlogF{\functionpar\argument}\big)^2-3\functionpar\functionlogF{\functionpar\argument}\functionswishF{\argument}\\
    &~~~+2\functionpar\bigl(\functionlogF{\functionpar\argument}\bigr)^2\functionswishF{\argument}\tj{summarizing the terms}\Bigr)\\
&=\functionpar\Bigl(\functionpar\functionswishF{\argument}\bigl(1-\functionlogF{\functionpar\argument}\bigr)+2\functionlogF{\functionpar\argument}\bigl(1-\functionpar\functionswishF{\argument}\bigr)\\
&~~~-2\bigl(\functionlogF{\functionpar\argument}\bigr)^2\bigl(1-\functionpar\functionswishF{\argument}\bigr)\Bigr)\tj{rearranging the terms}\\
&=\functionpar\Bigl(\functionpar\functionswishF{\argument}\bigl(1-\functionlogF{\functionpar\argument}\bigr)+2\functionlogF{\functionpar\argument}\bigl(1-\functionlogF{\functionpar\argument}\bigr)\bigl(1-\functionpar\functionswishF{\argument}\bigr)\Bigr)\tj{summarizing the terms further}\\
&=\functionpar\Bigl(\functionpar\functionswishF{\argument}+2\functionlogF{\functionpar\argument}\bigl(1-\functionpar\functionswishF{\argument}\bigr)\Bigr)\bigl(1-\functionlogF{\functionpar\argument}\bigr)\tj{simplifying one more time}
\,.
\end{align*}
\endgroup
We can then find further
\begingroup
\allowdisplaybreaks
\begin{align*}
  \functionswishDDF{\argument}
&=\functionpar\Bigl(\functionpar\functionswishF{\argument}+2\functionlogF{\functionpar\argument}\bigl(1-\functionpar\functionswishF{\argument}\bigr)\Bigr)\bigl(1-\functionlogF{\functionpar\argument}\bigr)\tj{previous display}\\
&=\functionpar\Bigl(2\functionpar\functionswishF{\argument}+2\functionlogF{\functionpar\argument}\bigl(1-\functionpar\functionswishF{\argument}\bigr)-\functionpar\functionswishF{\argument}\Bigr)\bigl(1-\functionlogF{\functionpar\argument}\bigr)\tj{adding a zero-valued term}\\
&=\functionpar\bigl(2\functionswishDF{\argument}
-\functionpar\functionswishF{\argument}\bigr)\bigl(1-\functionlogF{\functionpar\argument}\bigr)\tj{above-derived result for \functionswishD}\\
&=\functionpar\biggl(2\cdot\frac{1+(1+\functionpar \argument) e^{-\functionpar\argument}}{(1+e^{-\functionpar\argument})^2}-\functionpar\cdot\frac{\argument}{1+e^{-\functionpar\argument}}\biggr)\biggl(1-\frac{1}{1+e^{-\functionpar\argument}}\biggr)\tj{plugging in the explicit forms of \functionswishD,\functionswish,\functionlog}\\
&=\functionpar \cdot\frac{2+2(1+\functionpar \argument) e^{-\functionpar\argument}-\functionpar\argument(1+e^{-\functionpar\argument})}{(1+e^{-\functionpar\argument})^2}\cdot\frac{1+e^{-\functionpar\argument}-1}{1+e^{-\functionpar\argument}}\tj{combining the summands}\\
&=\functionpar \cdot\frac{2-\functionpar\argument+(2+\functionpar\argument) e^{-\functionpar\argument}}{(1+e^{-\functionpar\argument})^2}\cdot\frac{e^{-\functionpar\argument}}{1+e^{-\functionpar\argument}}\tj{consolidating}\\
&=\functionpar e^{-\functionpar\argument} \cdot\frac{2-\functionpar\argument+(2+\functionpar\argument) e^{-\functionpar\argument}}{(1+e^{-\functionpar\argument})^3}\tj{consolidating further}\,,
\end{align*}
\endgroup
as desired.
\detail{The explicit derivation looks as follows:
\begingroup
\allowdisplaybreaks
    \begin{align*}
&\functionswishDDF{\argument}\\
&=\frac{\partial}{\partial\argument}\functionswishDF{\argument}\tj{definition of the second derivative}\\
&=\frac{\partial}{\partial\argument}\biggl(\frac{1+(1+\functionpar \argument) e^{-\functionpar\argument}}{(1+e^{-\functionpar\argument})^2}\biggr)\tj{above results}\\
&=\frac{\Bigl(\frac{\partial}{\partial\argument}\bigl(1+(1+\functionpar \argument) e^{-\functionpar\argument}\bigr)\Bigr)(1+e^{-\functionpar\argument})^2 - \bigl(1+(1+\functionpar \argument) e^{-\functionpar\argument}\bigr)\Bigl(\frac{\partial}{\partial\argument}(1+e^{-\functionpar\argument})^2\Bigr)}{(1+e^{-\functionpar\argument})^4}\tj{quotient rule}\\
&=\frac{\Bigl(\bigl(\frac{\partial}{\partial\argument}1\bigr)+\bigl(\frac{\partial}{\partial\argument}\bigr((1+\functionpar \argument) e^{-\functionpar\argument}\bigr)\bigr)\Bigr)(1+e^{-\functionpar\argument})^2 - \bigl(1+(1+\functionpar \argument) e^{-\functionpar\argument}\bigr)\cdot 2(1+e^{-\functionpar\argument}) \Bigl(\bigl(\frac{\partial}{\partial\argument}1\bigr)+\bigl(\frac{\partial}{\partial\argument}e^{-\functionpar\argument}\bigr)\Bigr)}{(1+e^{-\functionpar\argument})^4}\tj{sum and chain rules}\\
&=\frac{\Bigl(\frac{\partial}{\partial\argument}\bigr((1+\functionpar \argument) e^{-\functionpar\argument}\bigr)\Bigr)(1+e^{-\functionpar\argument})^2 - \bigl(1+(1+\functionpar \argument) e^{-\functionpar\argument}\bigr)\cdot 2(1+e^{-\functionpar\argument}) (-\functionpar e^{-\functionpar\argument})}{(1+e^{-\functionpar\argument})^4}\tj{$1'=0$; $(e^{-\functionpar\argument})'=-\functionpar e^{-\functionpar\argument}$}\\
&=\frac{\Bigl(\frac{\partial}{\partial\argument}\bigr((1+\functionpar \argument) e^{-\functionpar\argument}\bigr)\Bigr)(1+e^{-\functionpar\argument}) +2\functionpar e^{-\functionpar\argument} \bigl(1+(1+\functionpar \argument) e^{-\functionpar\argument}\bigr)}{(1+e^{-\functionpar\argument})^3}\tj{consolidating}\\
&=\frac{\Bigl(\bigl(\frac{\partial}{\partial\argument}(1+\functionpar \argument)\bigr) e^{-\functionpar\argument}\bigr)\Bigr)+(1+\functionpar \argument)\bigl(\frac{\partial}{\partial\argument} e^{-\functionpar\argument}\bigr)\Bigr)(1+e^{-\functionpar\argument}) +2\functionpar e^{-\functionpar\argument} \bigl(1+(1+\functionpar \argument) e^{-\functionpar\argument}\bigr)}{(1+e^{-\functionpar\argument})^3}\tj{product rule}\\
&=\frac{\bigl((0+\functionpar)e^{-\functionpar\argument}+(1+\functionpar \argument)(-\functionpar e^{-\functionpar\argument})\bigr)(1+e^{-\functionpar\argument}) +2\functionpar e^{-\functionpar\argument} \bigl(1+(1+\functionpar \argument) e^{-\functionpar\argument}\bigr)}{(1+e^{-\functionpar\argument})^3}\tj{$1'=0$; $(\functionpar\argument)'=\functionpar$; $(e^{-\functionpar\argument})'=-\functionpar e^{-\functionpar\argument}$; sum rule}\\
&=\functionpar e^{-\functionpar\argument}\cdot\frac{\bigl(1-(1+\functionpar \argument)\bigr)(1+e^{-\functionpar\argument}) +2 \bigl(1+(1+\functionpar \argument) e^{-\functionpar\argument}\bigr)}{(1+e^{-\functionpar\argument})^3}\tj{consolidating}\\
&=\functionpar e^{-\functionpar\argument}\cdot\frac{-\functionpar \argument-\functionpar \argument e^{-\functionpar\argument} +2+2e^{-\functionpar\argument}+2\functionpar \argument e^{-\functionpar\argument}}{(1+e^{-\functionpar\argument})^3}\tj{consolidating further and expanding}\\
&=\functionpar e^{-\functionpar\argument}\cdot\frac{2-\functionpar \argument+(2+\functionpar\argument)e^{-\functionpar\argument}}{(1+e^{-\functionpar\argument})^3}\tj{consolidating}\,,
  \end{align*}
\endgroup
}

\end{proof}

\subsubsection{Output Ranges of \nameswish\ and Its Derivatives}
\begin{lemma}[Output ranges of \nameswish]
\label{swishbound}
\nameswish\ diverges in the limit of large arguments:
\begin{equation*}
  \lim_{\argument\to+\infty}\functionswishF{\argument}=\infty~~~~\text{for all}~\functionpar\in[0,\infty)\,.
\end{equation*}
Moreover,
the graphs of \nameswish\ satisfy
\begin{equation*}
  \lim_{\argument\to-\infty}\functionswishZF{\argument}=-\infty~~~~~\text{and}~~~~~  \lim_{\argument\to-\infty}\functionswishF{\argument}=0
\end{equation*}
  as well as
\begin{equation*}
  \argmin_{\argument\in\R}\functionswishF{\argument}= \frac{\argumentminO}{\functionpar}~~~\text{and}~~~  \min_{\argument\in\R}\functionswishF{\argument}= \frac{\functionswishOF{\argumentminO}}{\functionpar}=\frac{1+\argumentminO}{\functionpar}
\end{equation*}
for all $\functionpar\in(0,\infty)$ and 
 $\argumentminO\in\R$ the unique value that fulfills $1+(1+\argumentminO)e^{-\argumentminO}=0$.
\end{lemma}
\noindent 
These results highlight  that the graph of \nameswish\ is unbounded from below if the parameter~\functionpar\ equals zero but is bounded from below otherwise.
One can verify numerically that $\argumentminO\approx -1.278$ and $\functionswishOF{\argumentminO}\approx -0.278$.

\begin{proof}[Proof of Lemma~\ref{swishbound}]
The first three claims follow almost directly from the definition of \nameswish:
for every $\functionpar\in(0,\infty)$, 
we find
\begingroup
\allowdisplaybreaks
\begin{align*}
  \lim_{\argument\to+\infty}\functionswishF{\argument}&=\lim_{\argument\to+\infty}\frac{\argument}{1+e^{-\functionpar\argument}}\tj{definition of \functionswish}\\
&=\infty\tj{$\argument\to\infty$ and $e^{-\functionpar\argument}\to 0$ for $\argument\to\infty$}
\end{align*}
and, similarly (use Lemma~\ref{res:krankenhaus} in the last line),
\begin{align*}
  &\lim_{\argument\to-\infty}\functionswishZF{\argument}=\lim_{\argument\to-\infty}\frac{\argument}{1+e^{-0\argument}}=-\infty\,,\\
\text{and}~~~~~  &\lim_{\argument\to-\infty}\functionswishF{\argument}=\lim_{\argument\to-\infty}\frac{\argument}{1+e^{-\functionpar\argument}}= \lim_{\argument\to-\infty}\frac{1}{-\functionpar e^{-\functionpar\argument}}=0\,.
\end{align*}
\endgroup

Motivated by  the fact that \nameswish\ is twice differentiable (see Lemma~\ref{swishdiff}),
we try to find the minimizers and minima of \functionswish\ with $\functionpar\in(0,\infty)$ by setting its derivatives equal to zero.
In other words,
we consider the candidates $\argumentmin\in\{\argument\in\R\,:\,\functionswishDF{\argument}=0\}$.
We find that 
\begin{align*}
   \functionswishDF{\argumentmin}&=0\\
\Rightarrow~~~\frac{1+(1+\functionpar\argumentmin)e^{-\functionpar\argumentmin}}{{(1+e^{-\functionpar\argumentmin})^2}}&=0\tj{Lemma~\ref{swishdiff} on the derivatives of \nameswish}\\
\Rightarrow~~~1+(1+\functionpar\argumentmin)e^{-\functionpar\argumentmin}&=0\tj{${(1+e^{-\functionpar\argumentmin})^2}>0$}\,.
\end{align*}
Reparameterizing then yields $\argumentmin=\argumentminO/\functionpar$ (recall the assumption that $\functionpar\neq0$) with
\begin{equation*}
  1+(1+\argumentminO)e^{-\argumentminO}=0\,.
\end{equation*}
One can verify readily the fact that exactly one such~$\argumentminO$ exists,
that is, $\argumentmin$ is unique.

We now calculate the function values for $\argumentmin$.
We find for all $\functionpar\in(0,\infty)$
\begingroup
\allowdisplaybreaks
\begin{align*}
  \functionswishF{\argumentmin}&=\functionswishF{\argumentminO/\functionpar}\tj{previous derivations}\\
&=\frac{\argumentminO/\functionpar}{1+e^{-\functionpar\cdot(\argumentminO/\functionpar)}}\tj{definition of \nameswish}\\
&=\frac{1}{\functionpar}\cdot \frac{\argumentminO}{1+e^{-1\cdot \argumentminO}}\tj{simplification}\\
&=\frac{ \functionswishOF{\argumentminO}}{\functionpar}\tj{definition of \nameswish}\,.
\end{align*}
\endgroup
We find further
\begin{align*}
  \functionswishOF{\argumentminO}&=\frac{\argumentminO}{1+e^{-\argumentminO}}\tj{definition of \nameswish}\\
&=\frac{\argumentminO}{1-1/(1+\argumentminO)}\tj{$1+(1+\argumentminO)e^{-\argumentminO}=0\Rightarrow e^{-\argumentminO}=-1/(1+\argumentminO)$} \\
&=\frac{(1+\argumentminO)\argumentminO}{1+\argumentminO-1}\tj{multiplying numerator and denominator with $1+\argumentmin)$}\\
&=\frac{(1+\argumentminO)\argumentminO}{\argumentminO}\tj{consolidating}\\
&=1+\argumentminO\tj{consolidating further}\,.
\end{align*}

We finally have to  verify that~$\argumentmin$ is indeed a minimizer of~\functionswish.
Since~\functionswish\ tends to $\infty$ and $0$ in the limits $\argument\to\infty$ and $\argument\to-\infty$, respectively (see above),
it is sufficient to show that $\functionswishF{\argumentmin}<0$.
In view of the above display,
this is equivalent to $\functionswishOF{\argumentminO}<0$,
and this fact can be easily confirmed numerically.
\end{proof}

\begin{lemma}[Output ranges of \nameswish's first derivatives]
\label{swishboundD}
The first derivative of  \nameswish\ with parameter $\functionpar=1$ is constant:
\begin{equation*}
  \functionswishZDF{\argument}=\frac{1}{2}~~~~~\text{for all}~\argument\in\R\,.
\end{equation*}
In contrast,
for all $\functionpar\in(0,\infty)$,
the graphs of the first derivatives of \nameswish\ satisfy
\begin{equation*}
 \lim_{\argument\to+\infty}\functionswishDF{\argument}=1~~~~~\text{and}~~~~~  \lim_{\argument\to-\infty}\functionswishDF{\argument}=0
\end{equation*}
  as well as
\begin{multline*}
  \argmax_{\argument\in\R}\functionswishDF{\argument}=-\argmin_{\argument\in\R}\functionswishDF{\argument}= \frac{\argumentminOP}{\functionpar}\\
\text {and}~~~ \max_{\argument\in\R}\functionswishDF{\argument}= 1-\min_{\argument\in\R}\functionswishDF{\argument}= \functionswishODF{\argumentminOP}
\end{multline*}
for $\argumentminOP\in[0,\infty)$ the unique (nonnegative) value that fulfills $2-\argumentminOP+(2+\argumentminOP)e^{-\argumentminOP}=0$.
\end{lemma}
\noindent 
A numerical evaluation yields $\argumentminOP\approx2.218$ and $\functionswishODF{\argumentminOP}\approx1.098$.
The second derivatives can be treated very similarly.

\begin{proof}[Proof of Lemma~\ref{swishboundD}]
The first claim follows readily from Lemma~\ref{swishdiff}:
  \begin{align*}
    \functionswishZDF{\argument}&=   \frac{1+(1+0\cdot\argument) e^{-0\cdot \argument}}{(1+e^{-0\cdot \argument})^2}\tj{equality for \functionswishOD\ from Lemma~\ref{swishdiff}}\\
&=   \frac{1+(1+0)\cdot 1}{(1+1)^2}=\frac{1}{2}\tj{simplification}\,,
  \end{align*}
as desired.

The second and third claims follow similarly:
for all $\functionpar\in(0,\infty)$, 
we find
  \begin{align*}
    \lim_{\argument\to+\infty}\functionswishDF{\argument}&=    \lim_{\argument\to+\infty}\frac{1+(1+\functionpar \argument) e^{-\functionpar\argument}}{(1+e^{-\functionpar\argument})^2}\tj{equality for \functionswishD\ from Lemma~\ref{swishdiff}}\\
&=    \lim_{\argument\to+\infty}\frac{1+0}{(1+0)^2}=1\tj{$\functionpar>0$; $(1+\functionpar \argument) e^{-\functionpar\argument}\to 0$; $e^{-\functionpar\argument}\to 0$ for $\argument\to+\infty$}
  \end{align*}
and
  \begin{align*}
    \lim_{\argument\to-\infty}\functionswishDF{\argument}&=    \lim_{\argument\to-\infty}\frac{1+(1+\functionpar \argument) e^{-\functionpar\argument}}{(1+e^{-\functionpar\argument})^2}\tj{equality for \functionswishD\ from Lemma~\ref{swishdiff}}\\
&=    \lim_{\argument\to-\infty}\frac{\functionpar\argument e^{-\functionpar\argument}}{(e^{-\functionpar\argument})^2}\tj{$\functionpar>0$; $(1+\functionpar\argument )e^{-\functionpar\argument}\to \functionpar\argument e^{-\functionpar\argument}\ll 1$; $e^{-\functionpar\argument}\gg 1$}\\
&=    \lim_{\argument\to-\infty}\frac{\functionpar\argument }{e^{-\functionpar\argument}}=0\tj{simplification; $\functionpar\argument/e^{-\functionpar\argument}\to0$}\,,
  \end{align*}
as desired.

For the last claims,
we proceed similarly as in the proof of Lemma~\ref{swishbound}.
The critical points $\argumentminP\in\{\argument\in\R\,:\,\functionswishDDF{\argument}=0\}$ satisfy 
\begin{align*}
\functionswishDDF{\argumentminP}&=0\\
\Rightarrow~~~~ \functionpar e^{-\functionpar\argumentminP}\frac{2-\functionpar\argumentminP+(2+\functionpar\argumentminP)e^{-\functionpar\argumentminP} }{(1+e^{-\functionpar\argumentminP})^3}&=0\tj{equality for \functionswishDD\ from Lemma~\ref{swishdiff}}\\
\Rightarrow~~~~ 2-\functionpar\argumentminP +(2+\functionpar\argumentminP)e^{-\functionpar\argumentminP}&=0\tj{$\functionpar e^{-\functionpar\argumentminP},(1+e^{-\functionpar\argumentminP})^3>0$}\,.
\end{align*}

Hence, $\argumentminP=\argumentminOP/\functionpar$ with $\argumentminOP$ the unique value that satisfies
\begin{equation*}
  2-\argumentminOP+(2+\argumentminOP)e^{-\argumentminOP}=0\,.
\end{equation*}

Observe then that
\begin{align*}
  \min_{\argument\in\R}\functionswishDF{\argument}
  &=\functionswishDF{\argumentminOP/\functionpar}\tj{above results}\\
&=\frac{1+\bigl(1+\functionpar (\argumentminOP/\functionpar) \bigr) e^{-\functionpar(\argumentminOP/\functionpar)}}{(1+e^{-\functionpar(\argumentminOP/\functionpar)})^2}\tj{derivatives of \nameswish\ from Lemma~\ref{swishdiff}}\\
&=\frac{1+(1+1\cdot\argumentminOP) e^{-1\cdot\argumentminOP}}{(1+e^{-1\cdot \argumentminOP})^2}\tj{simplification}\\
&=\functionswishOF{\argumentminOP}\tj{derivatives of \nameswish\ from Lemma~\ref{swishdiff}}\,.
\end{align*}
We find further
\begingroup
\allowdisplaybreaks
\begin{align*}
  \functionswishOF{\argumentminOP}&=\frac{1+(1+\argumentminOP) e^{-\argumentminOP}}{(1+e^{-\argumentminOP})^2}\\
&=\frac{1+(1+\argumentminOP) (\argumentminOP-2)/(2+\argumentminOP)}{\bigl(1+(\argumentminOP-2)/(2+\argumentminOP)\bigr)^2}\tj{$2-\argumentminOP+(2+\argumentminOP)e^{-\argumentminOP}=0\Rightarrow e^{-\argumentminOP}=(\argumentminOP-2)/(2+\argumentminOP)$}\\
&=\frac{(2+\argumentminOP)^2+(2+\argumentminOP)(1+\argumentminOP) (\argumentminOP-2)}{\bigl((2+\argumentminOP)+(\argumentminOP-2)\bigr)^2}\tj{multiplying numerator and denominator by $(2+\argumentminOP)^2$}\\
&=\frac{(2+\argumentminOP)^2+\bigl((\argumentminOP)^2-4\bigr)(1+\argumentminOP) }{4(\argumentminOP)^2}\tj{$(v+u)(u-v)=u^2-v^2$}\\
&=\frac{4+2\argumentminOP+(\argumentminOP)^2+(\argumentminOP)^3+(\argumentminOP)^2-4\argumentminOP-4}{4(\argumentminOP)^2}\tj{expanding the terms in the numerator}\\
&=\frac{-2\argumentminOP+2(\argumentminOP)^2+(\argumentminOP)^3}{4(\argumentminOP)^2}\tj{consolidating the numerator}\\
&=\frac{-2+2\argumentminOP+(\argumentminOP)^2}{4\argumentminOP}\tj{simplifying}\,.
\end{align*}
\endgroup

We can then proceed similarly as in the proof of Lemma~\ref{swishbound}
and use Lemma~\ref{res:swishsym} about the symmetries of~\functionswishD\ to conclude
\end{proof}

\begin{lemma}[Symmetry properties of \nameswish's first derivatives]
\label{res:swishsym}
  It holds for all $\argument\in\R$ that
  \begin{equation*}
    2-\argument+(2+\argument)e^{-\argument}=0~~~\Rightarrow~~~2-(-\argument)+\bigl(2+(-\argument)\bigr)e^{-(-\argument)}=0\,.
  \end{equation*}
Moreover, 
it holds for all $\argument\in\R$ that
\begin{equation*}
    \functionswishDF{-\argument}=1- \functionswishDF{\argument}\,.
\end{equation*}
\end{lemma}
\noindent 
The first statement illustrates that the solution set $2-\argumentminOP+(2+\argumentminOP)e^{-\argumentminOP}=0$ over the entire real line is symmetric.
The second statement illustrates that the first derivatives of \nameswish\ are symmetric around the function value~$0.5$.
We can conclude,
for example, that $\min_{\argument\in\R}\functionswishDF{\argument}= \functionswishODF{-\argumentminOP}$.

\begin{proof}[Proof of Lemma~\ref{res:swishsym}]
For the first claim, we find
\begin{align*}
  2-\argument+(2+\argument)e^{-\argument}&=0\\
\Rightarrow~~~(2+\argument)e^{-\argument}&=-2+\argument\tj{adding $-2+\argument$ on both sides of the equation}\\
\Rightarrow~~~e^{-\argument}&=\frac{-2+\argument}{2+\argument}\tj{multiplying both sides by $2+\argument$ (verify that $\argument= -2$ is not a solution of the above equality)}\\
\Rightarrow~~~e^{\argument}&=\frac{2+\argument}{-2+\argument}\tj{using the reciprocals (verify that $\argument= 2$ is not a solution of the above equality either)}\\
 \Rightarrow~~~(-2+\argument)e^{\argument}&=2+\argument\tj{multiplying both sides by $-2+\argument$}\\
\Rightarrow~~~-2-\argument+(-2+\argument)e^{\argument}&=0\tj{adding $-2-\argument$ to both sides}\\
\Rightarrow~~~2+\argument+(2-\argument)e^{\argument}&=0\tj{multiplying both sides by $-1$}\\
\Rightarrow~~~2-(-\argument)+\bigr(2+(-\argument)\bigr)e^{-(-\argument)}&=0\tj{$\argument=-(-\argument)$}\,,
\end{align*}
as desired.

For the second claim,
we find
\begingroup
\allowdisplaybreaks
\begin{align*}
  \functionswishDF{-\argument}
  &=\frac{1+\bigl(1+\functionpar (-\argument)\bigr) e^{-\functionpar(-\argument)}}{(1+e^{-\functionpar(-\argument)})^2}\tj{Lemma~\ref{swishdiff} for \functionswishD}\\
&=\frac{1+(1-\functionpar \argument) e^{\functionpar\argument}}{(1+e^{\functionpar\argument})^2}\tj{simplifying}\\
&=\frac{(1+e^{\functionpar\argument})^2-(1+e^{\functionpar\argument})^2+1+(1-\functionpar \argument) e^{\functionpar\argument}}{(1+e^{\functionpar\argument})^2}\tj{adding a zero-valued term}\\
&=1-\frac{(1+e^{\functionpar\argument})^2-1-(1-\functionpar \argument) e^{\functionpar\argument}}{(1+e^{\functionpar\argument})^2}\tj{splitting the fraction into two parts}\\
&=1-\frac{1+2e^{\functionpar\argument}+e^{2\functionpar\argument}-1-e^{\functionpar\argument}+\functionpar \argument e^{\functionpar\argument}}{(1+e^{\functionpar\argument})^2}\tj{expanding the terms in the numerator}\\
&=1-\frac{e^{2\functionpar\argument}+(1+\functionpar \argument) e^{\functionpar\argument}}{(1+e^{\functionpar\argument})^2}\tj{consolidating the second term}\\
&=1-\frac{1+(1+\functionpar \argument) e^{-\functionpar\argument}}{e^{-2\functionpar\argument}(1+e^{\functionpar\argument})^2}\tj{consolidating the second term further}\\
&=1-\frac{1+(1+\functionpar \argument) e^{-\functionpar\argument}}{(1+e^{-\functionpar\argument})^2}\tj{summarizing the denominator}\\
 &=1- \functionswishDF{\argument}\tj{Lemma~\ref{swishdiff} for \functionswishD}\,,
\end{align*}
\endgroup
as desired.
\end{proof}

\newpage

\later{\section{include later}
\newcommand{\brabra}{]}
\begin{lemma}[$\operatorname{arctan}[\argument\brabra$ versus $(\operatorname{tan}[\argument\brabra)\inv$]\label{arch}
  The equality
  \begin{equation*}
    \operatorname{arctan}[\argument]\cdot \operatorname{tan}[\argument]=1
  \end{equation*}
holds for exactly two distinct values $\argument\in\R$.
\end{lemma}
\begin{proof}[Proof of Lemma~\ref{arch}]
The functions $\operatorname{arctan}$ and $\operatorname{tan}$ are antisymmetric:
\begin{equation*}
  \operatorname{arctan}[\argument]=-\operatorname{arctan}[-\argument]~~\text{and}~~\operatorname{tanh}[\argument]=-\operatorname{arctan}[-\argument]~~~~~~~\text{for all}~\argument\in\R\,.
\end{equation*}
Hence, the function $\argument\mapsto\operatorname{arctan}[\argument]\cdot \operatorname{tan}[\argument]$ is symmetric:
\begin{equation*}
\operatorname{arctan}[\argument]\cdot \operatorname{tan[\argument]}= \bigl(-\operatorname{arctan}[-\argument]\bigr)\cdot\bigl(-\operatorname{tan}[\argument]\bigr)=\operatorname{arctan}[-\argument]\cdot \operatorname{tan}[-\argument]~~~~~~~\text{for all}~\argument\in\R\,.
\end{equation*}
Motivates by this, we restrict ourselves to positive arguments $\argument\in[0,\infty)$ without loss of generality.

The functions $\operatorname{arctan}$ and $\operatorname{tan}$ are (i)~smooth, (ii)~increasing, (ii)~unbounded, and (iii)~$\operatorname{arctan}[0]=\operatorname{tan}[0]=0$. 
Since 
The product $\argument\mapsto\operatorname{arctan}[\argument]\cdot \operatorname{tan}[\argument]$ retains these four properties.
Hence, the intermediate value theorem ensures that 
$\operatorname{arctan}[\argument]\cdot \operatorname{tan}[\argument]$ for an $\argument\in[0,\infty)$.

We can use the above symmetry to conclude.
\end{proof}}

\end{document}